\pdfoutput=1

\documentclass{article}[11pt]
\usepackage{fullpage}

\usepackage[utf8]{inputenc}
\usepackage{hyperref}
\usepackage{amsfonts}
\usepackage{amsmath}
\usepackage{amssymb}
\usepackage{amsthm}
\usepackage{qtree}
\usepackage{graphicx}
\usepackage{booktabs}
\usepackage{subcaption}
\usepackage[section]{placeins}
\graphicspath{ {./figures/} }
\usepackage{xcolor}
\usepackage[english]{babel}
\usepackage[ruled,vlined]{algorithm2e}
\usepackage{bbm}
\usepackage{footmisc}

\usepackage{url}
\usepackage{float}
\usepackage{enumitem}
\usepackage{bbm}
\usepackage{mathtools}

\definecolor{DarkGreen}{rgb}{0.1,0.5,0.1}
\definecolor{DarkRed}{rgb}{0.5,0.1,0.1}
\definecolor{DarkBlue}{rgb}{0.1,0.1,0.5}
\hypersetup{
    unicode=false,          
    pdftoolbar=true,        
    pdfmenubar=true,        
    pdffitwindow=false,     
    pdfnewwindow=true,      
    colorlinks=true,       
    linkcolor=DarkGreen,    
    citecolor=DarkGreen,    
    filecolor=DarkGreen,    
    urlcolor=DarkBlue,
    linktocpage=true,
}

\usepackage{natbib}
\setcitestyle{authoryear,round,citesep={;},aysep={,},yysep={;}}

\usepackage{amsmath,amsfonts,bm}
\usepackage{amsthm}
\usepackage{url}
\usepackage{float}
\usepackage{enumitem}
\usepackage{bbm}
\usepackage{mathtools}

\newcommand{\R}{{\mathbb{R}}}
\newcommand{\E}{{\mathbb{E}}}

\newtheorem*{theorem*}{Theorem}
\newtheorem{theorem}{Theorem}
\newtheorem{claim}{Claim}
\newtheorem{lemma}{Lemma}

\newtheorem{proposition}{Proposition}

\newtheorem{assumption}{Assumption}
\newtheorem{definition}{Definition}
\newtheorem*{remark*}{Remark}
\newtheorem{remark}{Remark}

\makeatletter

\newcommand{\newreptheorem}[2]
{\newenvironment{rep#1}[1]
	{\def\rep@title{#2 \ref{##1}} \begin{rep@theorem}}%
		{\end{rep@theorem}}}
\makeatother
\newreptheorem{theorem}{Theorem}
\newreptheorem{lemma}{Lemma}
\newreptheorem{corollary}{Corollary}
\newreptheorem{proposition}{Proposition}

\makeatletter
\newcommand{\removelatexerror}{\let\@latex@error\@gobble}
\makeatother

\newcommand{\argmin}{\mathop{\mathrm{argmin}}}

\def\eqref#1{equation~\ref{#1}}

\def\ceil#1{\lceil #1 \rceil}

\def\1{\bm{1}}

\def\vb{{\bm{b}}}

\def\vp{{\bm{p}}}

\def\vw{{\bm{w}}}
\def\vx{{\bm{x}}}
\def\vy{{\bm{y}}}

\def\mD{{\bm{D}}}
\def\mE{{\bm{E}}}

\def\mW{{\bm{W}}}
\def\mX{{\bm{X}}}

\def\mZ{{\bm{Z}}}

\DeclareMathAlphabet{\mathsfit}{\encodingdefault}{\sfdefault}{m}{sl}
\SetMathAlphabet{\mathsfit}{bold}{\encodingdefault}{\sfdefault}{bx}{n}

\title{
How Do Transformers Learn Topic Structure: \\ Towards a Mechanistic Understanding
} 

\author{
Yuchen Li$^1$
\quad Yuanzhi Li$^{1,2}$
\quad Andrej Risteski$^1$ \\
 \normalsize{$^1$Carnegie Mellon University\qquad $^2$Microsoft Research}\\
\normalsize{ \texttt{yuchenl4@cs.cmu.edu, yuanzhil@andrew.cmu.edu, aristesk@andrew.cmu.edu}}}

\date{}

\begin{document}

\maketitle

\begin{abstract}
While the successes of transformers across many domains are indisputable, accurate understanding of the learning mechanics is still largely lacking. Their capabilities have been probed on benchmarks which include a variety of structured and reasoning tasks---but mathematical understanding is lagging substantially behind.   
Recent lines of work have begun studying representational aspects of this question: that is, the size/depth/complexity of attention-based networks to perform certain tasks. However, there is no guarantee the learning dynamics will converge to the constructions proposed. 
In our paper, we provide fine-grained mechanistic understanding of how transformers learn ``semantic structure'', understood as capturing co-occurrence structure of words. 
Precisely, we show, through a combination of mathematical analysis and experiments on Wikipedia data and synthetic data modeled by Latent Dirichlet Allocation (LDA), that the embedding layer and the self-attention layer encode the topical structure. 
In the former case, this manifests as higher average inner product of embeddings between same-topic words. 
In the latter, it manifests as higher average pairwise attention between same-topic words. 
The mathematical results involve several assumptions to make the analysis tractable, which we verify on data, and might be of independent interest as well. 

\end{abstract}

\vspace{-0.2cm}
\section{INTRODUCTION}
\vspace{-1mm}
The transformer architecture \citep{vaswani2017attention} is a critical building block of many leading approaches to natural language processing \citep{devlin2019bert, brown2020language}, and other domains such as vision \citep{dosovitskiy2021an} and protein structure prediction \citep{jumper2021highly}. 
While the NLP community has produced a large body of work on probing and visualizing trained networks \citep{hewitt2019structural, clark2019bert, tenney2019bert, kovaleva2019revealing}, we still have little formal understanding of the mechanisms by which transformers, trained with simple gradient-descent based algorithms, learn from their training data. 
The challenge is that the training dynamics are non-trivial, even for relatively simple structured data distributions, and even for simple (e.g. 1-layer) transformers. 

In particular, we study \emph{semantic structure}, as understood through the lens of \emph{co-occurrences} of words, and their topical structure. 
Precisely, if we fit topics to a real-life corpus like Wikipedia using a \emph{Latent Dirichlet Allocation} (LDA, \citealp{blei2003latent}) model, we find a pretrained BERT model produces token embeddings that are more similar (in terms of inner product or cosine similarity) if they belong to the same topic, and more different if they belong to different topics (see e.g. Figure~\ref{fig:wiki_emb_dot_eg}).

Inspired by these observations, we study LDA-generated data as a sandbox to understand---both through experiments on such synthetic data, and theoretical results---the process by which the embeddings and attention learn the topical structure. We find that the above observations from Wikipedia data are even more pronounced on synthetic LDA data. Moreover, we mathematically prove why such structure arises by analyzing a simplified \emph{two-stage training dynamics} for a single-layer transformer trained under the masked language modeling objective. We also verify the two-stage nature of training dynamics obtains for a wide variety of optimizers and hyperparameter settings.  
\footnote{Code is released at \url{https://github.com/YuchenLi01/transformer_topic_model_LDA}}

\vspace{-0.1cm}
\section{OVERVIEW OF RESULTS} \label{sec:overview}
\vspace{-0.2cm}

We focus on understanding the optimization dynamics of transformers in a simple sandbox: a single-layer transformer trained on (synthetic) data following a topic model distribution---and validate that our results robustly transfer to real data (Wikipedia \citealp{wikidump}). We show that topic structure can be encoded both in the embedding layer, and in the attention mechanism of the network. Moreover, even if one of these components is not trained (i.e. handicapped), the other can ``compensate'' for it.  

Theoretically, we characterize precisely how the topic structure is learned in the two extremal cases: when the attention mechanism is frozen to be uniform, and the only model parameters that are trained are the token embeddings; and when the token embeddings are frozen to be one-hot vectors, and the attention parameters (the key, query, and value matrices) are trained. 
We empirically verify our characterization on synthetic LDA-generated data, and also show that on real Wikipedia data, topic structure is learned both in the embeddings, and the attention mechanism.

\vspace{-0.1cm}
\subsection{Topic structure is encoded in token embeddings} 
\vspace{-0.1cm}

\begin{figure}[!t]
  \centering
  \begin{minipage}[b]{0.23\textwidth}
    \includegraphics[width=\textwidth]{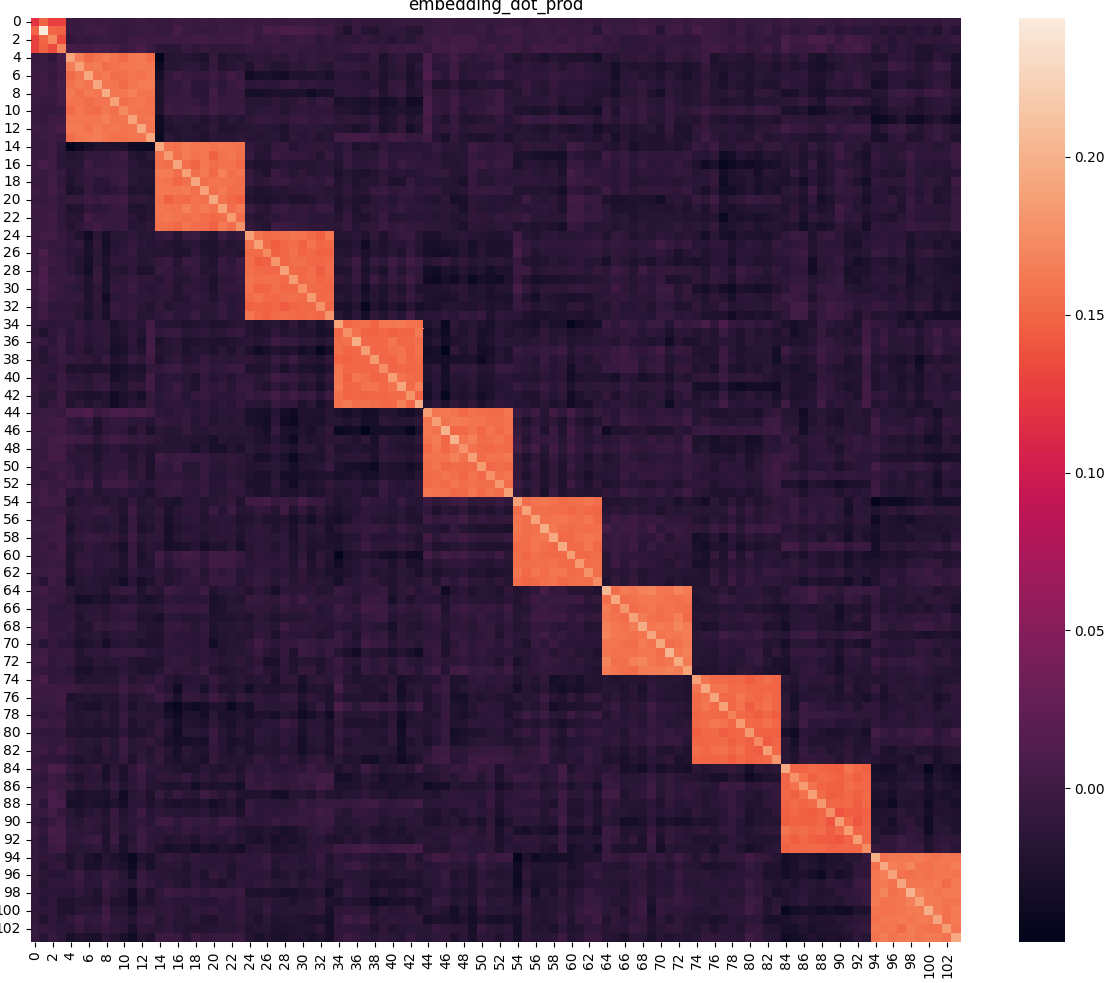}
  \end{minipage}
  \hfill
  \begin{minipage}[b]{0.23\textwidth}
    \includegraphics[width=\textwidth]{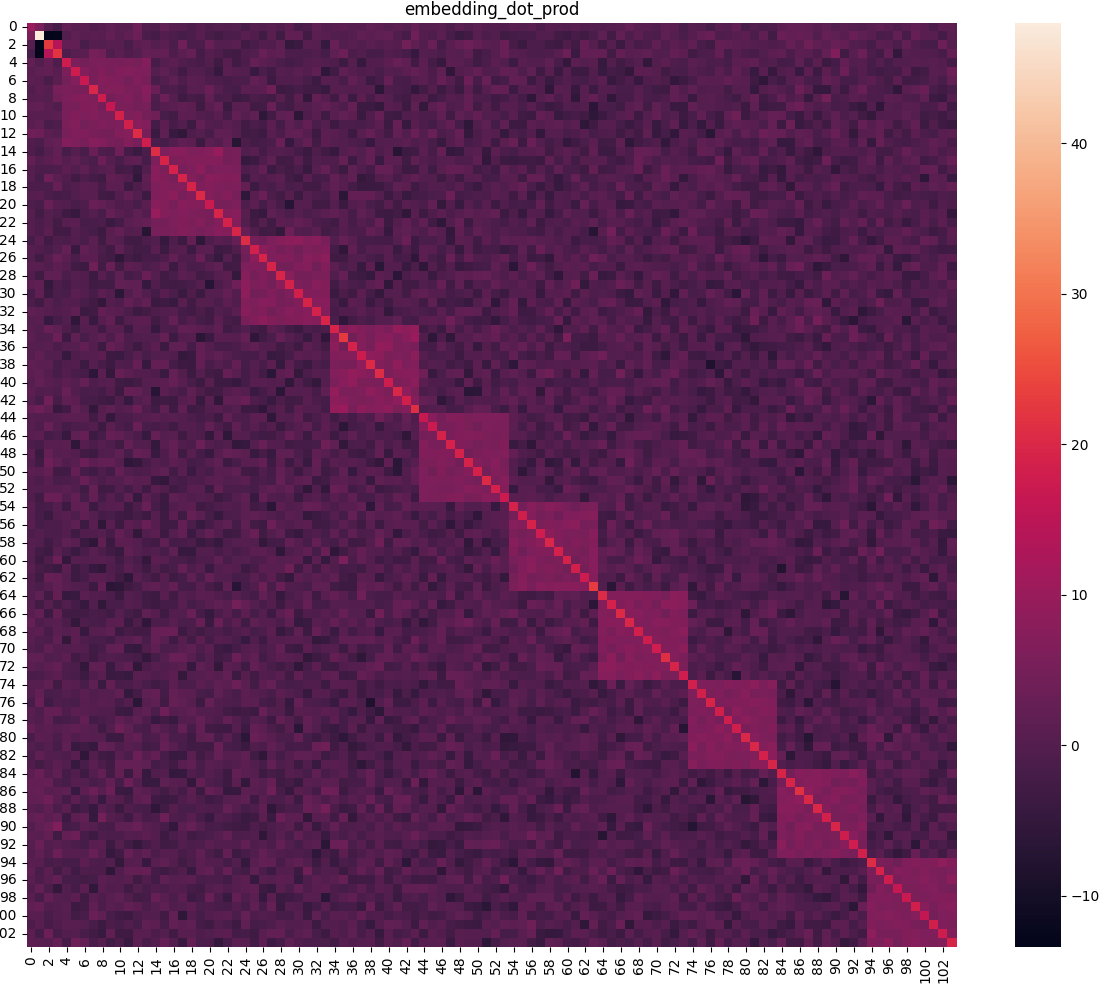}
  \end{minipage}
  \begin{minipage}[b]{0.23\textwidth}
    \includegraphics[width=\textwidth]{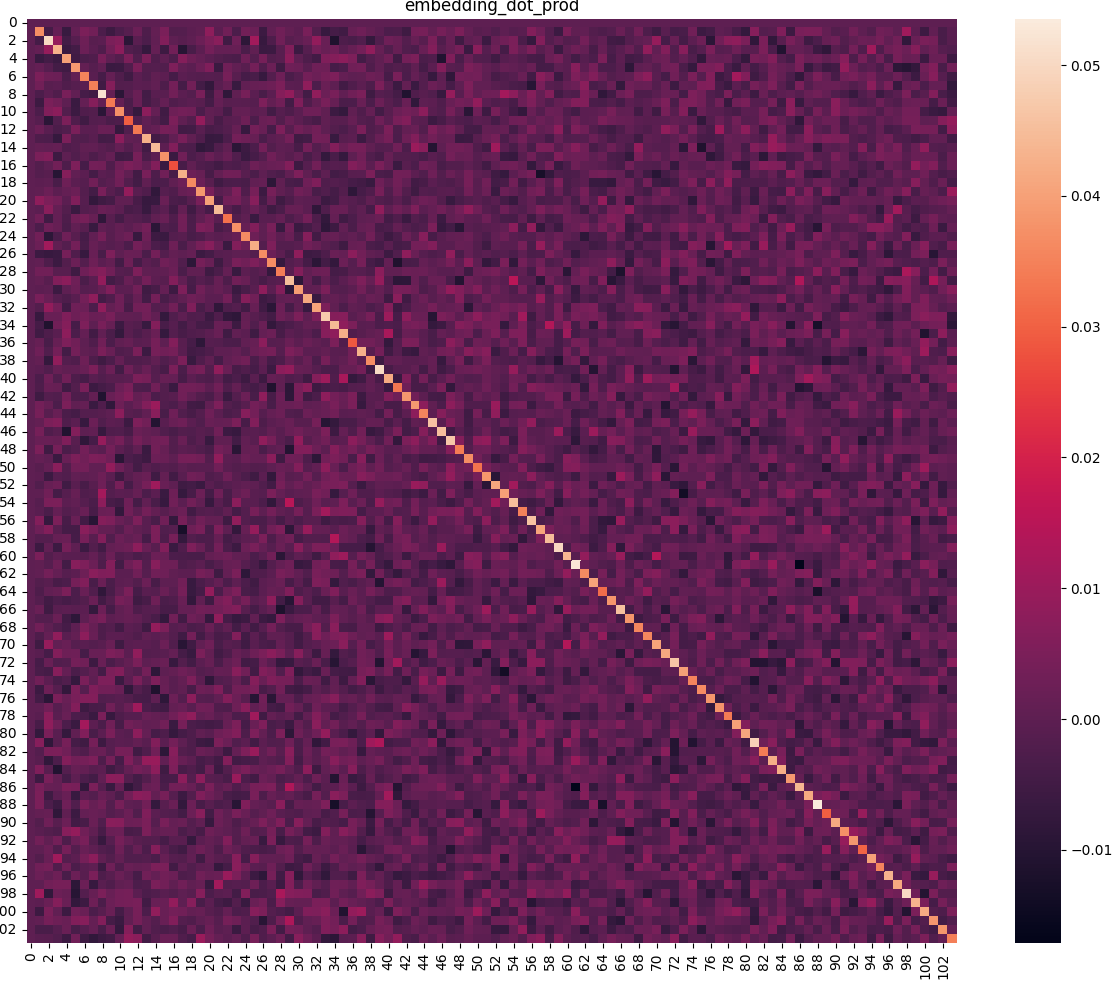}
  \end{minipage}
  \hfill
  \begin{minipage}[b]{0.23\textwidth}
    \includegraphics[width=\textwidth]{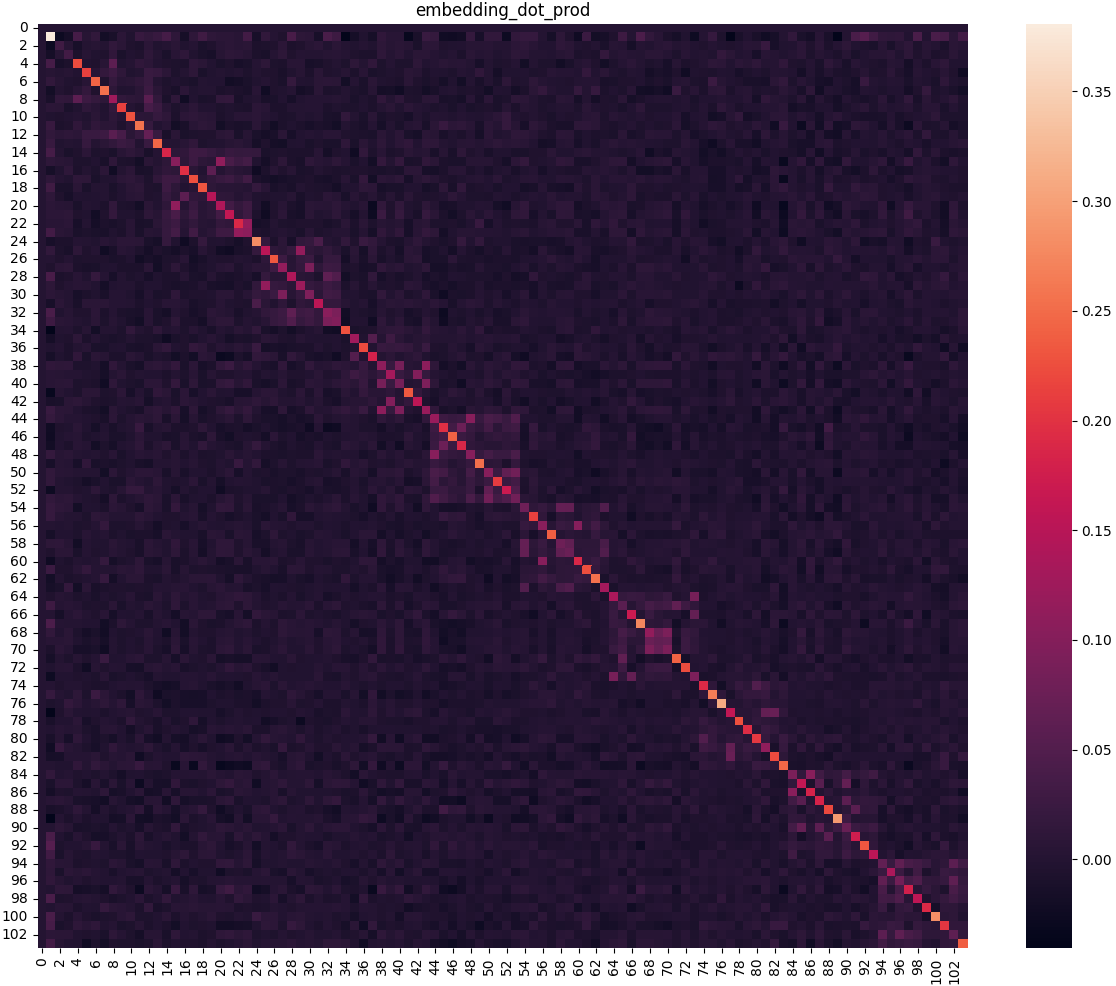}
  \end{minipage}
  \vspace{-0.3cm}
  \caption{Embedding weight dot product of models trained on synthetic topic modeling data (Section~\ref{sec:experiments:setup}). 
  The four plots correspond to different combinations of loss function and optimizer: 
  (left to right) cross-entropy with SGD, cross-entropy with Adam, squared loss with SGD, squared loss with Adam,
  all using learning rate 0.01.
  The block-wise pattern verifies our theory in Section~\ref{sec:embedding}.
  The 10 blocks correspond to the 10 topics in the data distribution in Section~\ref{sec:setup:topic_modeling}.
  In particular, a diagonal pattern is a special case of the block-wise optima that we prove (see Theorem~\ref{thm:optimal_embedding}).
  }
\label{fig:embedding_dot}
\vspace{-0.3cm}
\end{figure}

In the first extremal case, we analyze the optima when we solely train the embedding layer. Precisely, we show that even when we freeze the attention scores to be uniform and all other elements of the transformer are set to identity, the model can still achieve near optimal loss by ``encoding'' the topic structure in the embedding weights:

\begin{theorem*}[Optimal word embedding, informal] 
Suppose the training data follows a topic model data distribution,
and the transformer has trainable embedding layer, frozen (uniform) attention scores, and all other components set to identity.
Then, the optimal embedding layer of a single layer transformer is such that 
the inner product of the embeddings of a pair of words is larger when the words belong to the same topic, and smaller when they belong to different topics. 
\end{theorem*}

Intuitively, this result states that words of the same topic, after training, have more similar embeddings than words of different topics.
In this sense, the embedding layer captures the topic structure. 
We also empirically show (Section~\ref{sec:experiments} and Figure~\ref{fig:embedding_dot}) that this phenomenon is robust to differences in loss function and optimization method. 
See Section~\ref{sec:embedding} for the formal theorem and Appendix~\ref{sec:appendix:embedding} for the proof.

\vspace{-0.1cm}
\subsection{Topic structure is encoded in self-attention} 
\vspace{-0.1cm}

\begin{figure}[t!]
  \centering
  \begin{minipage}[b]{0.23\textwidth}
    \includegraphics[width=\textwidth]{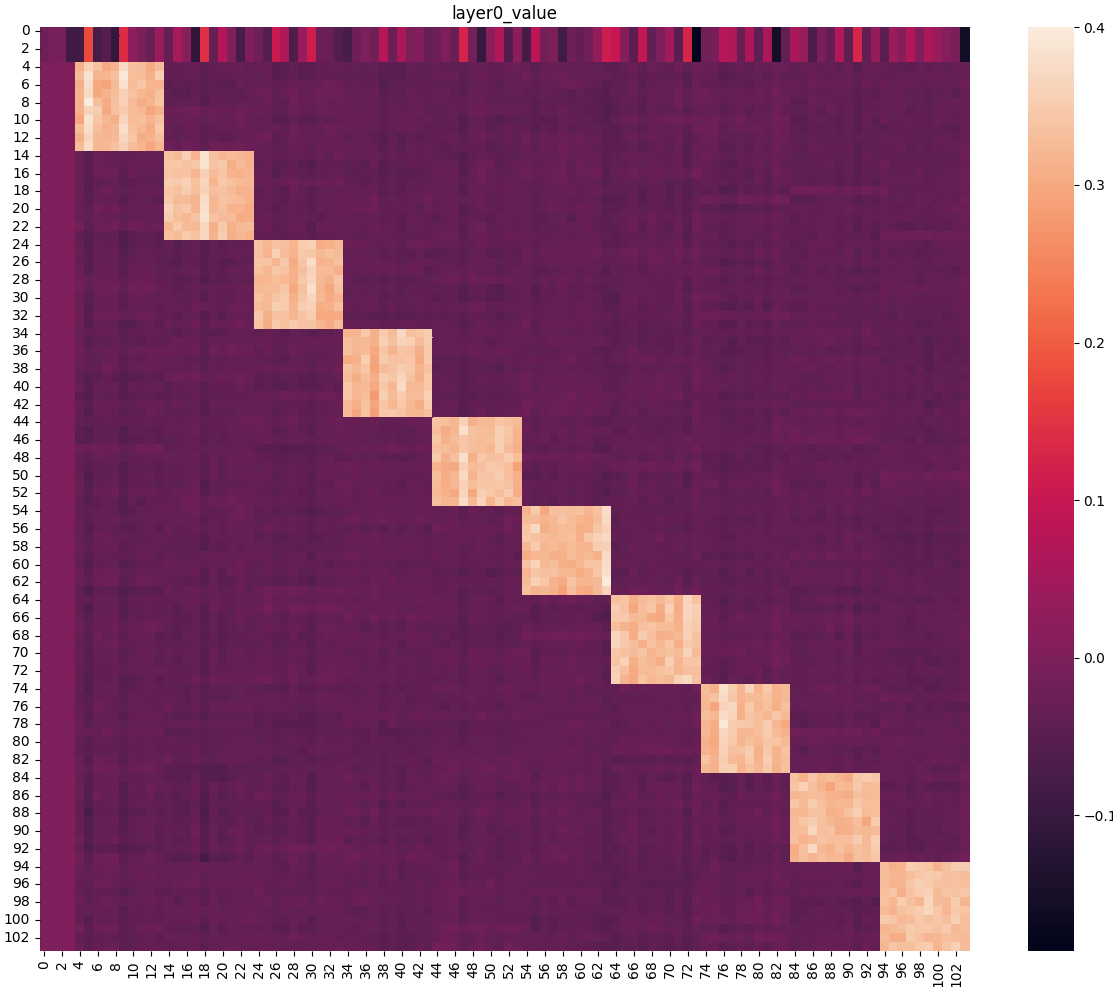}
  \end{minipage}
  \hfill
  \begin{minipage}[b]{0.23\textwidth}
    \includegraphics[width=\textwidth]{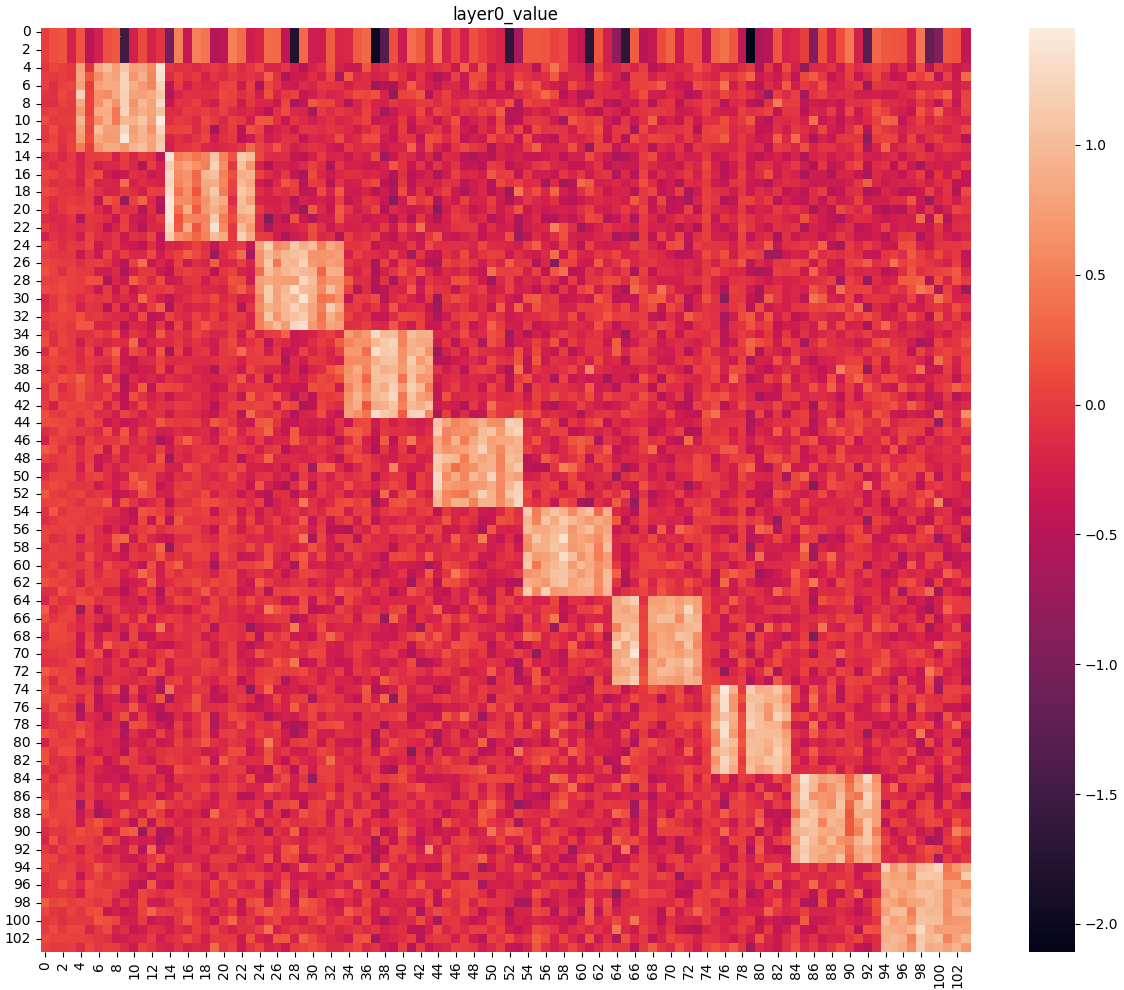}
  \end{minipage}
  \begin{minipage}[b]{0.23\textwidth}
    \includegraphics[width=\textwidth]{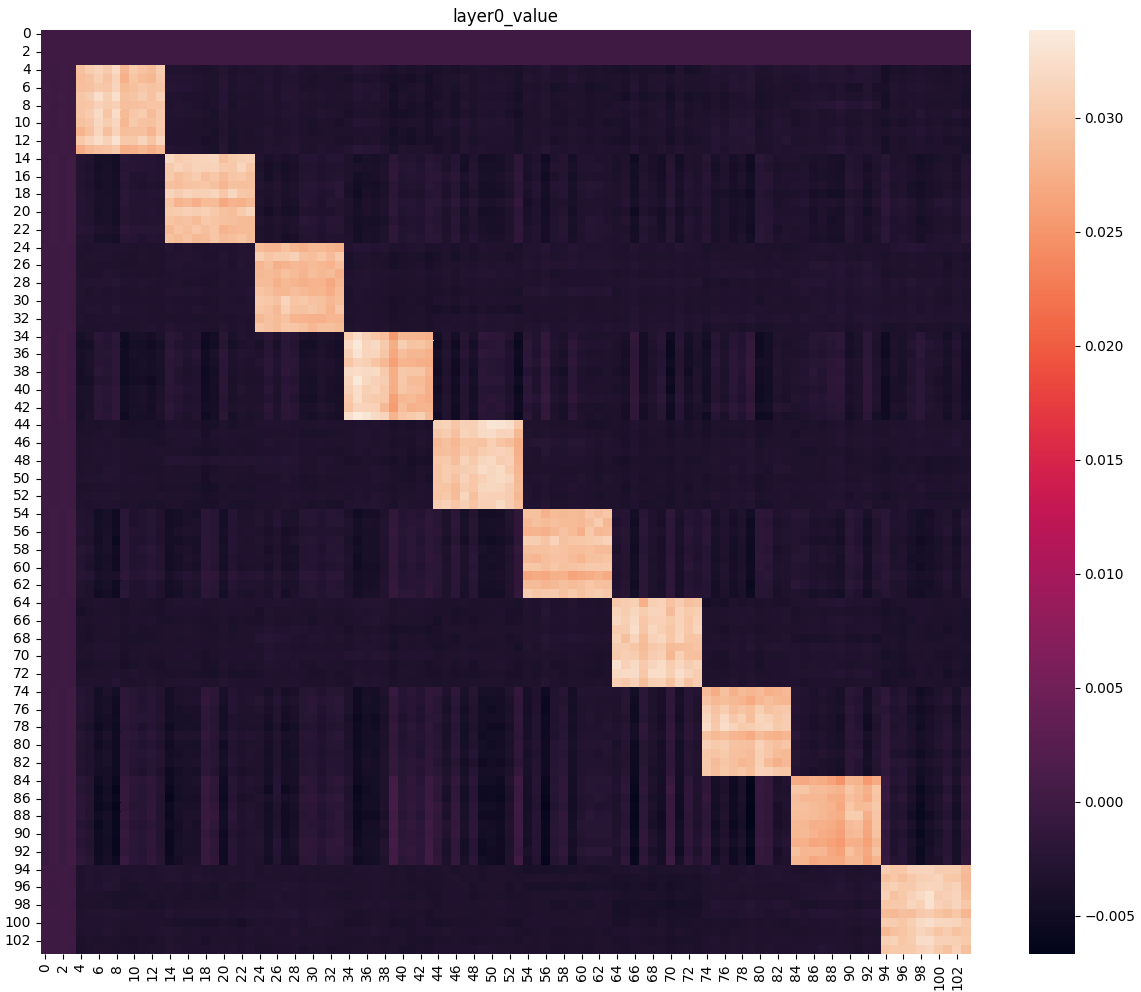}
  \end{minipage}
  \hfill
  \begin{minipage}[b]{0.23\textwidth}
    \includegraphics[width=\textwidth]{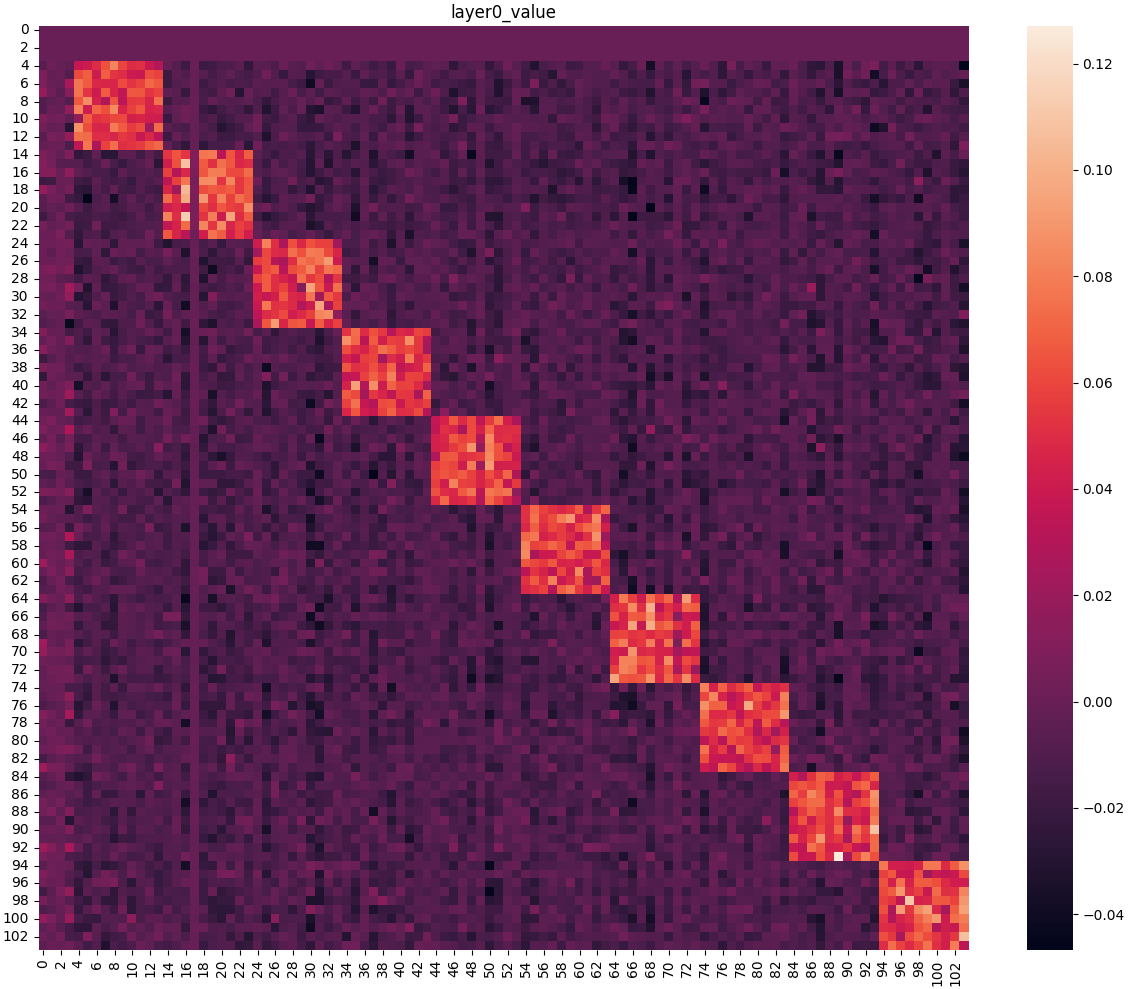}
  \end{minipage}
  \vspace{-0.3cm}
  \caption{Convergence point of trained $\mW^V$ (with $L_2$-regularization) when freezing uniform attention weights and one-hot word embedding. 
  The four plots correspond to different combinations of loss function and optimizer. 
  (Left to right) cross-entropy with SGD, cross-entropy with Adam, squared loss with SGD, squared loss with Adam,
  all using learning rate 0.01.
  The block-wise pattern verifies our theory in Section~\ref{sec:attention:value}.
  The 10 blocks correspond to the 10 topics in the data distribution.
  Results are qualitatively similar without $L_2$-regularization, or if we train $\mW^K$ and $\mW^Q$ instead of freezing them (see Appendix~\ref{sec:appendix:experiments:Wv}).
  }
\label{fig:Wv_one_hot_freeze_uniform_attention_l2reg}
\vspace{-0.3cm}
\end{figure}

In the second extreme, we study the behavior of the self-attention in a transformer trained on a topic modeling distribution, 
without the aid of trained token embeddings --- 
i.e. when we use hard-coded, \emph{one-hot} embeddings. 
The attention weight matrices $\mW^K$, $\mW^Q$, and $\mW^V$ are initialized to near-zero matrices.  
To make the analysis feasible, we break down the training process into two separate stages, and characterize the optima in each stage. In the \emph{first stage}, the attention is frozen to be uniform, and the matrix $\mW^V$ is trained. In the \emph{second stage}, the matrix $\mW^V$ is frozen to the optimal value from the first stage, and the optimal attention weights is analyzed. Intuitively, such a two-stage approximation is reasonable, because in the initial stages of training, the gradients for the value matrix are much larger than those for the key and query matrices 
(see Section~\ref{sec:discussions}).
While this is an approximation, this two-stage phenomenon can be observed empirically for a variety of hyperparameter settings (see Section~\ref{sec:attention:two_stage} and in particular Figure~\ref{fig:no_layernorm_trained_zero_init_emb_dec}).
We also provide empirical evidence that the optima characterized in our analysis closely track the actual convergence points of models. 

In brief, the self-attention function is $\text{Attn}(Z) \coloneqq \mW^V \mZ A(\mZ)$
in which $A(\mZ)$ denotes the attention weights,
and $\mW^V$ is the value matrix weight. Intuitively, $A(Z)_{ij}$ is the importance of the i-th word for predicting the j-th word,
and $\mW^V$ is aggregates the word embeddings in a sentence, weighted by the attention weights $A(\mZ)$.
The formal definition of the model architecture is in Section~\ref{sec:setup:transformer}.

\vspace{-0.2cm}
\subsubsection{Optimal \texorpdfstring{$W^V$}{Wv} in Stage 1}
\label{sec:overview:attention:value}
\vspace{-0.2cm}

We characterize the optimal $\mW^V$ in the initial stage of training:
$\mW^V$ will learn a block-wise structure (see Figure~\ref{fig:Wv_one_hot_freeze_uniform_attention_l2reg}), in which each block corresponds to a topic:

\begin{theorem*}[Optimal $\mW^V$, informal] 
Suppose the training data follows a topic model data distribution,
the token embeddings are frozen to be one-hot vectors, 
and attention scores are frozen to be uniform.
Then, under mild $L_2$ regularization,
the optimal $\mW^V$ for the masked language modeling objective has block-wise structure,
namely the $(i, j)$-th entry of $\mW^V$ is on average larger when the tokens $i$ and $j$ belong to the same topic, 
and on average smaller when the tokens $i$ and $j$ belong to different topics.   
\end{theorem*}

For the formal theorem statement, see Section~\ref{sec:attention}.
The proof is deferred to Appendix~\ref{sec:appendix:attention}. 
We also empirically show (Section~\ref{sec:experiments} and Figure~\ref{fig:Wv_one_hot_freeze_uniform_attention_l2reg}) that this phenomenon is robust to differences in training loss and optimization method.

\begin{figure}[!tbp]
  \centering
  \begin{minipage}[b]{0.5\textwidth}  %
    \centering
    \includegraphics[width=1.0\textwidth]{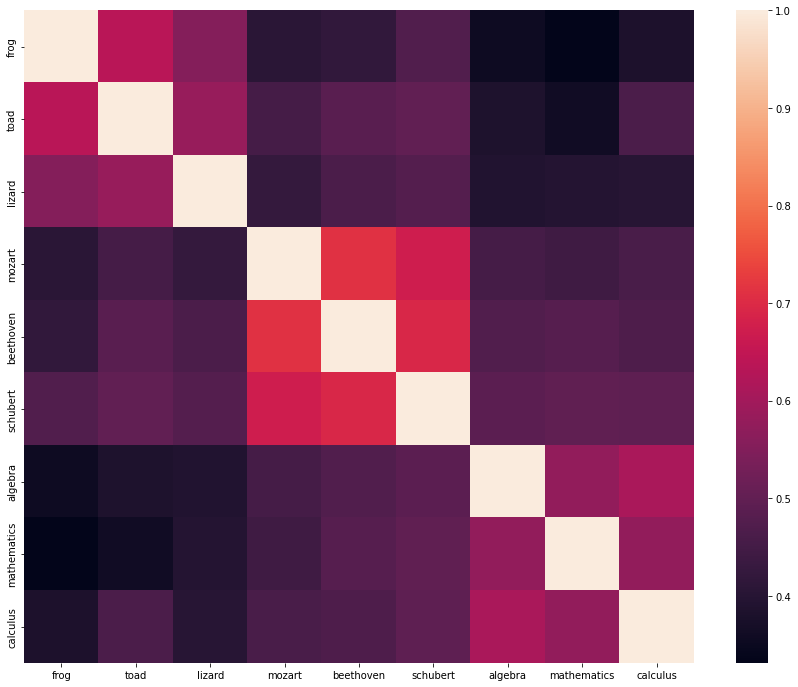}  %
  \end{minipage}
  \caption{For a BERT model pre-trained on Wikipedia corpus, 
  the cosine similarity of the word embeddings encodes topical structures,
  i.e. it is larger if the two words belong to the same topic, and smaller if they belong to different topics. This phenomenon is more pronounced for words that are very likely only under a few topics. 
  In this figure, the nine words fall into three topics: 
  \{frog, toad, lizard\} are animals,
  \{mozart, beethoven, schubert\} are musicians,
  and \{algebra, arithmetic, calculus\} are mathematical concepts.
  }  
  \label{fig:wiki_emb_dot_eg}
\end{figure}

\subsubsection{Optimal attention weights in Stage 2}

For the second stage of the training dynamics, we assume $\mW^V$ is frozen to the optimal value in the first stage, and train the attention weights. 

\begin{theorem*}[Optimal attention weights, informal]
Suppose a single layer transformer is trained on a topic model data distribution,
and $\mW^V$ is frozen to the block-wise first-stage optima. 
Then,
the optimal attention weight for the masked language modeling objective  
is such that on average:
a convex combination of \emph{same-word} attention and \emph{same-topic-different-words} attention should be relatively large, compared to \emph{different-topic} attention.
\end{theorem*}

For the formal assumption and theorem statements, see Section~\ref{sec:attention}.
The proof is deferred to Appendix~\ref{sec:appendix:attention}.

We empirically show (in Section~\ref{sec:experiments}) that even when the all the self-attention weight matrices are \emph{jointly} trained (instead of trained with the two-stage process described), 
the behavior of attention weights still follows the relations that the above theorem describes.

\subsection{Empirical results}
We provide empirical evidence that the main conclusions in our theoretical findings remain robust even under settings that are more complex and realistic than our theoretical setup, and under variations of the training algorithm and loss. 
For example, we also test on synthetic data using a Latent Dirichlet Allocation (LDA) topic model \citep{blei2003latent} instead of our simplified topic modeling distribution; finally, we report results for a model pre-trained on the Wikipedia textual corpus,
and discuss the connections with our conclusions derived in the synthetic setting.
We describe detailed experimental setup and results in Section~\ref{sec:experiments}, as well as Appendix~\ref{sec:appendix:experiments}.

\begin{figure}[t]
  \centering
  \begin{minipage}[b]{0.5\linewidth}
    \centering
    \includegraphics[width=1.0\linewidth]{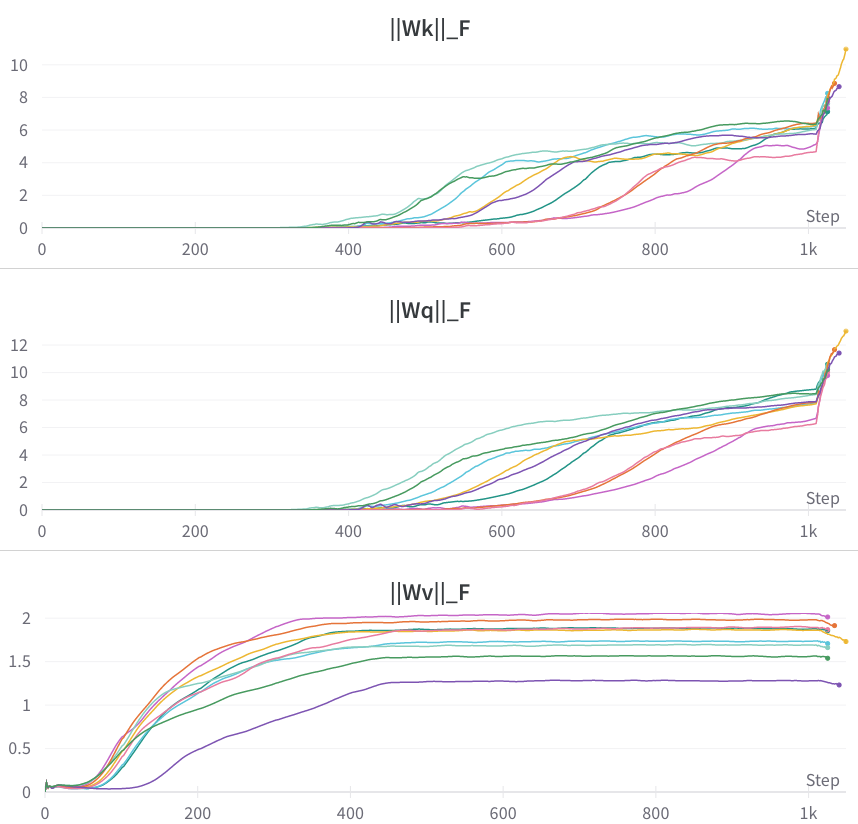}  %
  \end{minipage}
  \caption{
    Two-stage learning dynamics of a single-layer transformer trained on LDA data distribution. 
    All weight matrices are initialized to random matrices near zero, and \emph{simultaneously trained}.
    The learning dynamics naturally exhibits a \emph{two-stage} phenomenon:
    in \textbf{Stage 1} (steps 0-400), the norms of the key matrix ($W^K$, top) and the query matrix ($W^Q$, middle) stay close to 0,
    while the norm of the value matrix ($W^V$, bottom) increases significantly.
    In \textbf{Stage 2} (steps 400-1000), the norms of $W^K$ and $W^Q$ start increasing significantly,
    while the norm of $W^V$ stays relatively flat.
    Different curves in the figure correspond to different settings of the hyperparameters as well as different runs in each setting.
    (See Section~\ref{sec:discussions} for more details.)
    }
  \label{fig:no_layernorm_trained_zero_init_emb_dec}
\end{figure}

\section{PROBLEM SETUP}
\label{sec:setup}

\subsection{Topic models} \label{sec:setup:topic_modeling}

For our theoretical analysis, in order to have a well-defined notion of a ``ground truth'', we will consider data distribution generated by a topic model consisting of $T$ topics $\{1, \cdots, T\}$ and $T v$ words $\{1, \cdots, T v\}$. We will in fact, consider a special case of an LDA (Latent Dirichlet Allocation) model \citep{blei2003latent}. Precisely, each document $\vw$ is a sequence of words $w_1, \cdots, w_N$, and is generated by:
\footnote{Our theoretical results crucially depend on all topics being disjoint, i.e. they do not share common words. 
It is not crucial that the words in the same topic all have the same probabilities. 
Allowing these probabilities to be different would lead to results of similar flavor, but complicates the notation.
}
\begin{enumerate}
    \item \label{setup:topic_distribution} Randomly choose $\tau$ distinct topics $t_1, \cdots, t_\tau$ from $[T]$.
    \item For $n \in [N]$: 
    \begin{enumerate}[leftmargin=0mm]
        \item Randomly choose a topic $t$ from $\{t_1, \cdots, t_\tau\}$.
        \item Randomly choose $w_n$ from $\{{(t-1)v+1}, \cdots, t v\}$.
    \end{enumerate}
\end{enumerate}

Note, under this data distribution, each word belongs to exactly one topic, and different topics do not share common words. 
\begin{definition}[Topic-word indicator]
    \label{def:topic_word_belonging}
    A word $i$ belongs to topic $t$ (denoted as $i \in t$) if
    $i \in \{{(t-1)v+1}, \cdots, t v\}$.
    Correspondingly, $\texttt{topic}(i) \coloneqq \ceil{\frac{i}{v}}$ 
\end{definition}

Let $\mathcal{D}_\vw$ denote the distribution of documents following the above generative process.
Furthermore, for each document $\vw$, let $\mX \in \{0, 1\}^{(T v + 1) \times N}$ denote its \emph{one-hot} encoding, in which
$X_{ij} = 1$ if $w_j = i$, and 0 otherwise.
Analogous to $\mathcal{D}_\vw$, let $\mathcal{D}_\mX$ denote the distribution of document one-hot encodings. 

To simplify our theoretical analysis, we consider the \emph{infinitely-long-document} setting, such that within each document, the empirical token distribution is equal to the groundtruth token distribution:
\begin{assumption}[Infinitely-long documents]
    \label{assumption:infinitely_long_document}
    Each document $\vw$ consists of exactly $\tau$ topics $\{t_1, \cdots, t_\tau\}$. 
    Moreover, for each word $i \in \{1, \cdots, T v\}$ in the vocabulary, its empirical probability in the document 
    \[ \vp_\vw(i) = \frac{\sum_{n=1}^{N} \mathbbm{1}_{w_n = i}}{N} = \begin{cases}
        \frac{1}{\tau v}, \quad &\text{if } i \in \cup_{j=1}^\tau t_j \\ 
        0, \quad &\text{otherwise}
    \end{cases} \]
\end{assumption}

In our synthetic data experiments, we use a finite $N$ and generate data using an LDA model \citep{blei2003latent} which allows for slightly more variability---and demonstrates that our results are robust to changes in the setting. 
Detailed experimental setup is described in Section~\ref{sec:experiments}.

\subsection{Training objective} \label{sec:setup:training}

Given data following the distribution defined in Section~\ref{sec:setup:topic_modeling}, we train a transformer network using the masked language modeling objective \citep{devlin2019bert}.
We first define the token $\texttt{[MASK]} = 0$ in addition to the words $\{1, \cdots, T v\}$ of the topic model.
Three constant probabilities $p_m, p_c, p_r \in (0, 1)$ specify the masking scheme:
\begin{enumerate}
    \item \label{step:choose_masked_positions} For the original document $\vw = w_1 \cdots w_N$, first randomly choose a set of masked indices $M(\vw) \subset [N]$ such that $\forall i \in [N]$, with probability $p_m$, $i \in M(\vw)$. 
    \item \label{step:generate_masked_doc} Define the masked document $\tilde{\vw} = \tilde{w}_1 \cdots \tilde{w}_N$ such that for each $i \in [N]$,
    \begin{enumerate}
        \item \label{step:generate_masked_doc:not_masked} If $i \notin M(\vw)$, then $\tilde{w}_i = w_i$. 
        \item \label{step:generate_masked_doc:masked} If $i \in M(\vw)$, then $\tilde{w}_{i} = 
        \begin{cases}
            w_i, \text{with probability } p_c \\ 
            \text{random word in } [T v], \text{with probability } p_r \\ 
            \texttt{[MASK] = 0}, \text{with probability } 1 - p_c - p_r
        \end{cases}
        $
    \end{enumerate}
\end{enumerate}

Given a document $\vw$ and its masked version $\tilde{\vw}$, the model $f_\theta$ (parameterized by $\theta$) observes $\tilde{\vw}$ and is trained to predict the original words at the masked positions $M$.
More formally, given the one-hot encoding of the masked document $\tilde{\mX}$,
and the model prediction $\hat{\mX} = f_\theta(\tilde{\mX}) \in \R^{(T v + 1) \times N}$,
letting $\mX_{:j}$ denote the $j$-th column of matrix $\mX$,
for some loss function $l(\cdot, \cdot) \to \R$,
the training objective is $\min_\theta L(\theta)$ for
\begin{equation}
    \label{eq:objective}
    L(\theta) = \E_{\mX \sim \mathcal{D}_\mX} \E_{M} \frac{1}{|M|} \sum_{j \in M} l(f_\theta(\tilde{\mX})_{:j}, \mX_{:j} )
\end{equation}

Motivated by the empirical success of applying weight decay to training transformers, 
we also consider a regularized version of the above masked language modeling objective.
For $L_2$-regularization
\footnote{When $\theta$ is a vector, $L_2$-regularization penalizes $\| \theta \|_2$. When $\theta$ is a matrix, the correct norm to regularize is $\| \theta \|_F$.}
with parameter $\lambda > 0$:

\begin{equation}
    \label{eq:objective_l2reg}
    L_{\text{l2reg}}(\theta) = L(\theta) + \lambda \| \theta \|_2^2
\end{equation}

Our theoretical analysis uses the squared loss:
given a prediction vector $\vx \in \R^d$ and an one-hot label vector $\vy \in \{0, 1\}^d$ in which $y_i = 1$ and $\forall j \ne i, y_j = 0$ 
\begin{equation}
    \label{eq:squared_loss}
    l(\vx, \vy) \coloneqq l_{\text{sq}}(\vx, \vy) = \| \vx - \vy \|_2^2
\end{equation}

Our experiments additionally study the cross entropy loss:
\begin{equation}
    \label{eq:cross_entropy_loss}
    l(\vx, \vy) \coloneqq l_{\text{ce}}(\vx, \vy) = -\log \frac{\exp{(\vx_i)}}{\sum_{j=1}^d \exp{(\vx_j)}}
\end{equation}

\begin{remark} \label{rem:difficulty_cross_entropy}
    We give results for both types of loss functions because the cross-entropy loss, albeit practically more commonly used, is theoretically less convenient. Concretely, it involves the softmax operation which is invariant under addition by the same constant in each dimension 
    (implying that the optimal logits are not necessarily unique); 
    moreover, the optimal logits are often at infinity. By contrast, with squared loss, the set of optima is more easily characterized using some finite-valued closed form expressions. 
    
    Empirically, we will show (in Section~\ref{sec:experiments}) that the conclusions in our theoretical analyses hold for both the cross-entropy loss and the squared loss, as well as with variants of the training algorithm like SGD and Adam.
\end{remark}

\subsection{Transformer network architecture}  \label{sec:setup:transformer}

To theoretically reason about the role played by the embedding layer and the self-attention layer, we consider a one-layer transformer model \citep{vaswani2017attention} with the simplification that the residual connection and normalization layers are removed. Precisely:
\[ f(\mZ) = \mW^{\text{pred}} (\mW^V \mZ) \sigma(\frac{(\mW^K \mZ)^\top (\mW^Q \mZ)}{\sqrt{d_a}}) + \vb^{\text{pred}} \]
 
$\mZ \in \R^{d \times N}$ is the input representation. $d$ is the embedding dimension.
$\mW^{\text{pred}} \in \R^{V \times d}$ and $\vb^{\text{pred}} \in \R^V$ are the prediction head weights and biases. 
$V$ is the vocabulary size. 
In our masked language modeling setting (Section~\ref{sec:setup:training}), $V = T v + 1$.
$\mW^V \in \R^{d \times d}$ is the value matrix weight.
$\sigma: \R^{N \times N} \mapsto (0, 1)^{N \times N}$ is the column-wise softmax operation, such that 
$\sigma(A)_{ij} = \frac{\exp{(A_{ij})}}{\sum_{l=1}^N \exp{(A_{lj})}}$.
$d_a$ is the attention head size.
$\mW^K \in \R^{d_a \times d}$ is the key matrix. 
$\mW^Q \in \R^{d_a \times d}$ is the query matrix.
Let $A(\mZ)$ denote the attention weights:
\begin{equation}
    \label{eq:attention_weights}
    A(\mZ) \coloneqq \sigma\left(\frac{(\mW^K \mZ)^\top (\mW^Q \mZ)}{\sqrt{d_a}}\right) \in (0, 1)^{N \times N}
\end{equation}

Appendix~\ref{sec:appendix:setup} includes additional remarks on the architecture.

In our setting, the input $\mZ$ is the embedding of the masked document, 
i.e. $\mZ = \mW^E \tilde{\mX}$ for some embedding weights $\mW^E \in \R^{d \times (T v + 1)}$.
Moreover, following empirical best practice \citep{press2017using} and standard implementation in \citep{wolf2020transformers}, 
we weight-tie the prediction head weight $\mW^{\text{pred}}$ and the embedding weight $\mW^E$:
\begin{equation}
    \label{eq:simplified_transformer_with_emb}
    f(\tilde{\mX}) = {\mW^E}^\top \mW^V \mW^E \tilde{\mX} A(\mW^E \tilde{\mX}) + \vb^{\text{pred}}
\end{equation}

In part of our theoretical analysis (in Section~\ref{sec:attention}) and experiments (in Section~\ref{sec:experiments}), 
we freeze \emph{one-hot} word embeddings, to study the mechanism that self-attention represents the topic structures without the aid of trained token embeddings. That is, set $d = T v + 1$ and $\mW^E = I$:
\begin{equation}
    \label{eq:simplified_transformer_no_emb}
    f(\tilde{\mX}) = \mW^V \tilde{\mX} A(\tilde{\mX}) + \vb^{\text{pred}}
\end{equation}

\vspace{-3mm} %
\section{TOPIC STRUCTURE CAN BE ENCODED IN TOKEN EMBEDDINGS}  \label{sec:embedding}

The first result shows that, under the topic model data distribution, even if we freeze the self-attention to be uniform, the embedding layer can encode the topic structure. Precisely:

\begin{theorem}[Optimal token embedding] 
\label{thm:optimal_embedding}
Suppose the data distribution follows the topic modeling assumption in Section~\ref{sec:setup:topic_modeling} and Assumption~\ref{assumption:infinitely_long_document}. 
Suppose we train a single layer transformer given by \eqref{eq:simplified_transformer_with_emb} with $\mW^K = 0, \mW^Q = 0, \mW^V = I$ and $\forall i, \vb^{\text{pred}}_i = -\frac{p_m p_r}{\left( 1-(1-p_c)p_m \right) T v}$,
under the masked language modeling objective (\eqref{eq:objective}) with the squared loss (\eqref{eq:squared_loss}). 
Then, there exist constants $u_0, \cdots, u_{Tv} \in \R$ such that
the optimal word embedding weight $\mW^E$ and $\mE \coloneqq {\mW^E}^\top \mW^E$ satisfy:
\begin{enumerate}[leftmargin=*]
    \item The 0-th row of $\mE$ satisfies:
    \begin{enumerate}[leftmargin=*]
        \item $\mE_{00} = -\left(\frac{1}{p_m (1-p_c-p_r)} - 1\right) \cdot u_0$
        \item $\forall t \in [T], \sum_{l \in t} \mE_{0l} = u_0 v$
    \end{enumerate}
    \item The 0-th column of $\mE$ satisfies $\forall i \in \{1, \cdots, T v\}$:
    \begin{enumerate}[leftmargin=*]
        \item $\mE_{i0} = - \left( \frac{1}{(1-p_c-p_r) p_m} - 1 \right) u_i $
    \end{enumerate}
    \item $\mE_{ij}$ ($\forall i, j \in \{1, \cdots, T v\}$) satisfy:
    \begin{enumerate}[leftmargin=*]
        \item $\sum_{l \in \texttt{topic}(i)} \mE_{il} = u_i v + \frac{1}{1-(1-p_c)p_m}$
        \item $\forall t \in [T]$ such that $\texttt{topic}(i) \ne t$, $\sum_{l \in t} \mE_{il} = u_i v $
    \end{enumerate}
\end{enumerate}
\end{theorem}

\begin{remark}
    Point 3 is the important one among the list of conclusions.
    The way to read the theorem is that, among the entries of an optimal $\mE$:
    for $i$ and $j$ corresponding to the indices of tokens of the \textbf{same topic}, $\mE_{ij}$ is (on average) larger, meaning that the embeddings of same-topic tokens are more similar;
    for $i$ and $j$ corresponding to \textbf{different topics}, $\mE_{ij}$ is (on average) smaller, meaning that the embeddings of different-topic tokens are less similar.
    In particular, when the constants $u_0, \cdots, u_{Tv}$ are all zero, then the above larger-vs-smaller difference becomes a positive-vs-zero difference,
    which we roughly observe in practice.
\end{remark}

\begin{remark}
    Intuitively, the setting of the bias $\vb^{\text{pred}}$ is used to ``denoise" the masked sequence, 
    i.e. to subtract the probability caused by filling in random words in the masking process (described in Section~\ref{sec:setup:training}). 
\end{remark}

The proof of this theorem is deferred to Appendix~\ref{sec:appendix:embedding}.

Proving comparable results under cross-entropy loss (\eqref{eq:cross_entropy_loss}) is more challenging considering Remark~\ref{rem:difficulty_cross_entropy}. 
However, we empirically show that, such blockwise pattern in $\mE \coloneqq {\mW^E}^\top \mW^E$ tends to exist in a trained model under both the squared loss and the cross-entropy loss,
and regardless of whether we (i) train all layers or 
(ii) only train the embedding layer while freezing all other layers.
Moreover, the loss achieved in case (ii) is only slightly worse than in case (i). 
Finally, we also show (Figure~\ref{fig:wiki_emb_dot_eg}) that on real data, words that are unambiguous (e.g. ``calculus'', ``Mozart'') exhibit a similar pattern as Theorem~\ref{thm:optimal_embedding} states: 
same-topic words have more similar embeddings, and therefore larger embedding dot products, than different-topic words.
Quantitatively, if we only restrict ourselves to words that are unambigious (i.e. likely to be emitted only under few topics), a similar phenomenon can be observed (see Table~\ref{tab:wiki_emb_attn}).

\section{TOPIC STRUCTURE CAN BE ENCODED IN SELF-ATTENTION}
\label{sec:attention}
\vspace{-0.2cm} 

Whereas the previous section showed that the token embedding layer can in principle perform the heavy-lifting in learning the topic-modeling distribution, 
we further show that self-attention \emph{also} can encode the topic structures, when we disallow training the embedding layer.
That is, we freeze the token embeddings to be one-hot. 

\subsection{The two-stage optimization process of self-attention}
\label{sec:attention:two_stage}

While inspecting the training dynamics of this one-layer transformer on the topic modeling data distribution,
we observed a roughly \emph{two-stage} process (illustrated by Figure~\ref{fig:no_layernorm_trained_zero_init_emb_dec}):
with certain initialization and learning rate settings, 
in \textbf{Stage 1}, the key matrix ($W^K$) and the query matrix ($W^Q$) stay close to 0,
i.e. each position pays a near-uniform attention to all positions in the document,
while the norm of the value matrix ($W^V$) increases significantly.
In \textbf{Stage 2}, the norm of the the value matrix ($W^V$) already plateaus, and only after that, do the key and query matrices ($W^K$ and $W^Q$) start to move.

Thus, while reasoning about the training process of transformers in our data distribution, 
we take motivation from the above empirical observation of such two-stage process,
and consider a corresponding simplification: 
in Stage 1, the attention is frozen to be uniform, and only $\mW^V$ is trained; 
in Stage 2, $\mW^V$ is frozen, while $\mW^K$ and $\mW^Q$ are trained.
This simplification is a reasonable proxy for standard training, 
and we furthermore validate that our theoretical characterizations are robust to standard training, both using SGD and Adam. We provide more discussion on the two-stage optimization process in 
Section~\ref{sec:discussions}.

\subsection{Optimal \texorpdfstring{$W^V$}{Wv} given uniform attention}
\label{sec:attention:value}

The Stage 1 of optimization process is convex (but not strongly convex) in $\mW^V$, 
and we show that the set of minima consist of exactly the set of $\mW^V$ that exhibits a \emph{block-wise} pattern:

\begin{theorem}[Optimal $\mW^V$ with mild $L_2$-regularization when freezing uniform attention] 
    \label{thm:optimal_Wv_given_uniform_attention_l2reg}
    Suppose the data distribution follows the topic modeling assumption in Section~\ref{sec:setup:topic_modeling} and Assumption~\ref{assumption:infinitely_long_document}.
    Suppose we train a single layer transformer given by \eqref{eq:simplified_transformer_no_emb}
    with $\mW^K = 0, \mW^Q = 0, \vb^{\text{pred}} = 0$,
    under the $L_2$-regularized masked language modeling objective (\eqref{eq:objective_l2reg}) with the squared loss (\eqref{eq:squared_loss}).
    Then,
    $\lim_{\lambda \to 0} \argmin L_{\text{l2reg}}(\mW^V) = \{ \mW^{V*} \}$ in which $\mW^{V*} \in \R^{(T v + 1) \times (T v + 1)}$ satisfies:
    \begin{enumerate}[leftmargin=*]
        \item The 0-th row of $\mW^{V*}$:
        \begin{enumerate}[leftmargin=*]
            \item $\forall j \in \{0, \cdots, T v\}, \mW^{V*}_{0j} = 0$
        \end{enumerate}
        \item The 0-th column of $\mW^{V*}$:
        \begin{enumerate}[leftmargin=*]
            \item $\forall i \in \{1, \cdots, T v\}, \mW^{V*}_{i0} = \frac{c_2 c_3 - c_1 T v}{c_2^2 + T v}$
        \end{enumerate}
        \item $\mW^{V*}_{ij}$ ($\forall i, j \in \{1, \cdots, T v\}$):
        \begin{enumerate}[leftmargin=-0mm]
            \item $\forall l \notin \texttt{topic}(i), \; \mW^{V*}_{il} = \mW^{V*}_{\text{diff-topic}} \coloneqq -\frac{c_1 c_2 + c_3}{c_2^2 + T v} $
            \item $\forall l \in \texttt{topic}(i), \mW^{V*}_{il} = \mW^{V*}_{\text{same-topic}} \coloneqq \mW^{V*}_{\text{diff-topic}} + \frac{c_3}{v}$
        \end{enumerate}
    \end{enumerate}
    in which the constants are:
    \begin{itemize}
        \item $c_1 = \frac{p_r}{(1-p_c-p_r) \left(1-(1-p_c)p_m \right) T v} \in (0, 1)$
        \item $c_2 = \frac{1}{(1-p_c-p_r) p_m} - 1 \in (0, +\infty)$
        \item $c_3 = \frac{1}{1-(1-p_c)p_m} \in (1, +\infty)$
    \end{itemize}
\end{theorem}

Empirically, the loss achieved by freezing $\mW^K = \mW^Q = 0$ and only training $\mW^V$ is only slightly greater than the loss achieved by training all of them jointly, see Appendix~\ref{sec:appendix:experiments}.

Intuitively, this block-wise $\mW^V$ shows that, while inferring about the words at the masked positions:
the model looks at unmasked positions in the document,
each unmasked word only contributes to predicting words of the \emph{same topic},
each unmasked word does not contribute to predicting words of \emph{different topics},
and the model implicitly aggregates the topic distribution among the unmasked words, to infer the token distribution in the original document prior to masking.

The proof of this Theorem~\ref{thm:optimal_Wv_given_uniform_attention_l2reg} is deferred to Appendix~\ref{sec:appendix:Wv}. 
Proving a comparable result under the cross-entropy loss \eqref{eq:cross_entropy_loss} is more challenging due to the same reasons outlined in Remark~\ref{rem:difficulty_cross_entropy}.
However, empirically such block-wise $\mW^V$ shows up for both the cross-entropy loss and the squared loss, as we show in Section~\ref{sec:experiments}.

\subsection{Optimal attention weights} 
\label{sec:attention:weights}

In our analysis on the stage 2 optimization process, we freeze the $\mW^V$ to be some representative optima from stage~1 (Theorem~\ref{thm:optimal_Wv_given_uniform_attention_l2reg}), and characterize the optimal attention weights by comparing the following three types of attention weights:
among the \emph{same words} at different positions,
among different words of the \emph{same topic}, and
among words of \emph{different topics}.

We mainly consider the type of optimal $\mW^V$ characterized in Theorem~\ref{thm:optimal_Wv_given_uniform_attention_l2reg}:
$\mW^V$ with uniform blocks
(see Figure~\ref{fig:Wv_one_hot_freeze_uniform_attention_l2reg}).
Empirically, the model often approximately converges to these type of pattern (Section~\ref{sec:experiments}).

To formally reason about the behavior of average attention weights, we consider a simplified setting:

\begin{assumption}[Attention pattern]
    \label{assumption:alpha_beta_attention}
    Following the notation in \eqref{eq:attention_weights},
    assume that for any masked document $\tilde{\vw}$ with embedding $\tilde{\mX}$,
    \[ A(\tilde{\mX})_{ij} = \begin{cases}
        c_1, \text{if } \tilde{w}_i = \tilde{w}_j  \\ 
        c_2, \text{if } \tilde{w}_i \ne \tilde{w}_j \text{ but } \texttt{topic}(\tilde{w}_i) = \texttt{topic}(\tilde{w}_j) \\ 
        c_3, \text{if } \texttt{topic}(\tilde{w}_i) \ne \texttt{topic}(\tilde{w}_j) \\
        \end{cases} \]
    in which $c_2 = \alpha c_3$ and $c_1 = \beta c_3$.
\end{assumption}

We note that this family of attention weights is \emph{realizable}, and
by symmetricity (among different topics and among the words in the same topic) and convexity (in $A(\tilde{\mX})$), it is simple to prove that the attention pattern outlined in Assumption~\ref{assumption:alpha_beta_attention} is among the optimal attention patterns.

We will characterize the setting of $\alpha$ and $\beta$ that minimizes the loss, 
under the following assumptions:

\begin{assumption}
    \label{assumption:asymptotic}
    We consider these asymptotic settings:
    \begin{itemize}[leftmargin=*]
        \item $T \to \infty$, i.e. the total number of topics grows to infinity. 
        \item \textbf{(Sparse documents):} $\tau \to \infty, \tau = o(T)$, 
        i.e. the number of topics in each document also grows to infinity, but much smaller than the total number of topics. 
        (This is a common parameter regime: we typically think of each document as a sparse combination of topics.)
        \item \textbf{(No sparsely supported topics):} $v > (\frac{1}{1-(1-p_c) p_m} + 1)^2 + 1$ ($v$ is the number of tokens in each topic. $v \ge 10$ suffices under Assumption~\ref{assumption:correct_random_balance}. 
        This is also a common regime, where we assume no topic consists only of a small number of words.)
    \end{itemize}
\end{assumption}

\begin{assumption}
    \label{assumption:correct_random_balance}
    In the training objective (Section~\ref{sec:setup:training}),
    we consider the case $p_m < \frac{1}{2}, \; p_c = p_r \in (0, \frac{1}{2})$.
    \footnote{This setting is consistent with the masking scheme proposed in \citet{devlin2019bert}.}
\end{assumption}

\begin{theorem}[Optimal attention weights]
\label{thm:optimal_attention_weights_updated}
    Suppose the data distribution follows the topic modeling assumption in Section~\ref{sec:setup:topic_modeling} and Assumption~\ref{assumption:infinitely_long_document}.
    Suppose we train a single layer transformer given by \eqref{eq:simplified_transformer_no_emb}
    with $\vb^{\text{pred}} = 0$ and $\mW^V$ frozen to the optima in Theorem~\ref{thm:optimal_Wv_given_uniform_attention_l2reg},
    under masked language modeling objective (\eqref{eq:objective}) 
    with the squared loss (\eqref{eq:squared_loss}),
    under Assumption~\ref{assumption:alpha_beta_attention}, Assumption~\ref{assumption:asymptotic}, and Assumption~\ref{assumption:correct_random_balance}.
    Then, the optimal $(\alpha, \beta)$ satisfy
    \[ \frac{v-1}{v} \alpha + \frac{1}{v} \beta \in (\lambda_1 (\tau - 1), \lambda_2 T ) \]
    in which $\lambda_1 \coloneqq \frac{(1-(1-p_c) p_m + p_m p_r) (1+(1-p_c) p_m)}{2 (1-(1-p_c) p_m)}$ and 
    $\lambda_2 \coloneqq 100 (\frac{1-(1-p_c) p_m}{p_m p_r} + 1)$.
\end{theorem}

\begin{remark}
In particular, Theorem~\ref{thm:optimal_attention_weights_updated} implies that if we choose $\tau, T$ such that the lower bound exceeds 1, we expect the attention between same-topic words to be \emph{on average} larger than that between different-topic words.
\end{remark}

\begin{remark}
Note that when $\mW^V$ is block-diagonal with uniform blocks, 
it is impossible to meaningfully bound $\alpha$ or $\beta$ individually;
instead, only their weighted average ($\frac{v-1}{v} \alpha + \frac{1}{v} \beta$) matters.
In other words, different $(\alpha, \beta)$ will incur the same loss, as long as the above weighted average remains the same.
Intuitively, this is because such block-diagonal $\mW^V$ with uniform blocks \emph{sums up} the attention on all words in each topic, and make predictions \emph{solely based on the sums}. 
The proof of Theorem~\ref{thm:optimal_attention_weights_updated} is deferred to Appendix~\ref{sec:appendix:attention:weights:block}.
\end{remark}

\begin{remark}
When there is no $L_2$-regularization, the first-stage optima of $\mW^V$ is not unique.
We include additional analysis for representative cases of $\mW^V$ in Appendix~\ref{sec:appendix:attention:weights:diagonal}.
\end{remark}

\begin{remark}
When $T, \tau$ are finite, 
the loss expression turns out to be too complicated to characterize in closed form
(because all the $o(1)$ terms need to be expanded).
So we instead numerically compute the loss landscape as a function of $\alpha$ and $\beta$.
See Appendix~\ref{sec:appendix:attention:non_asymptotic}.
\end{remark}

\section{EXPERIMENTS}  \label{sec:experiments}

We analyze properties of the training dynamics via extensive experimental analysis. We will describe both the setup for synthetic (LDA-generated) data, and for Wikipedia data. 

\subsection{Results on synthetic LDA-generated data}
\label{sec:experiments:setup}

\paragraph{Experimental setup}  
In our experiments, we generate data following Section~\ref{sec:setup:topic_modeling} with $T=10, v=10$, $N$ uniformly randomly chosen from $[100, 150]$, 
except that Step~\ref{setup:topic_distribution} is changed to sampling the topic distribution according to the Dirichlet distribution (consistent with LDA, \citealp{blei2003latent}) with $\alpha = 0.1$. 
Most sentences contain 2 to 4 topics. 
Our training objective follows Section~\ref{sec:setup:training} with $p_m = 0.15, p_c = 0.1, p_r = 0.1$ following \citet{devlin2019bert}.
We use the model architecture following Section~\ref{sec:setup:transformer}
but add back the bias terms $\vb^{K}, \vb^{Q}, \vb^{V}$, following standard implementation in \citet{wolf2020transformers}.

\paragraph{Trained token embeddings}  
In Figure~\ref{fig:embedding_dot}, we show that for a model in which all components are trained, 
the learned embedding weight $\mW^E$ is such that ${\mW^E}^\top \mW^E$ displays a block-wise pattern. 
In particular, a diagonal pattern is a special case. 
These results show that our theory in Section~\ref{sec:embedding} characterizes the optima of embedding layer which can be found by using either cross-entropy or squared losses, either SGD or Adam optimizers, and even when the other layers in the model are trained instead of frozen.

\paragraph{Learned value matrix \texorpdfstring{$W^V$}{Wv}}

We show that when the word embeddings are \emph{frozen to one-hot}
and the attention weights are uniform (by setting $\mW^K = 0, \mW^Q = 0$), 
the trained $\mW^V$ has a block-wise pattern, corresponding to the topical structure
(see Figure~\ref{fig:Wv_one_hot_freeze_uniform_attention_l2reg}).

We show (in Figure~\ref{fig:Wv_one_hot_trained_attention} in Appendix~\ref{sec:appendix:experiments:Wv}) that even when the attention weights $\mW^K, \mW^Q$ are jointly trained with $\mW^V$,
the model would still approximately converge to the type of block-wise $\mW^V$ described in our analyses in Section~\ref{sec:attention:value}.

\paragraph{Convergence point of trained attention weights}

We show that, our conclusion in Theorem~\ref{thm:optimal_attention_weights_updated} holds not just when $\mW^V$ is \emph{frozen} to a block-wise pattern, 
but also when it is \emph{trained} and naturally converges to such pattern.
And we show (in Table~\ref{tab:topic_attn_trained_Wv_block} in Appendix~\ref{sec:appendix:experiments:attn}) that 
on average, each word pays more attention to words of \emph{the same topic} than to words of \emph{different topics}.

\subsection{Results on natural language data}
For a set of pre-trained transformer-based models (and their corresponding tokenizers) downloaded from Huggingface \citep{wolf2020transformers},
we compare the embedding similarity and attention weights between same-topic tokens and different-topic tokens.
The topics are determined by fitting an LDA model with 100 topics on a sample of Wikipedia corpus \citep{wikidump} tokenized by the above tokenizers. 
We filter stop words. 
For each topic, we only keep a fraction of tokens that LDA assigns the highest likelihood in this topic.
Consistent with our theoretical setting, we restrict to keeping only one topic for each word.
In Table~\ref{tab:wiki_emb_attn_1topics_per_word}, we provide the results after such pre-processing.
We provide additional details about the experimental setup and additional results (including when the last restriction of ``one topic per word" is removed) in Appendix~\ref{sec:appendix:experiments:wiki}.

\begin{table*}[!t]
\centering
\begin{tabular}{cc|ccc}
\toprule
{\bf Model} & {\bf Ambiguity} &  {\bf Avg embedding} &  {\bf Avg embedding} &  {\bf Avg attn weight}  \\ & {\bf Threshold}   & {\bf Cosine Similarity} & {\bf Dot Product} &  \bf{(Same-topic} \\ &  & {\bf (Same-topic/Diff-topic)} & {\bf (Same-topic/Diff-topic)} & \bf /Diff-topic)  \\
\midrule
Bert & 0.0005 & 1.21 & 1.19 &  1.32 \\
 & 0.001   & 1.13 & 1.15 & 1.28 \\
 & 0.002   & 1.11 & 1.13 & 1.22 \\
\hline
Albert & 0.0005  & 5.64 & 6.29 & 1.33 \\
 & 0.001  & 4.18 & 3.74 & 1.28 \\
 & 0.002  & 3.24 & 2.93 & 1.22 \\
\hline
Bart & 0.0005  & 2.80 & 2.67 & 1.35 \\
 & 0.001 & 1.95 & 1.92 & 1.31 \\
 & 0.002  & 1.63 & 1.62 & 1.23 \\
\hline
Electra & 0.0005   & 5.98 & 5.37 & 2.14 \\
 & 0.001   & 7.70 & 7.35 & 2.09 \\
 & 0.002   & 7.46 & 8.08 & 1.95 \\
\hline
Roberta & 0.0005  & 6.44 & 6.81 & 1.40 \\
 & 0.001  & 5.73 & 6.31 & 1.31 \\
 & 0.002  & 5.24 & 5.30 & 1.22 \\
\hline
Bert & 0.0005 & 1.00080 & 1.00063 & 0.99943 \\
(randomly & 0.001  & 0.99974 & 1.00036 & 0.99996  \\
initialized) & 0.002  & 1.00016 & 1.00027 & 1.00007  \\
\bottomrule
\end{tabular}
\vspace{-0.2cm} 
\caption{\label{tab:wiki_emb_attn_1topics_per_word}
For models pretained on Wikipedia dataset, their token embeddings and attention weights encode topic structure. 
The different columns are:
(1) The ``ambiguity threshold", i.e. the number of words per topic, divided by the vocabulary size; {\bf each word is only assigned one topic.}
(2) The average embedding cosine similarity between different words of the \emph{same topic}, divided by that between words of \emph{different topics}.
(3) The average embedding dot product between different words of the \emph{same topic}, divided by that between words of \emph{different topics}.
(4) The average attention weight between different words of the \emph{same topic}, divided by that between words of \emph{different topics}. (The attention weights are normalized for debiasing, see discussion in Appendix~\ref{sec:appendix:experiments:wiki} for more details).
Different rows represent different evaluation settings, controlled by ``ambiguity threshold".
Note that the avg same-topic embedding similarity and attention weight are consistently greater than the avg diff-topic counterparts,
verifying our conclusions in Theorem~\ref{thm:optimal_embedding} and Theorem~\ref{thm:optimal_attention_weights_updated}.
}
\vspace{-0.2cm} 
\end{table*}

\section{RELATED WORKS}

One line of prior works explain the success of transformers 
by empirically showing that the components (e.g. attention heads) of a trained model (e.g. BERT \citealp{devlin2019bert}), 
contain abundant information for solving a wide range of ``probing" tasks, across syntax and semantics \citep{hewitt2019structural, clark2019bert, tenney2019bert, hewitt2019designing, kovaleva2019revealing, belinkov2022probing},
or through other approaches involving the attention weights \citep{vig2019analyzing, htut2019attention, sun2021effective}.
Our result also formalizes some relevant intuitions given in \citet{elhage2021mathematical}, 
such as embedding layer capturing some bigram statistics.
In topic modeling distribution, such ``bigram statistics" translates to co-occurrence in a document.

Recent works start to combine theoretical constructions and controlled experiments to justify the expressive power of transformers through the lens of Turing completeness \citep{bhattamishra2020computational},
function approximation \citep{yun2020are},
representing formal languages \citep{bhattamishra2020ability, ebrahimi2020self, yao2021self, liu2022shortcuts},
learning abstract algebraic operations \citep{zhang2022unveiling},
statistical sample complexity \citep{wei2021statistically, edelman2022inductive},
and learning optimal latent representation \citep{zhang2023analysis}.
Methodologically, we join a long line of works that characterize the capacity of neural network models by assessing their abilities in learning some simple models of the data  
\citep{siegelmann1992rnn, gers2001lstm, weiss2018practical, suzgun2019lstm, merrill2019sequential, hewitt2020rnns, li2021limitations, yao2021self, zhang2022unveiling, liu2022shortcuts}. 
Our work extends this line of works, and in particular, our results indicate that there may be multiple reasonable \emph{representational} optima, 
which calls for formally analyzing the training dynamics 
to gain deeper understanding of what the model actually learns from such data distributions.

On the optimization side,
\citet{nguyen2019transformers, xiong2020layer, liu2020understanding, zhang2020adaptive, li2021robust} propose algorithmic improvements (often with theoretical motivations) to help stabilize the training process of transformers.
Towards explaining the training process of attention-based neural networks, 
\citet{sun2020understanding} analyzes the trends of two quantities that are relevant to model performance and interpretability in text classification setting.

Also relevant to our work, \citet{snell2021approximating} consider cross-attention in LSTM Seq2Seq models trained on machine-translation settings\footnote{Specifically, they consider a data model related to the IBM machine translation model.}.
By contrast, we focus on self-attention in transformers, and we consider a data distribution inspired by topic models. 
Notably, they also propose an intuitive simplifying assumption of a two-stage learning process of the attention heads similar to ours (but without theoretical or empirical validation).
Our work uses a similar assumption 
\footnote{
We independently proposed the two-stage training of attention heads,
and later discovered \citep{snell2021approximating} used a similar assumption.
Comparison with \citep{snell2021approximating} was added during an update of our paper.
Moreover, while \citet{snell2021approximating} is the earliest paper we are aware of that explicitly assumes a two-stage training process specifically for attention heads,
we note that similar approaches (more generally, alternating optimization) commonly appear in the optimization literature in a broad variety of settings.
}
(Section~\ref{sec:attention:two_stage}).
In our work, we validate our version of the two-stage assumption by providing a particular way to initialize the attention weight matrices, along with theoretical intuitions (Section~\ref{sec:discussions}) and empirical validation on synthetic data (Figure~\ref{fig:no_layernorm_trained_zero_init_emb_dec}) as well as real data (Figure~\ref{fig:wiki_zero_init}),
showing that this two-stage process can be a reasonable approximation to the early steps of the real training dynamics of attention-based models under the settings that we analyze.

Recent work by \citet{jelassi2022vision} theoretically shows how transformers learn the spatial structure of image-type datasets through gradient-descent-based optimization algorithms. 
In particular, their attention weights depend on the positional encodings only.
Different from their work, our result (motivated by studying the semantics in language) focuses on topic modeling distribution that actually ignores the position information, 
so the attention weights only depend on the ``bag of words" (i.e. the contents).
In that sense, \citet{jelassi2022vision} and our work complement each other,
since real-world data distribution usually involves a combination of position-dependent and position-independent factors.
An interesting future work would be studying how these factors interact during the training process.

Regarding the type of data distribution that we consider, 
we join a series of works that theoretically reason about the ability of learning under topic-modeling-based distributions \citep{sontag2011complexity, awasthi2015variational, arora2016topic, tosh2021contrastive, luo2022one}.
In particular, \citet{luo2022one} shows that if a model can achieve low loss on contrastive or mask-prediction objectives, then it can recover topic posterior. 
However, these prior works do not theoretically analyze the optimization process of the transformer architecture.
In fact, model architecture can indeed critically influence the resulting model obtained by masked-prediction-type tasks (see \citet{liu2022masked} who highlight the subtlety of the interaction between the particular form of the task and the model specification).
Hence, our analysis extends beyond the scope of these prior works by incorporating the theoretical analysis on the optimization process of transformers trained on topic modeling data distribution.
Empirically, \citet{sia2020tired, thompson2020topic, meng2022topic, zhang2022neural, talebpour2023topics} analyze topic discovery via clustering the \emph{contextualized} representations produced by pretrained language models. 
Different from these works, our theory and experiments on token embeddings focus on the convergence of embedding layer \emph{parameters}.

\section{DISCUSSION}
\label{sec:discussions}

\subsection{The two-stage optimization process}

This two-stage optimization process (Section~\ref{sec:attention:two_stage} and Figure~\ref{fig:no_layernorm_trained_zero_init_emb_dec}) can be thought of as one iteration of the alternating optimization procedure.
That is, we first train $\mW^V$ while freezing $(\mW^K, \mW^Q)$, and then freeze $\mW^V$ while training $(\mW^K, \mW^Q)$, and repeat this process.

In practice, $\mW^K, \mW^Q, \mW^V$ in transformers are typically trained jointly instead of alternatingly.
However, our empirical results show that, the conclusions drawn from the two-stage optimization analysis carry over even when they are trained jointly. 
Moreover, we don't find any qualitative aspects of normal training that are not captured by this two-stage approximation. 

Intuitively, such two-stage phenomena occurs because
if $\mW^K, \mW^Q, \mW^V$ are initialized to random matrices near zero, and simultaneously trained,
then in the initial steps, $\nabla_{\mW^K} L$ contains the term $\mW^Q$ (see \eqref{eq:attention_weights}), which is close to 0.
By contrast, $\nabla_{\mW^V} L$ contains the softmax-normalized attention weights $A(\tilde{\mX})$ (see \eqref{eq:simplified_transformer_no_emb}).
Comparing these two, we shall see that $\nabla_{\mW^V} L$ tends to be of larger in magnitude than $\nabla_{\mW^K} L$,
because each column of $\mW^Q$ sums up to approximately 0, 
whereas each column of $A(\tilde{\mX})$ sums up to exactly 1.

Therefore, in the initial steps (i.e. Stage 1), $\mW^V$ intuitively grows much faster than $\mW^K$.
For the same reason (note the symmetry between $\mW^K$ and $\mW^Q$, see \eqref{eq:attention_weights}),
$\mW^V$ intuitively grows much faster than $\mW^Q$, too.

In Stage 2, it is less intuitively clear why $\| \mW^V \|_F$ tends to plateau. 
Note that empirically, even when $\| \mW^V \|_F$ plateaus, the $\mW^V$ matrix itself still fluctuates with non-vanishing step-by-step changes. 
(That is, in each step, $\mW^V$ ``locally rotates" around the origin with an approximately constant norm.) 
Hence we refer to our Stage 2 analysis (which freezes $\mW^V$ itself) as a simplification. 
However, the final empirical convergence point of $\mW^V$ matches our theoretical analysis.

We show in Figure~\ref{fig:wiki_zero_init} 
that an approximate version of this multi-stage phenomenon can be observed on multi-layer transformers trained on Wikipedia as well.

Finally, this two-stage phenomenon is sensitive to hyperparameters like initialization and learning rate. 
In Figure~\ref{fig:no_layernorm_trained_zero_init_emb_dec}, the 
The training process is not usually visibly two-stage using the common default hyperparameters.
We leave it as an interesting future work to theoretically analyze the training dynamics when the two-stage phenomenon is not present.

\begin{figure}[!htb]
  \centering
  \begin{minipage}[b]{1.0\linewidth}
    \centering
    \includegraphics[width=1.0\linewidth]{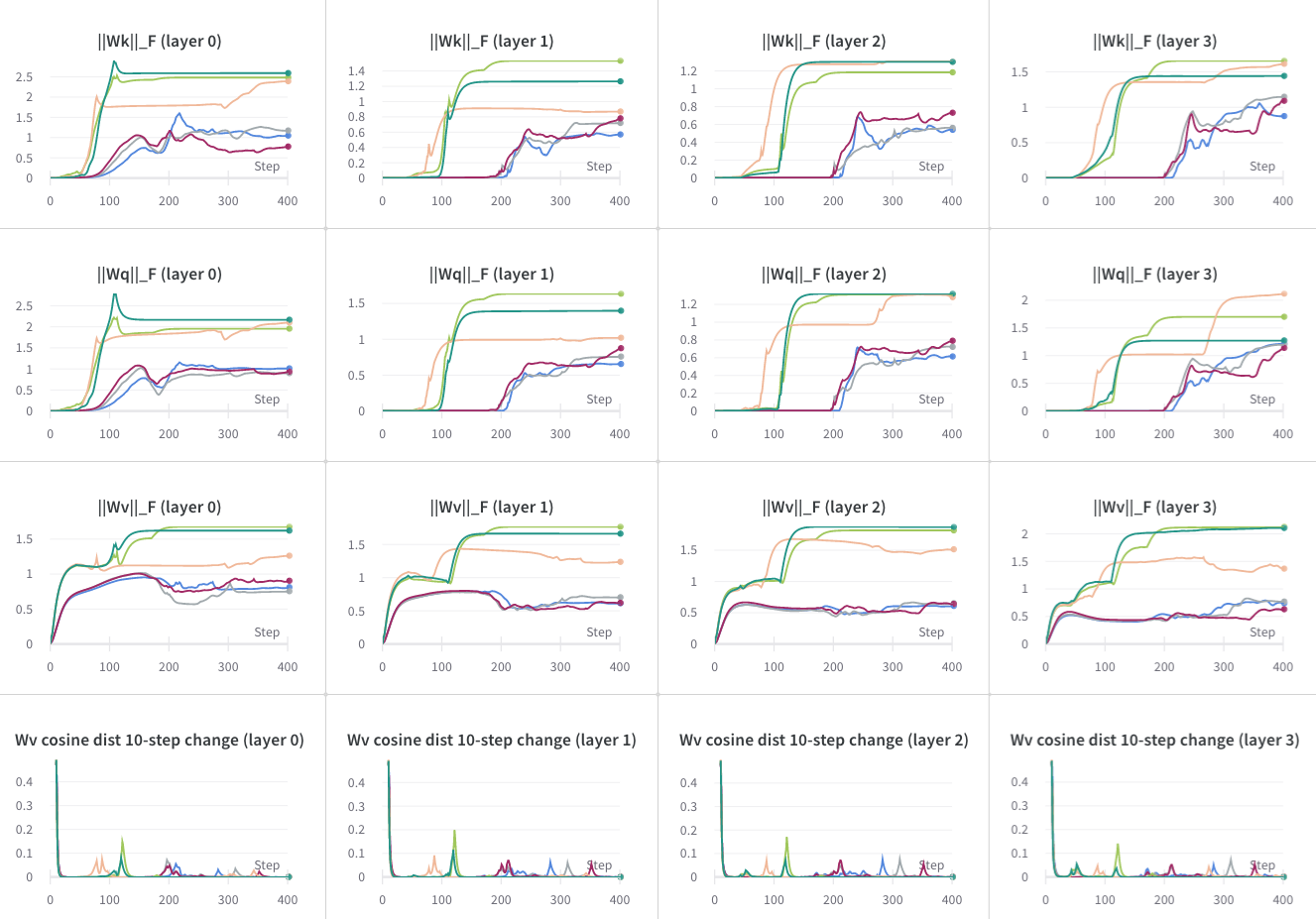}
  \end{minipage}
  \caption{
    Two-stage learning dynamics of a 4-layer, 4-head-per-layer transformer trained on Wikipedia data. 
    All weight matrices (key $\mW^K$, query $\mW^Q$, value $\mW^V$ in each layer) are initialized to random matrices near zero, and \emph{simultaneously trained}.
    Each column corresponds to one layer.
    The top 3 rows plot the trajectories of the Frobenius norms of $\mW^K$, $\mW^Q$, and $\mW^V$ (weights from all heads in the same layer are concatenated together) after each gradient step.
    The bottom row measures the rotation of $\mW^V$, i.e. the cosine distance between $\mW^V$ in step $t$ and $\mW^V$ in step $(t-10)$.
    Cosine distance is defined as $\frac{1-cs}{2} \in [0, 1]$, in which $cs$ is the classic cosine similarity. \\
    The initial 400 steps of the learning dynamics naturally exhibit an \emph{approximately two-stage} phenomenon:
    in \textbf{Stage 1} (roughly steps 0-100), for all 4 layers, the norms of $\mW^K$ and $\mW^Q$ stay close to 0,
    while the norm of $\mW^V$ increases significantly
    and the orientation of $\mW^V$ changes rapidly.
    In \textbf{Stage 2} (roughly steps 100-400), the norms of $\mW^K$'s and $\mW^Q$'s start increasing significantly,
    much later than $\mW^V$ matrices do.
    Different curves in the figure correspond to different settings of the hyperparameters as well as different runs in each setting.
    }
  \label{fig:wiki_zero_init}
\end{figure}

\clearpage
\subsection{Do topic-wise behaviors perfectly correlate with co-occurrence counts?}

Additionally, we note that fitting a topic model is closely related to word co-occurrence statistics, 
which raises the following question: should those empirical phenomenon (i.e. higher same-topic attention and more similar same-topic embeddings, shown in Table~\ref{tab:wiki_emb_attn}) be more fundamentally attributed to larger co-occurrence counts?

In the following, we also compare them with some preliminary empirical results on the behavior of embedding and attention, 
from both topic modeling and co-occurrence perspectives. 
Specifically, we compare the average attention weights and average embedding dot products, between same-topic word pairs and the $N$ pairs of words that co-occur the most frequently in a sample of the Wikipedia corpus. 
The cutoff $N$ is determined so that the number of "top co-occurring word pairs" is the same as the number of word pairs in each topic (controlled by the ambiguity threshold).
The results are summarized in Table~\ref{tab:wiki_co_occur}.

Based on those results, we conjecture that the topic-wise behavior of token embeddings and attention weights cannot be fully explained by simple co-occurrence counts. 

Reasoning about their connections more formally would require analyzing some data distributions that better decouple these factors. 
We think that would be an interesting direction of future work.

\begin{table*}[!h]
\centering
\begin{tabular}{c|cccc}
\toprule
{\bf \# Word Pairs} & {\bf Avg Attn Weight}  & {\bf Avg Attn Weight}  & {\bf Avg Embedding}  &  {\bf Avg Embedding}  \\  & {\bf (Same-Topic)}  & {\bf (Top Co-occur.)}  & {\bf Cosine Similarity} & {\bf Cosine Similarity} \\ & & & {\bf(Same-Topic)} & {\bf(Top Co-occur.)} \\
\midrule
105 & 0.00659 & 0.00751 & 0.468 & 0.316  \\
435  & 0.00621 &  0.00695 & 0.461 & 0.311  \\
1711  & 0.00597 & 0.00677 & 0.425 & 0.323  \\
\bottomrule
\end{tabular}
\caption{\label{tab:wiki_co_occur} 
For a BERT model pretained on Wikipedia dataset, the topic-wise behavior of its token embeddings and attention weights (shown in Table~\ref{tab:wiki_emb_attn_1topics_per_word}) cannot be fully explained by co-occurrence. 
The different columns are:
(1) The number of pairs of tokens that have the highest co-occurrence counts (with stop tokens removed). The cutoffs are selected so that each row contains the same number of words pairs as one topic, corresponding to the rows in Table~\ref{tab:wiki_emb_attn_1topics_per_word};
(2) The average attention weights between same-topic words;
(3) The average attention weights between tokens that co-occur the most;
(4) The average embedding cosine similarity between different words of the \emph{same topic}.
(5) The average embedding cosine similarity between between tokens that co-occur the most.
Note that for all ``\# word pairs" cutoffs considered, same-topic tokens have smaller average attention weight, but larger average embedding cosine similarity.
}
\end{table*}

\section{CONCLUSION}

We initiated the study of understanding training dynamics of transformers in the presence of semantic structure captured by a topic model.
Interesting directions of future work includes extending the analysis to data distributions that captures ``syntactic'' structure, e.g. through simple sandboxes like PCFGs. When both the model and the data distributions are complex, it remains a daunting challenge to ``disentangle" how the many different aspects of the data (e.g. semantic and syntactic elements)
are learned through the different parts of model architecture (e.g. attention, positional encodings, and embeddings).
\vspace{-2mm}

\subsubsection*{ACKNOWLEDGEMENTS}
We thank Bingbin Liu, Yusha Liu, and Tanya Marwah for proofreading and providing constructive comments,
Yewen Fan for helpful suggestions on empirically obtaining the two-stage optimization process,
and Emmy Liu and Graham Neubig for insightful discussions on the connections with empirical observations.

Andrej Risteski and Yuchen Li acknowledge support by NSF awards IIS-2211907 and CCF-2238523. Andrej Risteski also acknowledges support by Amazon Research Award ``Causal + Deep Out-of-Distribution Learning''.

\bibliography{references}
\bibliographystyle{ref_style}

\newpage
\appendix

\thispagestyle{empty}

\def\toptitlebar{
\hrule height4pt
\vskip .25in}

\def\bottomtitlebar{
\vskip .25in
\hrule height1pt
\vskip .25in}

\newcommand{\makesupplementtitle}{\hsize\textwidth
    \linewidth\hsize \toptitlebar {\centering
        {\Large\bfseries Supplementary Material \par}}
    \bottomtitlebar}

\makesupplementtitle

\renewcommand{\theequation}{\thesection.\arabic{equation}}

\tableofcontents

\newpage

\section{ADDITIONAL INFORMATION ON THE SETUP}
\label{sec:appendix:setup}

The positional encoding at the input is also removed, because the position information of a word in a document is irrelevant to the topic model defined in Section~\ref{sec:setup:topic_modeling}.

We also use a single-head attention.

\subsection{Lemma on the optimal linear transform when freezing uniform attention}
\label{sec:appendix:optimal_linear_transform}

Under our setting, we first prove the following useful Lemma~\ref{lemma:optimal_linear_transform_given_uniform_attention}.
Intuitively, it states that, when freezing uniform attention, 
the output of self-attention weights essentially counts the \emph{unmasked} tokens in the document (as a result of the masking process described in Section~\ref{sec:setup:training}).
Given those counts, the best way to predict a token at the masked positions in the \emph{original} document (i.e. prior to the masking process) is to:
\begin{enumerate}
    \item First, aggregate the counts of the unmasked words within each topic, to infer the topic distribution in the observed document. In this, we further have the restriction that: 
    \begin{itemize}
        \item Each unmasked word only contributes to predicting words of the \emph{same topic}
        \item Each unmasked word does not contribute to predicting words of \emph{different topics}
        \item Never predict the mask token (\texttt{[MASK]}), because the original document does not contain any \texttt{[MASK]}
    \end{itemize}
    \item Second, we ``denoise" the topic distribution, i.e. we subtract the probability caused by filling in random words in the masking process (described in Section~\ref{sec:setup:training}).
\end{enumerate}

In line with our single layer transformer architecture (Section~\ref{sec:setup:transformer}, \eqref{eq:simplified_transformer_no_emb}),
we consider a special case in which the attention is \emph{uniform}, 
i.e.  $\forall i, j \in \{1,\cdots,N\}, A(\tilde{\mX})_{ij} = \frac{1}{N}$,
denoted by $A(\tilde{\mX}) = \left[ \frac{1}{N} \right]_{N \times N}$.
(This can be achieved by setting $\mW^K = 0, \mW^Q = 0$.) 
\begin{equation}
    \label{eq:linear_transform_with_attention}
    f(\tilde{\mX}) = \mW \tilde{\mX} \left[ \frac{1}{N} \right]_{N \times N}
\end{equation}
which applies self-attention (\eqref{eq:attention_weights}) on the one-hot representation of the masked document $\tilde{\mX} \in \{0, 1\}^{(T v + 1) \times N}$. %

\begin{lemma}[optimal linear transform when freezing uniform attention] 
    \label{lemma:optimal_linear_transform_given_uniform_attention}
    Consider the simplified transformer architecture given by \eqref{eq:linear_transform_with_attention} with , as well as the masked language modeling objective (\eqref{eq:objective}) with squared loss (\eqref{eq:squared_loss}). 
    Then the set of minimizers $\argmin L(\mW)$ consists of all $\mW \in \R^{(T v + 1) \times (T v + 1)}$ that satisfy:
    there exist constants $u_0, \cdots, u_{Tv} \in \R$ such that
    \begin{enumerate} %
        \item The 0-th row of $\mW$:
        \begin{enumerate}
            \item $\mW_{00} = -\left(\frac{1}{p_m (1-p_c-p_r)} - 1\right) \cdot u_0$
            \item $\forall t \in [T], \sum_{l \in t} \mW_{0l} = u_0 v$
        \end{enumerate}
        \item The 0-th column of $\mW$:
        \begin{enumerate}
            \item $\forall i \in \{1, \cdots, T v\}, \mW_{i0} = -\frac{p_r}{(1-p_c-p_r) \left(1-(1-p_c)p_m \right) T v} - \left( \frac{1}{(1-p_c-p_r) p_m} - 1 \right) u_i $
        \end{enumerate}
        \item $\mW_{ij}$ ($\forall i, j \in \{1, \cdots, T v\}$):
        \begin{enumerate}
            \item $\sum_{l \in \texttt{topic}(i)} \mW_{il} = \frac{1}{1-(1-p_c)p_m} + u_i v$
            \item $\forall t \in [T]$ such that $\texttt{topic}(i) \ne t$, $\sum_{l \in t} \mW_{il} = u_i v $
        \end{enumerate}
    \end{enumerate}
\end{lemma}

\begin{remark}
\label{rem:not_strongly_convex}
At the first glance, it might seem that the objective has a unique optima because it involves a squared loss, which is strongly convex.
However, such uniqueness is undermined by the uniform attention condition: 
$\mW$ is multiplied with a rank-1 matrix $A(\tilde{\mX}) = \left[ \frac{1}{N} \right]_{N \times N}$.
This $A(\tilde{\mX})$ will appear as a matrix multiplier in the Hessian of the objective with respect to $\mW$,
and so the Hessian is of rank 1, 
and therefore cannot have a positive minimum eigenvalue,
implying that the objective is in fact \emph{not} strongly convex.

In fact, this optimization objective becomes strongly convex with an $L_2$ regularization for some $\lambda > 0$.
\[ \argmin_{\mW^V} L_{MLM}(\mW^V) + \lambda \| \mW^V \|_F \]
\end{remark}

\begin{proof}
For document $\vw$ and the corresponding (masked) one-hot embedding $ \tilde{\mX}$ :
\begin{align*}
    &\quad \left[ \tilde{\mX} A(\tilde{\mX}) \right]_{ij} \\
    &= \frac{1}{N} \sum_{l=1}^N  \tilde{\mX}_{il} \quad \text{(i.e. independent of $j$)}  \\
    &= \frac{1}{N} \sum_{l=1}^N \1_{\tilde{\mX}_{il} = 1} \quad \text{(since $\tilde{\mX}$ is one-hot)} \\
    &= \begin{cases}
        p_m (1-p_c-p_r) \quad &\text{if } i=0 \\ 
        P_{\vw}(i) (1-(1-p_c) p_m) + \frac{p_m p_r}{v T} \quad &\text{if } i \in \{1, \cdots, T v\}
    \end{cases} \quad \text{(by \eqref{eq:observed_token_distribution})}
\end{align*}

Thus, the model prediction $\mW \tilde{\mX} A(\tilde{\mX})$ satisfies
\begin{equation}
    \label{eq:pred}
    \begin{split}
        &\quad (\mW \tilde{\mX} A(\tilde{\mX}))_{ij} = \mW_{i0} p_m (1-p_c-p_r) + \sum_{l=1}^{T v} \mW_{il} \left( P_{\vw}(l) (1-(1-p_c) p_m) + \frac{p_m p_r}{v T} \right) \\ 
        &= \mW_{i0} p_m (1-p_c-p_r) + \left( 1-(1-p_c) p_m \right) \cdot \sum_{l=1}^{T v} \mW_{il} P_{\vw}(l) + \frac{p_m p_r}{v T} \cdot \sum_{l=1}^{T v} \mW_{il} \\
        &= \mW_{i0} p_m (1-p_c-p_r) + \left( 1-(1-p_c) p_m \right) \cdot \left( \sum_{l \in \texttt{topic}(i)} \mW_{il} P_{\vw}(i) + \sum_{l \notin \texttt{topic}(i)} \mW_{il} P_{\vw}(l) \right) + \frac{p_m p_r}{v T} \cdot \sum_{l=1}^{T v} \mW_{il} 
    \end{split}
\end{equation}
and the last step follows since $\forall l \in \texttt{topic}(i), P_{\vw}(l) = P_{\vw}(i)$ under our setting in Section~\ref{sec:setup:topic_modeling}. 

Recall that the loss is 
\[ L(\mW) = \E_{\mX \sim \mathcal{D}_\mX} \E_{M} \frac{1}{|M|} \sum_{j \in M} \| (\mW \tilde{\mX} A(\tilde{\mX}))_{:j} - \mX_{:j} \|_2^2 \]
We will show that the average taken over $j \in M$ is the same as the average taken over all positions $j \in [N]$,  
by Assumption~\ref{assumption:infinitely_long_document} and because $M$ is uniformly randomly sampled from $[N]$. 
Moreover, note that $A(\tilde{\mX}) = \left[ \frac{1}{N} \right]_{N \times N}$,
so $(\mW \tilde{\mX} A(\tilde{\mX}))_{:j}$ is independent of $j$.
The above observations imply that the loss can be simplified to
\[ L(\mW) = \E_{\mX \sim \mathcal{D}_\mX} \frac{1}{N} \sum_{j=1}^N \| (\mW \tilde{\mX} A(\tilde{\mX}))_{:j} - \mX_{:j} \|_2^2 \]
and so $L(\mW)$ is minimized when $\forall \mX$,
\[ (\mW \tilde{\mX} A(\tilde{\mX}))_{:j} = \frac{1}{N} \sum_{l=1}^N \mX_{:l} \]

which requires $\forall i \in \{0, \cdots, T v + 1 \}$,
\begin{equation}
    \label{eq:optimal_pred}
    \begin{split}
        (\mW \tilde{\mX} A(\tilde{\mX}))_{0j} &= 0 \\
        (\mW \tilde{\mX} A(\tilde{\mX}))_{ij} &= P_{\vw}(i), \quad \forall i \in \{1, \cdots, T v \}
    \end{split}
\end{equation}

From \eqref{eq:pred} and \eqref{eq:optimal_pred} we get:
\begin{equation}
    \label{eq:solution_pred}
    \begin{split}
        0 &= \mW_{00} p_m (1-p_c-p_r) + \left( 1-(1-p_c) p_m \right) \cdot \sum_{l=1}^{T v} \mW_{0l} P_{\vw}(l) + \frac{p_m p_r}{v T} \cdot \sum_{l=1}^{T v} \mW_{0l} \\
        P_{\vw}(i) &= \mW_{i0} p_m (1-p_c-p_r) + \left( 1-(1-p_c) p_m \right) \cdot \left( \sum_{l \in \texttt{topic}(i)} \mW_{il} P_{\vw}(i) + \sum_{l \notin \texttt{topic}(i)} \mW_{il} P_{\vw}(l) \right) + \frac{p_m p_r}{v T} \cdot \sum_{l=1}^{T v} \mW_{il}
    \end{split} 
\end{equation}

Note that under the topic modeling distribution in Section~\ref{sec:setup:topic_modeling}, 
for any topic $t \in [T]$,
\[ P_{\vw}((t-1)v+1) = P_{\vw}((t-1)v+2) = \cdots P_{\vw}(tv) \]

Hence we simplify \eqref{eq:solution_pred} by considering the proportions of the ``representative" tokens for each topic:
\[ \{ P_{\vw}(tv): t \in [T] \} \]
We obtain: for all sets of 
$\{ P_{\vw}(i): i \in [T v] \}$ satisfying our distribution in Section~\ref{sec:setup:topic_modeling}
\begin{equation}
    \label{eq:solution_pred_simplified_0}
        0 = \mW_{00} p_m (1-p_c-p_r) + \left( 1-(1-p_c) p_m \right) \cdot \sum_{t=1}^{T} \sum_{l \in t} \mW_{0l} P_{\vw}(tv) + \frac{p_m p_r}{v T} \cdot \sum_{t=1}^{T} \sum_{l \in t} \mW_{0l} 
\end{equation}
and $\forall i \in \{1, \cdots, T v\}$
\begin{equation}
    \label{eq:solution_pred_simplified_i}
        P_{\vw}(i) = \mW_{i0} p_m (1-p_c-p_r) + \left( 1-(1-p_c) p_m \right) \cdot \biggl( \sum_{l \in \texttt{topic}(i)} \mW_{il} P_{\vw}(i) + \sum_{t \ne \texttt{topic}(i)} \sum_{l \in t} \mW_{il} P_{\vw}(tv) \biggr) + \frac{p_m p_r}{v T} \cdot \sum_{l=1}^{T v} \mW_{il} 
\end{equation}

\begin{claim}
\label{claim:diff_topics_avg}
$\forall i \in \{1, \cdots, T v\}, \exists u_i \in \R$ 
such that $\forall t \ne \texttt{topic}(i), \sum_{l \in t} \mW_{il} = u_i v$.
When $i = 0, \exists u_0 \in \R$ such that $\forall t \in [T], \sum_{l \in t} \mW_{0l} = u_0 v$.
\end{claim}

\begin{proof}
    $\forall i \in \{1, \cdots, T v\}, \exists u_i \in \R$, suppose towards contradiction that $\exists t_1, t_2 \ne \texttt{topic}(i)$ such that $\sum_{l \in t_1} \mW_{il} > \sum_{l \in t_2} \mW_{il}$.
    We will show that \eqref{eq:solution_pred_simplified_i} cannot hold for all sets of $\{ P_{\vw}(i): i \in [T v] \}$ satisfying our distribution in Section~\ref{sec:setup:topic_modeling}.

    Specifically, fix $P_{\vw}(i) = \frac{1}{2v}$ and consider the following settings of $\{ P_{\vw}(j): j \notin \texttt{topic}(i) \}$:
    \begin{itemize}
        \item $P_{\vw}(j) = \frac{1}{2v}$ if $\texttt{topic}(j) = t_1$ and 0 otherwise.
            Then \eqref{eq:solution_pred_simplified_i} becomes 
            \[ \frac{1}{2v} = \mW_{i0} p_m (1-p_c-p_r) + \left( 1-(1-p_c) p_m \right) \cdot \biggl( \sum_{l \in \texttt{topic}(i)} \mW_{il} \frac{1}{2v} + \sum_{l \in t_1} \mW_{il} \frac{1}{2v} \biggr) + \frac{p_m p_r}{v T} \cdot \sum_{l=1}^{T v} \mW_{il} \]
        \item $P_{\vw}(j) = \frac{1}{2v}$ if $\texttt{topic}(j) = t_2$ and 0 otherwise.
            Then \eqref{eq:solution_pred_simplified_i} becomes 
            \[ \frac{1}{2v} = \mW_{i0} p_m (1-p_c-p_r) + \left( 1-(1-p_c) p_m \right) \cdot \biggl( \sum_{l \in \texttt{topic}(i)} \mW_{il} \frac{1}{2v} + \sum_{l \in t_2} \mW_{il} \frac{1}{2v} \biggr) + \frac{p_m p_r}{v T} \cdot \sum_{l=1}^{T v} \mW_{il} \]
    \end{itemize}
    Clearly the above two equations cannot both hold, because $\sum_{l \in t_1} \mW_{il} > \sum_{l \in t_2} \mW_{il}$.
    
    Hence we proved by contradiction that $\forall t_1, t_2 \ne \texttt{topic}(i), \sum_{l \in t_1} \mW_{il} = \sum_{l \in t_2} \mW_{il}$.
    Likewise, when $i = 0$, $\forall t_1, t_2 in [T], \sum_{l \in t_1} \mW_{0l} = \sum_{l \in t_2} \mW_{0l}$.
    
\end{proof}

By Claim~\ref{claim:diff_topics_avg}, \eqref{eq:solution_pred_simplified_0} becomes
\begin{align*}
    0 &= \mW_{00} p_m (1-p_c-p_r) + \left( 1-(1-p_c) p_m \right) \cdot \sum_{t=1}^{T} u_0 v P_{\vw}(tv) + \frac{p_m p_r}{v T} \cdot \sum_{t=1}^{T} u_0 v \\
    &= \mW_{00} p_m (1-p_c-p_r) + \left( 1-(1-p_c) p_m \right) \cdot u_0 + \frac{p_m p_r}{v T} \cdot T u_0 v \\
    &= \mW_{00} p_m (1-p_c-p_r) + \left( 1-(1-p_c) p_m \right) \cdot u_0 + p_m p_r u_0 \\
    &= \mW_{00} p_m (1-p_c-p_r) + \left( 1-(1-p_c-p_r) p_m \right) \cdot u_0
\end{align*}
Therefore 
\[ \mW_{00} = -\frac{\left( 1-(1-p_c-p_r) p_m \right) \cdot u_0}{p_m (1-p_c-p_r)} = -\left(\frac{1}{p_m (1-p_c-p_r)} - 1\right) \cdot u_0 \]

By Claim~\ref{claim:diff_topics_avg}, \eqref{eq:solution_pred_simplified_i} becomes
\begin{align*}
    P_{\vw}(i) &= \mW_{i0} p_m (1-p_c-p_r) + \left( 1-(1-p_c) p_m \right) \cdot \biggl( \sum_{l \in \texttt{topic}(i)} \mW_{il} P_{\vw}(i) + \sum_{t \ne \texttt{topic}(i)} u_i v P_{\vw}(tv) \biggr) \\
    &\quad + \frac{p_m p_r}{v T} \cdot (\sum_{l \in \texttt{topic}(i)} \mW_{il} + (T-1) u_i v) \\
    &= \mW_{i0} p_m (1-p_c-p_r) + \left( 1-(1-p_c) p_m \right) \cdot \biggl( \sum_{l \in \texttt{topic}(i)} \mW_{il} P_{\vw}(i) + u_i (1 - v P_{\vw}(i)) \biggr) \\
    &\quad + \frac{p_m p_r}{v T} \cdot (\sum_{l \in \texttt{topic}(i)} \mW_{il} + (T-1) u_i v) \\
    &= \left( 1-(1-p_c) p_m \right) (\sum_{l \in \texttt{topic}(i)} \mW_{il} - u_i v) P_{\vw}(i) + \mW_{i0} p_m (1-p_c-p_r) + \left( 1-(1-p_c) p_m \right) u_i  \\
    &\quad + \frac{p_m p_r}{v T} \cdot (\sum_{l \in \texttt{topic}(i)} \mW_{il} + (T-1) u_i v) \\
\end{align*}
Since this has to hold for all $P_{\vw}(i) \in [0, \frac{1}{v}]$,
the coefficients must match, i.e.
\begin{align}
    \left( 1-(1-p_c) p_m \right) (\sum_{l \in \texttt{topic}(i)} \mW_{il} - u_i v) &= 1 \label{eq:sum_topic} \\
    \mW_{i0} p_m (1-p_c-p_r) + \left( 1-(1-p_c) p_m \right) u_i + \frac{p_m p_r}{v T} \cdot (\sum_{l \in \texttt{topic}(i)} \mW_{il} + (T-1) u_i v) &= 0 \label{eq:non_topic}
\end{align}
By \eqref{eq:sum_topic},
\[ \sum_{l \in \texttt{topic}(i)} \mW_{il} = u_i v + \frac{1}{1-(1-p_c) p_m} \]
Plugging into \eqref{eq:non_topic},
\begin{align*}
    \mW_{i0} &= -\frac{\left( 1-(1-p_c) p_m \right) u_i + \frac{p_m p_r}{v T} \cdot (u_i v + \frac{1}{1-(1-p_c) p_m} + (T-1) u_i v)}{p_m (1-p_c-p_r)} \\
    &= -\frac{\left( 1-(1-p_c) p_m \right) u_i + \frac{p_m p_r}{v T} \cdot (\frac{1}{1-(1-p_c) p_m} + T u_i v)}{p_m (1-p_c-p_r)} \\
    &= -\frac{\left( 1-(1-p_c) p_m \right) u_i + \frac{p_m p_r}{v T (1-(1-p_c) p_m)} + p_m p_r u_i}{p_m (1-p_c-p_r)} \\
    &= -\frac{p_r}{(1-p_c-p_r) v T (1-(1-p_c) p_m)} - \frac{\left( 1-(1-p_c-p_r) p_m \right)}{p_m (1-p_c-p_r)} u_i \\
    &= -\frac{p_r}{(1-p_c-p_r) (1-(1-p_c) p_m) T v} - \left( \frac{1}{p_m (1-p_c-p_r)} - 1 \right) u_i \\
\end{align*}

\end{proof}

\newpage
\section{PROOF OF THEOREM~\ref{thm:optimal_embedding}: OPTIMAL TOKEN EMBEDDING}
\label{sec:appendix:embedding}

\begin{theorem*}[optimal token embedding, Theorem~\ref{thm:optimal_embedding} restated] 

Consider training a transformer given by \eqref{eq:simplified_transformer_with_emb} with $\mW^K = 0, \mW^Q = 0, \mW^V = I$ and $\forall i \in \{1, \cdots, T v\}, \vb^{\text{pred}}_i = -\frac{p_m p_r}{\left( 1-(1-p_c)p_m \right) T v}$ on data coming from the topic model described in Section~\ref{sec:setup}, with 
the masked language modeling objective (\eqref{eq:objective}) with squared loss (\eqref{eq:squared_loss}). 

Then, the optimal word embeddings $\mW^E$ are such that $\mE \coloneqq {\mW^E}^\top \mW^E$ satisfies:
there exist constants $u_0, \cdots, u_{Tv} \in \R$ such that
\begin{enumerate} %
    \item The 0-th row of $\mE$:
    \begin{enumerate}
        \item $\mE_{00} = -\left(\frac{1}{p_m (1-p_c-p_r)} - 1\right) \cdot u_0$
        \item $\forall t \in [T], \sum_{l \in t} \mE_{0l} = u_0 v$
    \end{enumerate}
    \item The 0-th column of $\mE$:
    \begin{enumerate}
        \item $\forall i \in \{1, \cdots, T v\}, \mE_{i0} = - \left( \frac{1}{(1-p_c-p_r) p_m} - 1 \right) u_i $
    \end{enumerate}
    \item $\mE_{ij}$ ($\forall i, j \in \{1, \cdots, T v\}$):
    \begin{enumerate}
        \item $\sum_{l \in \texttt{topic}(i)} \mE_{il} = \frac{1}{1-(1-p_c)p_m} + u_i v$
        \item $\forall t \in [T]$ such that $\texttt{topic}(i) \ne t$, $\sum_{l \in t} \mE_{il} = u_i v $
    \end{enumerate}
\end{enumerate}
\end{theorem*}

\begin{proof}
Under this setting, the model output is 
\begin{equation} \label{eq:embedding_proof_equivalence}
\begin{split}
    f(\tilde{\mX}) &= {\mW^E}^\top \mW^E \tilde{\mX} A(\mW^E \tilde{\mX}) + \vb^{\text{pred}} \\
    &= \mE \tilde{\mX} \frac{1}{N} \1_{N \times N} + \vb^{\text{pred}} \\
    &= \mE' \tilde{\mX} \frac{1}{N} \1_{N \times N} 
\end{split}
\end{equation}
in which $\1$ refers to the all-one matrix, 
and $\mE' \in \R^{(Tv+1) \times (Tv+1)}$ is defined such that
\[ \mE'_{ij} = \begin{cases}
        \mE_{ij} - \frac{p_r}{(1-p_r-p_c) \left( 1-(1-p_c)p_m \right) T v}, \quad &\text{if } i \in \{1, \cdots, T v\}, j = 0 \\ 
        \mE_{ij}, \quad &\text{otherwise } 
    \end{cases} \]
and the last step is because by \eqref{eq:observed_token_distribution},
\[ \left( \tilde{\mX} \frac{1}{N} \1_{N \times N} \right)_{0j} = p_m (1-p_c-p_r) \quad \forall j \]
and $\forall i \in \{1, \cdots, T v\}$, 
\begin{align*}
    &\quad \left( \mE' \tilde{\mX} \frac{1}{N} \1_{N \times N} \right)_{ij} \\
    &= \left( \mE \tilde{\mX} \frac{1}{N} \1_{N \times N} \right)_{ij} - \frac{p_r}{(1-p_r-p_c) \left( 1-(1-p_c)p_m \right) T v} \cdot p_m (1-p_c-p_r) \\
    &= \left( \mE \tilde{\mX} \frac{1}{N} \1_{N \times N} \right)_{ij} - \frac{p_m p_r}{\left( 1-(1-p_c)p_m \right) T v}  \\
    &= \left( \mE \tilde{\mX} \frac{1}{N} \1_{N \times N} \right)_{ij} + \vb^{\text{pred}}_i  
\end{align*}

Let $\mE'^*$ denote any matrix in 
\[ \argmin_{\mE'} \E_{\mX \sim \mathcal{D}_\mX} \E_{M} \frac{1}{|M|} \sum_{j \in M} \| (\mE' \tilde{\mX} \frac{1}{N} \1_{N \times N})_{:j} - \mX_{:j} \|_2^2 \]
then by Lemma~\ref{lemma:optimal_linear_transform_given_uniform_attention},
{
there exist constants $u_0, \cdots, u_{Tv} \in \R$ such that
\begin{enumerate} %
    \item The 0-th row of $\mE'^*$:
    \begin{enumerate}
        \item $\mE'^*_{00} = -\left(\frac{1}{p_m (1-p_c-p_r)} - 1\right) \cdot u_0$
        \item $\forall t \in [T], \sum_{l \in t} \mE'^*_{0l} = u_0 v$
    \end{enumerate}
    \item The 0-th column of $\mE'^*$:
    \begin{enumerate}
        \item $\forall i \in \{1, \cdots, T v\}, \mE'^*_{i0} = -\frac{p_r}{(1-p_c-p_r) \left(1-(1-p_c)p_m \right) T v} - \left( \frac{1}{(1-p_c-p_r) p_m} - 1 \right) u_i $
    \end{enumerate}
    \item $\mE'^*_{ij}$ ($\forall i, j \in \{1, \cdots, T v\}$):
    \begin{enumerate}
        \item $\sum_{l \in \texttt{topic}(i)} \mE'^*_{il} = \frac{1}{1-(1-p_c)p_m} + u_i v$
        \item $\forall t \in [T]$ such that $\texttt{topic}(i) \ne t$, $\sum_{l \in t} \mE'^*_{il} = u_i v $
    \end{enumerate}
\end{enumerate}
}

Therefore, by \eqref{eq:embedding_proof_equivalence},
let $\mE^*$ denote any matrix in 
\[ \argmin_{\mE} \E_{\mX \sim \mathcal{D}_\mX} \E_{M} \frac{1}{|M|} \sum_{j \in M} \| (\mE \tilde{\mX} \frac{1}{N} \1_{N \times N})_{:j} + \vb^{\text{pred}} - \mX_{:j} \|_2^2 \]
then 
{
there exist constants $u_0, \cdots, u_{Tv} \in \R$ such that
\begin{enumerate} %
    \item The 0-th row of $\mE^*$:
    \begin{enumerate}
        \item $\mE^*_{00} = -\left(\frac{1}{p_m (1-p_c-p_r)} - 1\right) \cdot u_0$
        \item $\forall t \in [T], \sum_{l \in t} \mE^*_{0l} = u_0 v$
    \end{enumerate}
    \item The 0-th column of $\mE^*$:
    \begin{enumerate}
        \item $\forall i \in \{1, \cdots, T v\}, \mE^*_{i0} = - \left( \frac{1}{(1-p_c-p_r) p_m} - 1 \right) u_i $
    \end{enumerate}
    \item $\mE^*_{ij}$ ($\forall i, j \in \{1, \cdots, T v\}$):
    \begin{enumerate}
        \item $\sum_{l \in \texttt{topic}(i)} \mE^*_{il} = \frac{1}{1-(1-p_c)p_m} + u_i v$
        \item $\forall t \in [T]$ such that $\texttt{topic}(i) \ne t$, $\sum_{l \in t} \mE^*_{il} = u_i v $
    \end{enumerate}
\end{enumerate}
}

Finally, note that a subset of this family of optima is \emph{realizable}, 
in the sense that there exists such $\mE^*$ and $u_0, \cdots, u_{Tv} \in \R$ s.t.
there exists $\mW^E \in \R^{d \times (Tv+1)}$ s.t.
$\mE^* = {\mW^E}^\top \mW^E$.
The simplest example is 
\begin{align*}
    u_0, \cdots, u_{Tv} &= 0 \\
    d &= Tv+1 \\
    \mE^* &= \frac{1}{1-(1-p_c)p_m} I \\
    \mW^E &= \frac{1}{\sqrt{1-(1-p_c)p_m}} I 
\end{align*}

\end{proof}

\newpage
\clearpage
\section{PROVING OPTIMAL \texorpdfstring{$\mW^V$}{Wv} IN SELF-ATTENTION}
\label{sec:appendix:Wv}

\subsection{Optimal \texorpdfstring{$\mW^V$}{Wv} when freezing uniform attention without regularization}

\begin{theorem}[optimal $\mW^V$ when freezing uniform attention] 
    \label{thm:optimal_Wv_given_uniform_attention}
    On the topic modeling data distribution described in Section~\ref{sec:setup:topic_modeling},
    with the \texttt{topic} relation defined in Definition~\ref{def:topic_word_belonging},
    under Assumption~\ref{assumption:infinitely_long_document},
    with a single layer transformer given by \eqref{eq:simplified_transformer_no_emb} whose $\mW^K = 0, \mW^Q = 0, \vb^{\text{pred}} = 0$,
    under masked language modeling objective (\eqref{eq:objective}) with the squared loss (\eqref{eq:squared_loss}),
    $\argmin L(\mW^V)$ consists of all $\mW^V \in \R^{(T v + 1) \times (T v + 1)}$ that satisfy:
    there exist constants $u_0, \cdots, u_{Tv} \in \R$ such that
    \begin{enumerate} %
        \item The 0-th row of $\mW^V$:
        \begin{enumerate}
            \item $\mW^V_{00} = -\left(\frac{1}{p_m (1-p_c-p_r)} - 1\right) \cdot u_0$
            \item $\forall t \in [T], \sum_{l \in t} \mW^V_{0l} = u_0 v$
        \end{enumerate}
        \item The 0-th column of $\mW^V$:
        \begin{enumerate}
            \item $\forall i \in \{1, \cdots, T v\}, \mW^V_{i0} = -\frac{p_r}{(1-p_c-p_r) \left(1-(1-p_c)p_m \right) T v} - \left( \frac{1}{(1-p_c-p_r) p_m} - 1 \right) u_i $
        \end{enumerate}
        \item $\mW^V_{ij}$ ($\forall i, j \in \{1, \cdots, T v\}$):
        \begin{enumerate}
            \item $\sum_{l \in \texttt{topic}(i)} \mW^V_{il} = \frac{1}{1-(1-p_c)p_m} + u_i v$
            \item $\forall t \in [T]$ such that $\texttt{topic}(i) \ne t$, $\sum_{l \in t} \mW^V_{il} = u_i v $
        \end{enumerate}
    \end{enumerate}
\end{theorem}

\begin{proof}
Note that this is exactly the statement of Lemma~\ref{lemma:optimal_linear_transform_given_uniform_attention} (proved in Appendix~\ref{sec:appendix:optimal_linear_transform})
in the case of $\mW \coloneqq \mW^V$.
\end{proof}

\subsection{Proof of Theorem~\ref{thm:optimal_Wv_given_uniform_attention_l2reg}: case when adding \texorpdfstring{$L_2$}{L2} regularization}

\begin{theorem*}[optimal $\mW^V$ with mild $L_2$-regularization when freezing uniform attention, restated] 
    On the topic modeling data distribution described in Section~\ref{sec:setup:topic_modeling},
    with the \texttt{topic} relation defined in Definition~\ref{def:topic_word_belonging},
    under Assumption~\ref{assumption:infinitely_long_document},
    with a single layer transformer given by \eqref{eq:simplified_transformer_no_emb} whose $\mW^K = 0, \mW^Q = 0, \vb^{\text{pred}} = 0$,
    under the $L_2$-regularized masked language modeling objective (\eqref{eq:objective_l2reg}) with the squared loss (\eqref{eq:squared_loss}),
    $\lim_{\lambda \to 0} \argmin L_{\text{l2reg}}(\mW^V) = \{ \mW^{V*} \}$ in which $\mW^{V*} \in \R^{(T v + 1) \times (T v + 1)}$ satisfies:
    \begin{enumerate} %
        \item The 0-th row of $\mW^{V*}$:
        \begin{enumerate}
            \item $\forall j \in \{0, \cdots, T v\}, \mW^{V*}_{0j} = 0$
        \end{enumerate}
        \item The 0-th column of $\mW^{V*}$:
        \begin{enumerate}
            \item $\forall i \in \{1, \cdots, T v\}, \mW^{V*}_{i0} = \frac{c_2 c_3 - c_1 T v}{c_2^2 + T v}$
        \end{enumerate}
        \item $\mW^{V*}_{ij}$ ($\forall i, j \in \{1, \cdots, T v\}$):
        \begin{enumerate}
            \item $\forall l \notin \texttt{topic}(i), \; \mW^{V*}_{il} = \mW^{V*}_{\text{diff-topic}} \coloneqq -\frac{c_1 c_2 + c_3}{c_2^2 + T v} $
            \item $\forall l \in \texttt{topic}(i), \; \mW^{V*}_{il} = \mW^{V*}_{\text{same-topic}} \coloneqq \mW^{V*}_{\text{diff-topic}} + \frac{c_3}{v}$
        \end{enumerate}
    \end{enumerate}
    in which the constants
    \begin{itemize}
        \item $c_1 = \frac{p_r}{(1-p_c-p_r) \left(1-(1-p_c)p_m \right) T v} \in (0, 1)$
        \item $c_2 = \frac{1}{(1-p_c-p_r) p_m} - 1 \in (0, +\infty)$
        \item $c_3 = \frac{1}{1-(1-p_c)p_m} \in (1, +\infty)$
    \end{itemize}
\end{theorem*}

\begin{proof}

We proceed in the following two steps.

\textbf{Step 1: the optima converges to one outlined in Lemma~\ref{lemma:optimal_linear_transform_given_uniform_attention}} 

Let $S$ denote the set of optima outlined in Lemma~\ref{lemma:optimal_linear_transform_given_uniform_attention}.
Suppose towards contradiction that $\exists \mW^{V*} \notin S$ such that $\mW^{V*} \in \lim_{\lambda \to 0} \argmin L_{\text{l2reg}}(\mW^V)$.

In comparison, $\forall \mW \in S$, by Lemma~\ref{lemma:optimal_linear_transform_given_uniform_attention},
since $\mW^{V*} \notin S$,
\[ L(\mW) < L(\mW^{V*})  \]

Moreover, note that since $\| \mW  \|_F$ is finite, 
\[ \lim_{\lambda \to 0} \lambda \| \mW  \|_F^2 = 0 \le \lim_{\lambda \to 0} \lambda \| \mW^{V*}  \|_F^2  \]

Combining the above two observations gives
\[ \lim_{\lambda \to 0} L_{\text{l2reg}}(\mW) 
= L(\mW) + \lim_{\lambda \to 0} \lambda \| \mW  \|_F^2
< L(\mW^{V*}) + \lim_{\lambda \to 0} \lambda \| \mW^{V*} \|_F^2
= \lim_{\lambda \to 0} L_{\text{l2reg}}(\mW^{V*}) 
\]
which contradicts $\mW^{V*} \in \lim_{\lambda \to 0} \argmin L_{\text{l2reg}}(\mW^V)$.

Therefore, we have proved by contradiction that 
\[ \forall \mW^V \in \lim_{\lambda \to 0} \argmin L_{\text{l2reg}}(\mW^V), \quad \mW^V \in S \]

\textbf{Step 2: solve for the coefficients that minimize the $L_2$ penalty} 

By Step 1, 
\begin{align*}
    \lim_{\lambda \to 0} \argmin L_{\text{l2reg}}(\mW^V) &= \lim_{\lambda \to 0} \argmin_{\mW^V \in S} L_{\text{l2reg}}(\mW^V) \\
    &= \lim_{\lambda \to 0} \argmin_{\mW^V \in S} L(\mW^{V}) + \lambda \| \mW^{V} \|_F^2 \\
    &= \lim_{\lambda \to 0} \argmin_{\mW^V \in S} \min L(\mW^{V}) + \lambda \| \mW^{V} \|_F^2 \\
    &= \lim_{\lambda \to 0} \argmin_{\mW^V \in S} \lambda \| \mW^{V} \|_F^2 \\
\end{align*}
in which the last step is because $\forall \mW^V \in S$, $L(\mW^{V}) = \min L(\mW^{V})$, 
which is a constant independent of $\mW^V$.

Then it suffices to find the constants $u_0, \cdots, u_{Tv} \in \R$ that minimizes $\| \mW^{V} \|_F$.

\end{proof}

\newpage
\section{ADDITIONAL RESULTS ON ATTENTION WEIGHTS}
\label{sec:appendix:attention}

\subsection{Helping lemmas on masking probabilities}
\label{sec:appendix:masking_prob_calculation}

In this section, we will calculate a few expressions for the masking probabilities, which will be useful for the proofs later on. We will also introduce a few constants for brevity of notation.

A straightforward calculation shows that the probabilities after the masking process satisfy: 
\begin{proposition}[Probabilities after masking] After the masking process as in Section~\ref{sec:setup:topic_modeling} is applied to a document $\vw$, the distribution for the new document $\tilde{\vw}$ satisfies  
\begin{equation}
    \label{eq:observed_token_distribution}
    P_{\tilde{\vw}}(i) = \begin{cases}
         \frac{1}{v \tau} (1-(1-p_c) p_m) + \frac{p_m p_r}{v T}, \quad &\text{if } \texttt{topic}(i) \in \{t_1, \cdots, t_\tau\}  \\ 
        p_m (1-p_c-p_r) , \quad &\text{if } i = \texttt{[MASK]} \coloneqq 0 \\
        \frac{p_m p_r}{v T}, \quad &\text{otherwise } 
        \end{cases}
\end{equation}
\end{proposition}

For convenience, we will introduce the notation 
\begin{equation}
    \label{eq:p_1}
    p_1 := \frac{1}{v \tau} (1-(1-p_c) p_m) + \frac{p_m p_r}{v T}
\end{equation}
\begin{equation}
    \label{eq:p_2}
    p_2 :=\frac{p_m p_r}{v T}
\end{equation}

Another straightforward calculation can be used to express the relationship between the constant $c_3$ in Assumption~\ref{assumption:alpha_beta_attention} and the $\alpha, \beta$. Namely, we have: 
\begin{proposition}[Expressing $c_3$ in terms of $\alpha, \beta$] 
\label{proposition:calculate_c_3}
The constant $c_3$ in Assumption~\ref{assumption:alpha_beta_attention} satisfies:
\[ c_3 = \begin{cases} 
        \frac{1}{(\beta p_1 + \alpha p_1 (v-1) + p_1 v (\tau-1) + p_2 v (T-\tau)) N}, \quad &\text{if } \tilde{w}_j \in \{t_i\}_{i \in [\tau]}  \\ 
        \frac{1}{(\beta p_2 + \alpha p_2 (v-1) + p_1 v \tau + p_2 v (T-\tau-1)) N}, \quad &\text{if } \tilde{w}_j \in [T] \backslash \{t_i\}_{i \in [\tau]}
    \end{cases}\]
\label{p:c3ab}
\end{proposition}

Again, for notational convenience, we will introduce $z_1, z_2$, s.t.  
\begin{align*}
    z_1 &:= \beta p_1 + \alpha p_1 (v-1) + p_1 v (\tau-1) + p_2 v (T-\tau)\\ 
    z_2 &:= \beta p_2 + \alpha p_2 (v-1) + p_1 v \tau + p_2 v (T-\tau-1) 
\end{align*}

\begin{proof}[Proof of Proposition \ref{p:c3ab}]

We will get these equalities by considering the marginalization constraints, depending on the topic of $\tilde{w}_j$. 
Consider first a $j$, such that $\tilde{w}_j \in \{t_i\}_{i \in [\tau]}$: %
\begin{itemize}
    \item Note, for every position $i$, with probability $p_1$, we have $\tilde{w}_i = \tilde{w}_j$, so $A(\tilde{\mX})_{ij} = \beta c_3$ by Assumption~\ref{assumption:alpha_beta_attention}. 
    \item Note also, for every position $i$, with probabilitiy $p_1 (v-1)$, we have  $\tilde{w}_i \ne \tilde{w}_j$ but $\texttt{topic}(\tilde{w}_i) = \texttt{topic}(\tilde{w}_j)$, and so $A(\tilde{\mX})_{ij} = \alpha c_3$.
    \item Finally, note that for every position $i$, with probability  $(p_1 v (\tau-1) + p_2 v (T-\tau))$ we have $\texttt{topic}(\tilde{w}_i) \ne \texttt{topic}(\tilde{w}_j)$, so $A(\tilde{\mX})_{ij} = c_3$.
\end{itemize}
Since $\sum_{i=1}^N A(\tilde{\mX})_{ij} = 1$,
we obtain $c_3 = \frac{1}{(\beta p_1 + \alpha p_1 (v-1) + p_1 v (\tau-1) + p_2 v (T-\tau)) N}$.

Consider next a $j$, s.t.  $\tilde{w}_j \in [T] \backslash \{t_i\}_{i \in [\tau]}$. By similar considerations as before, 
\begin{itemize}
    \item With probability $p_2$, a position $i$ in $\tilde{\vw}$ satisfies $\tilde{w}_i = \tilde{w}_j$, so $A(\tilde{\mX})_{ij} = \beta c_3$.
    \item With probability $p_2 (v-1)$, a position $i$ satisfies $\tilde{w}_i \ne \tilde{w}_j$ but $\texttt{topic}(\tilde{w}_i) = \texttt{topic}(\tilde{w}_j)$, so $A(\tilde{\mX})_{ij} = \alpha c_3$.
    \item Finally, with probability $p_1 v \tau + p_2 v (T-\tau-1)$, a  position $i$ in satisfies  $\texttt{topic}(\tilde{w}_i) \ne \texttt{topic}(\tilde{w}_j)$, so $A(\tilde{\mX})_{ij} = c_3$.
\end{itemize}
Since $\sum_{i=1}^N A(\tilde{\mX})_{ij} = 1$,
we obtain $c_3 = \frac{1}{(\beta p_2 + \alpha p_2 (v-1) + p_1 v \tau + p_2 v (T-\tau-1)) N}$. The proposition thus follows.

\end{proof}

\subsection{Implication of topic-wise attention assumption on model output}
In this section we calculate the part $\tilde{\mX} A(\tilde{\mX})$ using the results of Appendix~\ref{sec:appendix:masking_prob_calculation}:

\begin{proposition}
Using the calculation and notations of $z_1$ and $z_2$ in Appendix~\ref{sec:appendix:masking_prob_calculation}:
\begin{equation}
    \label{eq:pre_Wv_weighted_sum}
    \begin{split}
    &\quad \left[ \tilde{\mX} A(\tilde{\mX}) \right]_{ij} =  \sum_{l=1}^N \tilde{\mX}_{il} A(\tilde{\mX})_{lj} 
    = \sum_{l=1}^N \1_{\tilde{\mX}_{il} = 1} A(\tilde{\mX})_{lj}
    = P(\tilde{\mX}_{il} = 1) \cdot N \cdot A(\tilde{\mX})_{lj}  \\
    &= \begin{cases}
        p_1 N \frac{\beta}{z_1 N} = \frac{p_1 \beta}{z_1}, \quad &\text{if } i = j, \texttt{topic}(j) \in \{t_1, \cdots, t_\tau\} \text{\textbf{(Same token)}}  \\
        p_1 N \frac{\alpha}{z_1 N} = \frac{p_1 \alpha}{z_1}, \quad &\text{if } i \ne j, \texttt{topic}(i) = \texttt{topic}(j) \in \{t_1, \cdots, t_\tau\}\text{\textbf{(Different token, same topic)}}  \\
        p_1 N \frac{1}{z_1 N} = \frac{p_1}{z_1}, \quad &\text{if } \texttt{topic}(i) \ne \texttt{topic}(j), \texttt{topic}(i) \in \{t_1, \cdots, t_\tau\}, \texttt{topic}(j) \in \{t_1, \cdots, t_\tau\}  \\
        p_2 N \frac{1}{z_1 N} = \frac{p_2}{z_1}, \quad &\text{if } \texttt{topic}(i) \ne \texttt{topic}(j), \texttt{topic}(i) \notin \{t_1, \cdots, t_\tau\}, \texttt{topic}(j) \in \{t_1, \cdots, t_\tau\}  \\
        p_2 N \frac{\beta}{z_2 N} = \frac{p_2 \beta}{z_2}, \quad &\text{if } i = j, \texttt{topic}(j) \notin \{t_1, \cdots, t_\tau\}  \\
        p_2 N \frac{\alpha}{z_2 N} = \frac{p_2 \alpha}{z_2}, \quad &\text{if } i \ne j, \texttt{topic}(i) = \texttt{topic}(j) \notin \{t_1, \cdots, t_\tau\}  \\
        p_1 N \frac{1}{z_2 N} = \frac{p_1}{z_2}, \quad &\text{if } \texttt{topic}(i) \ne \texttt{topic}(j), \texttt{topic}(i) \in \{t_1, \cdots, t_\tau\}, \texttt{topic}(j) \notin \{t_1, \cdots, t_\tau\}  \\
        p_2 N \frac{1}{z_2 N} = \frac{p_2}{z_2}, \quad &\text{if } \texttt{topic}(i) \ne \texttt{topic}(j), \texttt{topic}(i) \notin \{t_1, \cdots, t_\tau\}, \texttt{topic}(j) \notin \{t_1, \cdots, t_\tau\}  \\
    \end{cases} \\
    \end{split}
\end{equation}
\end{proposition}

\clearpage
\subsection{Proof of Theorem~\ref{thm:optimal_attention_weights_updated} (optimal attention when freezing \texorpdfstring{$\mW^V$}{Wv} to uniform blocks)}
\label{sec:appendix:attention:weights:block}

\begin{theorem*}[optimal attention weights when freezing block-wise \texorpdfstring{$\mW^V$}{Wv}, Theorem~\ref{thm:optimal_attention_weights_updated} restated]
    Suppose the data distribution follows the topic modeling assumption in Section~\ref{sec:setup:topic_modeling} and Assumption~\ref{assumption:infinitely_long_document}.
    Suppose we train a single layer transformer given by \eqref{eq:simplified_transformer_no_emb}
    with $\vb^{\text{pred}} = 0$ and $\mW^V$ frozen to the optima in Theorem~\ref{thm:optimal_Wv_given_uniform_attention_l2reg},
    under masked language modeling objective (\eqref{eq:objective}) 
    with the squared loss (\eqref{eq:squared_loss}),
    under Assumption~\ref{assumption:alpha_beta_attention}, Assumption~\ref{assumption:asymptotic}, and Assumption~\ref{assumption:correct_random_balance}.
    Then, the optimal $(\alpha, \beta)$ satisfy
    \[ \frac{v-1}{v} \alpha + \frac{1}{v} \beta \in (\lambda_1 (\tau - 1), \lambda_2 T ) \]
    in which the constants $\lambda_1 \coloneqq \frac{(1-(1-p_c) p_m + p_m p_r) (1+(1-p_c) p_m)}{2 (1-(1-p_c) p_m)}$ and $\lambda_2 \coloneqq 100 (\frac{1-(1-p_c) p_m}{p_m p_r} + 1)$.
\end{theorem*}

\begin{proof}

Define $\gamma \coloneqq \frac{v-1}{v} \alpha + \frac{1}{v} \beta$.

Recall the architecture under consideration, i.e.
\[ \hat{\mX} \coloneqq \mW^V \tilde{\mX} A(\tilde{\mX}) \]

The squared loss (\eqref{eq:squared_loss}) is
\begin{align*}
    &\quad \E_{\mX \sim \mathcal{D}_\mX} \E_{M} \left[ \frac{1}{|M|} \sum_{j \in M} l(f(\tilde{\mX})_{:j}, \mX_{:j} ) \right]_{ij} \\
    &= \frac{1}{p_m N} \E_{\mX \sim \mathcal{D}_\mX} \E_{M} \left[ \sum_{j: \tilde{w}_j = \texttt{[MASK]}} l(f(\tilde{\mX})_{:j}, \mX_{:j} ) + \sum_{j \in M, \tilde{w}_j \ne \texttt{[MASK]} } l(f(\tilde{\mX})_{:j}, \mX_{:j} ) \right]  \\
    &= \frac{1}{p_m N} \E_{\mX \sim \mathcal{D}_\mX} \E_{M} \left[ \sum_{j: \tilde{w}_j = \texttt{[MASK]}} \| (\mW^V \tilde{\mX} A(\tilde{\mX}))_{:j} - \mX_{:j} \|_2^2 + \sum_{j \in M, \tilde{w}_j \ne \texttt{[MASK]} } \| (\mW^V \tilde{\mX} A(\tilde{\mX}))_{:j} - \mX_{:j} \|_2^2 \right]  \\
    &= \frac{1}{p_m N} \E_{\mX \sim \mathcal{D}_\mX} \E_{M} \left[ \sum_{j: \tilde{w}_j = \texttt{[MASK]}} \| \mW^V \tilde{\mX} A(\tilde{\mX})_{:j} - \mX_{:j} \|_2^2 + \sum_{j \in M, \tilde{w}_j \ne \texttt{[MASK]} } \| \mW^V \tilde{\mX} A(\tilde{\mX})_{:j} - \mX_{:j} \|_2^2 \right]
\end{align*}

Note that when $\tilde{w}_j = \texttt{[MASK]}$,
$A(\tilde{\mX})_{:j}$ is the attention from \texttt{[MASK]} to other tokens,
and therefore is independent of the setting of $\alpha$ and $\beta$ in Assumption~\ref{assumption:alpha_beta_attention}.
Thus, in the following, we only consider the case in which $j \in M, \tilde{w}_j \ne \texttt{[MASK]}$,
namely, $w_j$ is masked, but $\tilde{w}_j$ is chosen to be either the correct token or the random token.
Hence define:
\begin{equation}
    \label{eq:l_gamma}
    L(\gamma) \coloneqq \frac{1}{p_m N} \E_{\mX \sim \mathcal{D}_\mX} \E_{M} \left[ \sum_{j \in M, \tilde{w}_j \ne \texttt{[MASK]} } \| \mW^V \tilde{\mX} A(\tilde{\mX})_{:j} - \mX_{:j} \|_2^2 \right]
\end{equation}

Note that $\forall \vy \in \R^{T v + 1}$
\[ (\mW^V \vy)_i = \begin{cases}
        0, \quad & i = 0 \\ 
        q(\frac{1}{v} \sum_{l \in \text{topic}(i)} y_l), \quad & i \in \{ 1, \cdots, T v + 1 \}
    \end{cases}  \]
in which 
\[ q(x) \coloneqq \frac{1}{1-(1-p_c) p_m} x - \frac{p_m p_r}{(1-(1-p_c)p_m)Tv} \]

In our context, we will consider $\vy = \tilde{\mX} A(\tilde{\mX})_{:j}$ in $L(\gamma)$ above.

For a document $\vw$ which contains topics $t_1, \cdots, t_\tau \in [T]$,
there are the following cases:

\paragraph{Case 1: $\texttt{topic}(\tilde{w}_j) = \texttt{topic}(w_j)$}
When $\tilde{w}_j$ after masking belongs to the same topic as the correct token $w_j$. 
(This happens with probability $p_c + \frac{p_r}{T}$)

By \eqref{eq:pre_Wv_weighted_sum},
\begin{equation}
    \label{eq:pre_Wv_weighted_sum_case_majority_block}
    \frac{1}{v} \sum_{l \in \text{topic}(i)} \left[ \tilde{\mX} A(\tilde{\mX}) \right]_{lj} = \begin{cases}
        &\frac{1}{v} \left( \frac{p_1 \beta}{z_1} + \sum_{l \in \text{topic}(i), l \ne \tilde{w}_j} \frac{p_1 \alpha}{z_1} \right) = \frac{1}{v} \left( \frac{p_1 \beta + (v-1) p_1 \alpha}{z_1} \right) = \frac{p_1 \gamma}{z_1}, \quad \text{if } \texttt{topic}(i) = \texttt{topic}(\tilde{w}_j)  \\
        &\frac{1}{v} \left( \sum_{l \in \text{topic}(i)} \frac{p_1}{z_1} \right) = \frac{p_1}{z_1}, \quad \text{if } \texttt{topic}(i) \ne \texttt{topic}(\tilde{w}_j), \texttt{topic}(i) \in \{t_1, \cdots, t_\tau\}  \\
        &\frac{1}{v} \left( \sum_{l \in \text{topic}(i)} \frac{p_2}{z_1} \right) = \frac{p_2}{z_1}, \quad \text{if } \texttt{topic}(i) \ne \texttt{topic}(\tilde{w}_j), \texttt{topic}(i) \notin \{t_1, \cdots, t_\tau\}
    \end{cases}
\end{equation}

Recall that the label is
\[ \mX_{:j} = \begin{cases}
        1, \quad & i = w_j  \\
        0, \quad & i \in \{0, \cdots, T v\} \backslash w_j
    \end{cases} \]

Hence the contribution to the loss from token $\tilde{w}_j$ is
\begin{align*}
    &\quad (p_c + \frac{p_r}{T}) [\left(1-q\left(\frac{p_1 \gamma}{z_1}\right)\right)^2 + q(\frac{p_1 \gamma}{z_1})^2 \cdot (v-1) + q(\frac{p_1}{z_1})^2 \cdot v (\tau-1) + q(\frac{p_2}{z_1})^2 \cdot v (T-\tau) ] \\
    &=(p_c + \frac{p_r}{T}) [ (1-q(\frac{p_1 \gamma}{p_1 \gamma v + p_1 v (\tau-1) + p_2 v (T-\tau)}))^2 \\
    &+ q(\frac{p_1 \gamma}{p_1 \gamma v + p_1 v (\tau-1) + p_2 v (T-\tau)})^2 (v-1) + q(\frac{p_1}{p_1 \gamma v + p_1 v (\tau-1) + p_2 v (T-\tau)})^2 v (\tau - 1) \\
    &+ q(\frac{p_2}{p_1 \gamma v + p_1 v (\tau-1) + p_2 v (T-\tau)})^2 v (T-\tau) ]
\end{align*}

Plugging in the asymptotics from Assumption~\ref{assumption:asymptotic}, 
the above becomes 
\begin{equation}
    \label{eq:case1l_block} 
   p_c [ (1-q(\frac{p_1 \gamma}{p_1 \gamma v + p_1 v (\tau-1) + p_2 v (T-\tau)}))^2 + q(\frac{p_1 \gamma}{p_1 \gamma v + p_1 v (\tau-1) + p_2 v (T-\tau)})^2 (v-1) ] \pm O(\frac{1}{T})
\end{equation}

\paragraph{Case 2: $\texttt{topic}(\tilde{w}_j) \in \{t_1, \cdots, t_\tau\} \backslash \{\texttt{topic}(w_j)\}$}
When $\tilde{w}_j$ after masking belongs to a different topic from that of the correct token $w_j$, but still a topic existing in $\vw$. 
(This happens with probability $\frac{p_r (\tau - 1)}{T}$)

$\left[ \tilde{\mX} A(\tilde{\mX}) \right]_{:j}$ is the same as \eqref{eq:pre_Wv_weighted_sum_case_majority_block}.

Hence the loss is
\begin{align*}
    &\quad p_r \frac{\tau-1}{T} [(1-q(\frac{p_1}{z_1}))^2 + q(\frac{p_1 \gamma}{z_1})^2 \cdot v + q(\frac{p_1}{z_1})^2 \cdot (v (\tau-1) - 1) + q(\frac{p_2}{z_1})^2 \cdot v (T-\tau)] \\
    &= p_r \frac{\tau-1}{T} [(1-q(\frac{p_1}{p_1 \gamma v + p_1 v (\tau-1) + p_2 v (T-\tau)})^2 + q(\frac{p_1 \gamma}{p_1 \gamma v + p_1 v (\tau-1) + p_2 v (T-\tau)})^2 \cdot v \\
    &\quad + q(\frac{p_1}{p_1 \gamma v + p_1 v (\tau-1) + p_2 v (T-\tau)})^2 (v (\tau-1) - 1) \\
    &+ q(\frac{p_2}{p_1 \gamma v + p_1 v (\tau-1) + p_2 v (T-\tau)})^2 v (T-\tau) ] \\
\end{align*}

Plugging in the asymptotics from Assumption~\ref{assumption:asymptotic}, 
the above terms vanish.

\paragraph{Case 3: $\texttt{topic}(\tilde{w}_j) \in [T] \backslash \{t_1, \cdots, t_\tau\}$}
When $\tilde{w}_j$ after masking belongs to a topic that does not exist in $\vw$. 
(This happens with probability $p_r (1-\frac{\tau}{T})$)

By \eqref{eq:pre_Wv_weighted_sum},
\begin{equation}
    \label{eq:pre_Wv_weighted_sum_case_minority}
    \left[ \tilde{\mX} A(\tilde{\mX}) \right]_{ij} = \begin{cases}
        \frac{p_2 \beta}{z_2}, \quad &\text{if } i = \tilde{w}_j  \\
        \frac{p_2 \alpha}{z_2}, \quad &\text{if } i \ne \tilde{w}_j, \texttt{topic}(i) = \texttt{topic}(\tilde{w}_j)  \\
        \frac{p_1}{z_2}, \quad &\text{if } \texttt{topic}(i) \ne \texttt{topic}(\tilde{w}_j), \texttt{topic}(i) \in \{t_1, \cdots, t_\tau\}  \\
        \frac{p_2}{z_2}, \quad &\text{if } \texttt{topic}(i) \ne \texttt{topic}(\tilde{w}_j), \texttt{topic}(i) \notin \{t_1, \cdots, t_\tau\}
    \end{cases}
\end{equation}

\begin{equation}
    \label{eq:pre_Wv_weighted_sum_case_minority_block}
    \frac{1}{v} \sum_{l \in \text{topic}(i)} \left[ \tilde{\mX} A(\tilde{\mX}) \right]_{lj} = \begin{cases}
        &\frac{1}{v} \left( \frac{p_2 \beta}{z_2} + \sum_{l \in \text{topic}(i), l \ne \tilde{w}_j} \frac{p_2 \alpha}{z_2} \right) = \frac{1}{v} \left( \frac{p_2 \beta + (v-1) p_2 \alpha}{z_2} \right) = \frac{p_2 \gamma}{z_2}, \quad \text{if } \texttt{topic}(i) = \texttt{topic}(\tilde{w}_j)  \\
        &\frac{1}{v} \left( \sum_{l \in \text{topic}(i)} \frac{p_1}{z_2} \right) = \frac{p_1}{z_2}, \quad \text{if } \texttt{topic}(i) \ne \texttt{topic}(\tilde{w}_j), \texttt{topic}(i) \in \{t_1, \cdots, t_\tau\}  \\
        &\frac{1}{v} \left( \sum_{l \in \text{topic}(i)} \frac{p_2}{z_2} \right) = \frac{p_2}{z_2}, \quad \text{if } \texttt{topic}(i) \ne \texttt{topic}(\tilde{w}_j), \texttt{topic}(i) \notin \{t_1, \cdots, t_\tau\} \\
    \end{cases}
\end{equation}

Hence the loss is
\begin{align*}
    &\quad p_r (1-\frac{\tau}{T}) [(1-q(\frac{p_1}{z_2}))^2 + q(\frac{p_1}{z_2})^2 \cdot (v \tau - 1) + q(\frac{p_2 \gamma}{z_2})^2 \cdot v + q(\frac{p_2}{z_2})^2 \cdot v (T-\tau - 1)] \\
    &= p_r (1-\frac{\tau}{T}) [(1-q(\frac{p_1}{p_2 \gamma v + p_1 v \tau + p_2 v (T-\tau-1)}))^2 + q(\frac{p_1}{p_2 \gamma v + p_1 v \tau + p_2 v (T-\tau-1)})^2 (v \tau - 1) \\
    &+ q(\frac{p_2 \gamma}{p_2 \gamma v + p_1 v \tau + p_2 v (T-\tau-1)})^2 v + q(\frac{p_2}{p_2 \gamma v + p_1 v \tau + p_2 v (T-\tau-1)})^2 v (T-\tau-1)]
\end{align*}

Plugging in the asymptotics from Assumption~\ref{assumption:asymptotic}, 
the above becomes 
\begin{equation}
    \label{eq:case4l_block}
    p_r (1 + q(\frac{p_2 \gamma}{p_2 \gamma v + p_1 v \tau + p_2 v (T-\tau-1)})^2 )
\end{equation}

\paragraph{Combining the above cases}
Adding \eqref{eq:case1l_block} and \eqref{eq:case4l_block}, we can see in the asymptotic regime of interest, we have: 
\begin{align*}
    L(\gamma) &= p_c [ (1-q(\frac{p_1 \gamma}{p_1 \gamma v + p_1 v (\tau-1) + p_2 v (T-\tau)}))^2 + q(\frac{p_1 \gamma}{p_1 \gamma v + p_1 v (\tau-1) + p_2 v (T-\tau)})^2 (v-1) ] \\
    &\quad + p_r (1 + q(\frac{p_2 \gamma}{p_1 \gamma v + p_1 v \tau + p_2 v (T-\tau-1)})^2 ) \pm O(\frac{1}{T}) \\
    &= p_c [ (1-\frac{c_4 p_1 \gamma}{p_1 \gamma v + p_1 v (\tau-1) + p_2 v (T-\tau)})^2 + (\frac{c_4 p_1 \gamma}{p_1 \gamma v + p_1 v (\tau-1) + p_2 v (T-\tau)})^2 (v-1) ] \\
    &\quad + p_r + p_r (\frac{c_4 p_2 \gamma}{p_2 \gamma v + p_1 v \tau + p_2 v (T-\tau-1)})^2  \pm O(\frac{1}{T}) \\
\end{align*}

in which the constant $c_4$ is defined as
\begin{itemize}
    \item $c_4 \coloneqq \frac{1}{1-(1-p_c) p_m} \in (1, 2)$
\end{itemize}

Plugging in the definition of $p_1, p_2$ in \eqref{eq:p_1}, \eqref{eq:p_2}

\begin{equation} \label{eq:l_gamma_block_asymptotic}
\begin{split}
    L(\gamma) &= p_c [ (1-\frac{\frac{1}{v \tau} \gamma}{\frac{1}{c_4 \tau} \gamma + \frac{1}{c_4 \tau} (\tau-1) + \frac{p_m p_r}{T} (T-\tau)})^2 + (\frac{\frac{1}{v \tau} \gamma}{\frac{1}{c_4 \tau} \gamma  + \frac{1}{c_4 \tau} (\tau-1) + \frac{p_m p_r}{T} (T-\tau)})^2 (v-1) ] \\
    &\quad + p_r + p_r (\frac{c_4 \frac{p_m p_r}{v T} \gamma}{\frac{p_m p_r}{v T} \gamma v + \frac{1}{c_4} + \frac{p_m p_r}{T} (T-\tau-1)})^2  \pm O(\frac{1}{T}) \\
    &= p_c \left[ (1-\frac{c_4 \gamma}{v \gamma + v (\tau-1) + c_4 p_m p_r v \tau})^2 + (\frac{c_4 \gamma}{v \gamma  + v (\tau-1) + c_4 p_m p_r v \tau})^2 (v-1) \right] \\
    &\quad + p_r + p_r (\frac{c_4 p_m p_r \gamma}{p_m p_r \gamma v + (\frac{1}{c_4} + p_m p_r) v T})^2  \pm O(\frac{1}{T}) \\
\end{split}
\end{equation}

We will again consider several possible cases for $\gamma$ in \eqref{eq:l_gamma_block_asymptotic}. 

\textbf{Case 1:}
When $\gamma \le \frac{(1 + c_4 p_m p_r) (2 - c_4)}{2 c_4} (\tau - 1)$. \\

Let $c_5$ denote the constant:
\[ c_5 \coloneqq \frac{(1 + c_4 p_m p_r) (2 - c_4)}{2 c_4} \]
then focusing on this term in the loss \eqref{eq:l_gamma_block_asymptotic}:
\begin{align*}
    &\frac{c_4 \gamma}{v \gamma + v (\tau-1) + c_4 p_m p_r v \tau} \\
    &< \frac{c_4 \gamma}{v \gamma + v \frac{\gamma}{c_5} + c_4 p_m p_r v \frac{\gamma}{c_5}} \\
    &= \frac{c_4}{v + v \frac{1}{c_5} + c_4 p_m p_r v \frac{1}{c_5}} \\
    &= \frac{c_4}{v (1 + \frac{1 + c_4 p_m p_r }{c_5})}
\end{align*}
and so
\begin{equation} \label{eq:l_gamma_block_asymptotic_case1}
\begin{split}
L(\gamma) > p_c \left(1 - \frac{c_4}{v (1 + \frac{1 + c_4 p_m p_r }{c_5})}\right)^2 + p_r \pm o(1)
\end{split}
\end{equation}

\textbf{Case 2:}
When $\gamma \ge 100 \frac{\frac{1}{c_4} + p_m p_r}{p_m p_r} T$. \\

then since $\tau = o(T)$ by Assumption~\ref{assumption:asymptotic}:
\begin{align*}
    \frac{c_4 \gamma}{v \gamma + v (\tau-1) + c_4 p_m p_r v \tau} &= \frac{c_4 \gamma}{v \gamma} + o(1) = \frac{c_4}{v} + o(1) \\
    \frac{c_4 \gamma}{v \gamma  + v (\tau-1) + c_4 p_m p_r v \tau} &= \frac{c_4 \gamma}{v \gamma} + o(1) = \frac{c_4}{v} + o(1) \\
    \frac{c_4 p_m p_r \gamma}{p_m p_r \gamma v + (\frac{1}{c_4} + p_m p_r) v T} &\ge \frac{c_4 p_m p_r \gamma}{p_m p_r \gamma v + \frac{1}{100} p_m p_r v \gamma}
    = \frac{100 c_4}{ 101 v } \\
\end{align*}
and therefore plugging into \eqref{eq:l_gamma_block_asymptotic}:
\begin{equation} \label{eq:l_gamma_block_asymptotic_case2}
\begin{split}
    L(\gamma) &\ge p_c [ (1-\frac{c_4}{v} \pm o(1))^2 + (\frac{c_4}{v} \pm o(1))^2 (v-1) ] + p_r + p_r (\frac{100 c_4}{ 101 v })^2 \pm o(1) \\
    &= p_c [ 1 - \frac{2 c_4}{v} + \frac{c_4^2}{v^2} + \frac{c_4^2}{v^2} (v-1) ] + p_r + p_r (\frac{100 c_4}{ 101 v })^2 \pm o(1) \\
    &= p_c [ 1 - \frac{2 c_4}{v} + \frac{c_4^2}{v} ] + p_r + p_r (\frac{100 c_4}{ 101 v })^2 \pm o(1) \\
    &= p_c [ 1 - \frac{c_4 (2 - c_4)}{v} ] + p_r + p_r (\frac{100 c_4}{ 101 v })^2 \pm o(1)
\end{split}
\end{equation}

\textbf{Case 3:}
When $\frac{(1 + c_4 p_m p_r) (2 - c_4)}{2 c_4} (\tau - 1) < \gamma < 100 \frac{\frac{1}{c_4} + p_m p_r}{p_m p_r} T$. \\

Note that this case is the complement of Case 1 and Case 2 above,
and so we have considered all possibilities.
We will show that there exists $\gamma$ in this case such that $L(\gamma)$ is smaller (by an $\Omega(1)$ constant difference) than the lower bound of $L(\gamma)$ proven in Case 1 and Case 2 above,
based on which we know $\argmin L(\gamma)$ cannot lie in Case 1 or Case 2,
and thus conclude that $\argmin L(\gamma)$ is within this case.
Specifically, let:
\[ \gamma = \sqrt{\tau T} \]
then similar to Case 2, since $\tau = o(T)$ by Assumption~\ref{assumption:asymptotic}:
\begin{align*}
    \frac{c_4 \gamma}{v \gamma + v (\tau-1) + c_4 p_m p_r v \tau} &= \frac{c_4 \gamma}{v \gamma} + o(1) = \frac{c_4}{v} + o(1) \\
    \frac{c_4 \gamma}{v \gamma  + v (\tau-1) + c_4 p_m p_r v \tau} &= \frac{c_4 \gamma}{v \gamma} + o(1) = \frac{c_4}{v} + o(1) \\
    \frac{c_4 p_m p_r \gamma}{p_m p_r \gamma v + (\frac{1}{c_4} + p_m p_r) v T} &= o(1)
\end{align*}
and therefore plugging into \eqref{eq:l_gamma_block_asymptotic}:
\begin{equation} \label{eq:l_gamma_block_asymptotic_case3}
\begin{split}
    L(\gamma) &= p_c [ (1-\frac{c_4 \gamma}{v \gamma} \pm o(1))^2 + (\frac{c_4 \gamma}{v \gamma} \pm o(1))^2 (v-1) ] + p_r \pm o(1) \\
    &= p_c [ (1-\frac{c_4}{v} \pm o(1))^2 + (\frac{c_4}{v} \pm o(1))^2 (v-1) ] + p_r \pm o(1) \\
    &= p_c [ 1 - \frac{2 c_4}{v} + \frac{c_4^2}{v^2} + \frac{c_4^2}{v^2} (v-1) ] + p_r \pm o(1) \\
    &= p_c [ 1 - \frac{2 c_4}{v} + \frac{c_4^2}{v} ] + p_r \pm o(1) \\
    &= p_c [ 1 - \frac{c_4 (2 - c_4)}{v} ] + p_r \pm o(1)
\end{split}
\end{equation}

\textbf{Comparing the above cases}

Note that $L(\gamma)$ in Case 3 is strictly smaller than $L(\gamma)$ in Case 1 and Case 2, because:
\begin{itemize}
    \item Comparing \eqref{eq:l_gamma_block_asymptotic_case1} and \eqref{eq:l_gamma_block_asymptotic_case3}: $(1 - \frac{c_4}{v (1 + \frac{1 + c_4 p_m p_r }{c_5})})^2 > 1 - \frac{c_4 (2 - c_4)}{v}$ because $c_5 \in (0, \frac{(1 + c_4 p_m p_r) (2 - c_4)}{c_4})$
    \item Comparing \eqref{eq:l_gamma_block_asymptotic_case2} and \eqref{eq:l_gamma_block_asymptotic_case3}: in the former, the term $p_r (\frac{100 c_4}{ 101 v })^2 > 0$ is the extra constant (of scale $\Omega(1)$, i.e. non-vanishing even under our asymptotic assumptions Assumption~\ref{assumption:asymptotic}) compared with the latter. 
\end{itemize}

Therefore we conclude that 
\[ \argmin L(\gamma) \subseteq (\frac{(1 + c_4 p_m p_r) (2 - c_4)}{2 c_4} (\tau - 1), 100 \frac{\frac{1}{c_4} + p_m p_r}{p_m p_r} T ) \]

\end{proof}

\begin{remark}
    \label{rem:challenge_of_precise_optima}
    In Theorem~\ref{thm:optimal_attention_weights_updated}, we specify some necessary conditions that the optimal $\gamma$ must satisfy.
    It is challenging to precisely characterize the optima (to within $o(1)$ error),
    because doing so may require explicitly writing those smaller scale terms hidden (in $\pm o(1)$) by our asymptotic setting (Assumption~\ref{assumption:asymptotic}). 
    Those smaller scale terms, however, do not affect our analysis, 
    because these $\pm o(1)$ terms cannot reverse the $\Omega(1)$ constant separation between the loss in the above different cases.
\end{remark}

\clearpage
\subsection{Optimal attention weights (when freezing diagonal \texorpdfstring{$\mW^V$}{Wv})}
\label{sec:appendix:attention:weights:diagonal}

Our Stage-2 analysis on the optimal attention weights (\eqref{eq:attention_weights}) is based on freezing $\mW^V$ to be the Stage-1 optima characterized in Theorem~\ref{thm:optimal_Wv_given_uniform_attention_l2reg}.
Notably, in Theorem~\ref{thm:optimal_Wv_given_uniform_attention_l2reg}, the uniqueness of the optima (i.e. a clean block-wise pattern) crucially depends on the $L_2$ regularization. 
Indeed, as we prove in Theorem~\ref{thm:optimal_Wv_given_uniform_attention} (in Appendix~\ref{sec:appendix:Wv}),
without the regularization, there is a family of optima (depending on a series of free constants) all of which can encode the topic structure.

Among these alternative optima, we are particularly interested in a special case --- one that has a \emph{diagonal} pattern.
This type of diagonally structured $\mW^V$ often occurs when we train the single-layered transformer model without $L_2$ regularization.

\begin{figure}[ht]
  \centering
  \begin{minipage}[b]{0.45\textwidth} %
    \centering
    \includegraphics[width=1.0\textwidth]{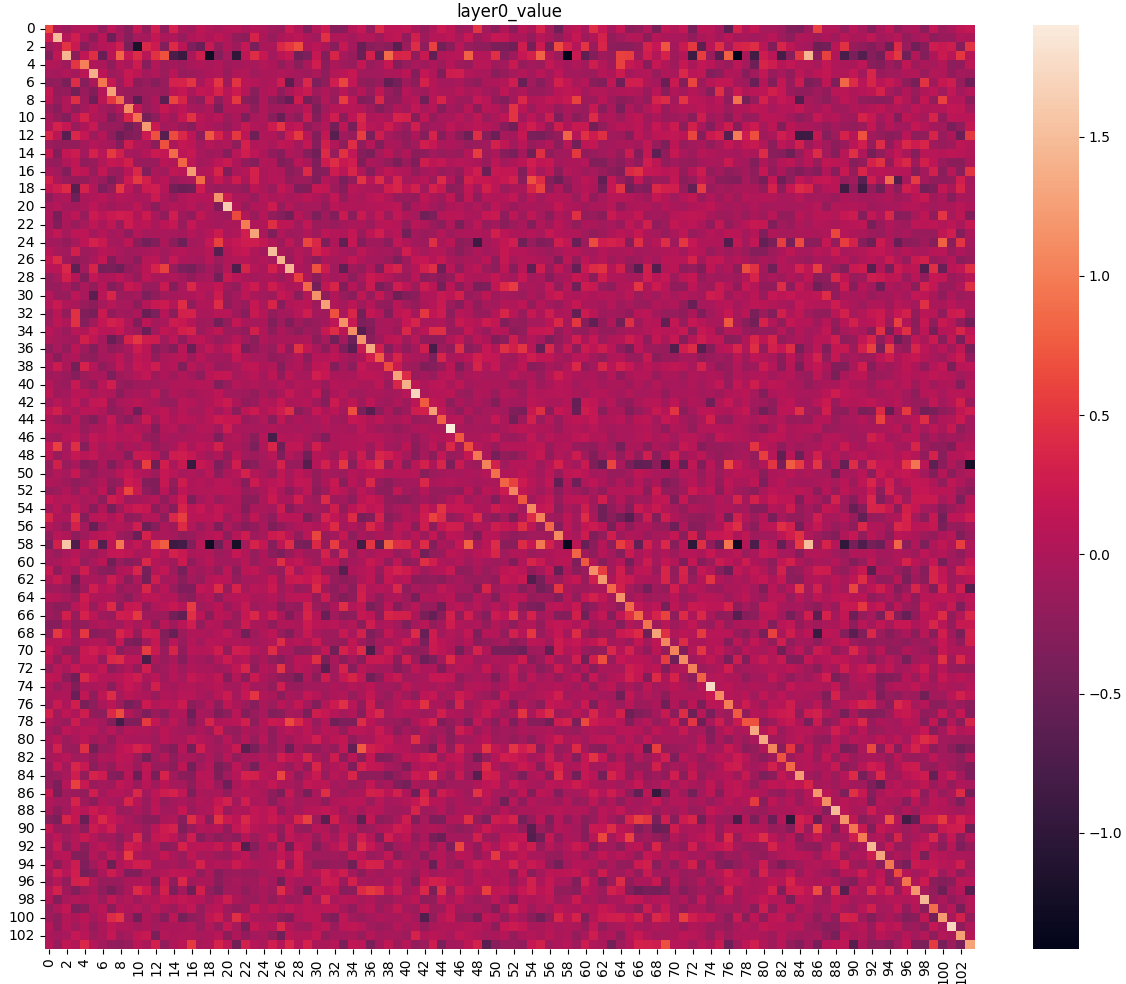}
  \end{minipage}
  \caption{Without $L_2$ regularization, the trained $\mW^V$ sometimes shows a \emph{diagonal} pattern, which is a special case of the family of optima characterized in Theorem~\ref{thm:optimal_Wv_given_uniform_attention} (in Appendix~\ref{sec:appendix:Wv}).}  
  \label{fig:Wv_convergence}
\end{figure}

Motivated by this empirical observation, we formally define the particular optima from Theorem~\ref{thm:optimal_Wv_given_uniform_attention} (in Appendix~\ref{sec:appendix:Wv}) that is a diagonal pattern.

\begin{definition}[diagonal $\mW^V$]
    \label{def:diagonal_Wv}
    The diagonal optima of $\mW^V$, denoted as $\mD^V$, 
    is the only matrix in $\R^{(T v + 1) \times (T v + 1)}$ that satisfies both
    $\mD^V \in \argmin L(\mW^V)$ (in Theorem~\ref{thm:optimal_Wv_given_uniform_attention}) and 
    \[  \forall i, j \in \{1, \cdots, T v\}, \; \mW^V_{ij} = 0 \text{ if } i \ne j \]
\end{definition}

Corresponding to this case, we provide an analysis on the Stage-2 optimal attention weights,
which shows a very interesting different behavior from the result in Theorem~\ref{thm:optimal_attention_weights_updated} (for block-wise $\mW^V$).

\clearpage
\begin{theorem}[optimal attention weights when freezing diagonal \texorpdfstring{$\mW^V$}{Wv}]
\label{thm:optimal_attention_weights_updated:WvI}
    Suppose the data distribution follows the topic modeling assumption in Section~\ref{sec:setup:topic_modeling} and Assumption~\ref{assumption:infinitely_long_document}.
    Suppose we train a single layer transformer given by \eqref{eq:simplified_transformer_no_emb}
    with $\vb^{\text{pred}} = 0$ and $\mW^V$ frozen to $\mD^V$ in Definition~\ref{def:diagonal_Wv},
    under masked language modeling objective (\eqref{eq:objective}) 
    with the squared loss (\eqref{eq:squared_loss}),
    under Assumption~\ref{assumption:alpha_beta_attention}, Assumption~\ref{assumption:asymptotic}, and Assumption~\ref{assumption:correct_random_balance}.
    Then, the optimal $(\alpha, \beta)$ satisfy:
    \begin{align*}
        \lambda_3 \tau < &\beta < \lambda_4 T \\
        \alpha &< \lambda_5 \beta
    \end{align*}
    in which the constants 
    \begin{align*}
        \lambda_3 &\coloneqq \frac{1-(1-p_c) p_m + p_m p_r}{100} \cdot v \\
        \lambda_4 &\coloneqq \frac{1-(1-p_c) p_m}{\sqrt{v-1} - 2 + (1-p_c) p_m} \cdot \frac{1-(1-p_c-p_r) p_m}{p_m p_r} \cdot v \\
        \lambda_5 &\coloneqq \frac{1}{(v-1) (1-(1-p_c) p_m)}
    \end{align*}
\end{theorem}

\begin{proof}
Following the same steps leading to \eqref{eq:l_gamma},
define:
\begin{equation}
    \label{eq:l_alpha_beta}
    L(\alpha, \beta) \coloneqq \frac{1}{p_m N} \E_{\mX \sim \mathcal{D}_\mX} \E_{M} \left[ \sum_{j \in M, \tilde{w}_j \ne \texttt{[MASK]} } \| \mD^V \tilde{\mX} A(\tilde{\mX})_{:j} - \mX_{:j} \|_2^2 \right]
\end{equation}

Note that $\forall \vy \in \R^{T v + 1}$
\[ (\mD^V \vy)_i = \begin{cases}
        0, \quad & i = 0 \\ 
        q(y_i), \quad & i \in \{ 1, \cdots, T v + 1 \}
    \end{cases}  \]
in which 
$ q(x) \coloneqq \frac{1}{1-(1-p_c) p_m} x - \frac{p_m p_r}{(1-(1-p_c)p_m)Tv} $

In our context, we will consider $\vy = \tilde{\mX} A(\tilde{\mX})_{:j}$ in $L(\alpha, \beta)$ above.

For a document $\vw$ which contains topics $t_1, \cdots, t_\tau \in [T]$,
there are the following cases:

\paragraph{Case 1: $\tilde{w}_j = w_j$}
When $\tilde{w}_j$ after masking is the correct token $w_j$. 
(This happens with probability $p_c + \frac{p_r}{v T}$)

By \eqref{eq:pre_Wv_weighted_sum},
\begin{equation}
    \label{eq:pre_Wv_weighted_sum_case_majority}
    \left[ \tilde{\mX} A(\tilde{\mX}) \right]_{ij} = \begin{cases}
        \frac{p_1 \beta}{z_1}, \quad &\text{if } i = \tilde{w}_j  \\
        \frac{p_1 \alpha}{z_1}, \quad &\text{if } i \ne \tilde{w}_j, \texttt{topic}(i) = \texttt{topic}(\tilde{w}_j) \\
        \frac{p_1}{z_1}, \quad &\text{if } \texttt{topic}(i) \ne \texttt{topic}(\tilde{w}_j), \texttt{topic}(i) \in \{t_1, \cdots, t_\tau\}  \\
        \frac{p_2}{z_1}, \quad &\text{if } \texttt{topic}(i) \ne \texttt{topic}(\tilde{w}_j), \texttt{topic}(i) \notin \{t_1, \cdots, t_\tau\}
    \end{cases}
\end{equation}

Recall that the label is
$ \mX_{:j} = \begin{cases}
        1, \quad & i = w_j  \\
        0, \quad & i \in \{0, \cdots, T v\} \backslash w_j
    \end{cases} $

Hence the contribution to the loss from token $\tilde{w}_j$ is
\begin{align*}
    &\quad (p_c + \frac{p_r}{v T}) [\left(1-q\left(\frac{p_1 \beta}{z_1}\right)\right)^2 + q(\frac{p_1 \alpha}{z_1})^2 \cdot (v-1) + q(\frac{p_1}{z_1})^2 \cdot v (\tau-1) + q(\frac{p_2}{z_1})^2 \cdot v (T-\tau) ] \\
    &=(p_c + \frac{p_r}{v T}) [ (1-q(\frac{p_1 \beta}{\beta p_1 + \alpha p_1 (v-1) + p_1 v (\tau-1) + p_2 v (T-\tau)}))^2 \\
    &+ q(\frac{p_1 \alpha}{\beta p_1 + \alpha p_1 (v-1) + p_1 v (\tau-1) + p_2 v (T-\tau)})^2 (v-1) \\
    &+ q(\frac{p_1}{\beta p_1 + \alpha p_1 (v-1) + p_1 v (\tau-1) + p_2 v (T-\tau)})^2 v (\tau - 1) \\
    &+ q(\frac{p_2}{\beta p_1 + \alpha p_1 (v-1) + p_1 v (\tau-1) + p_2 v (T-\tau)})^2 v (T-\tau) ]
\end{align*}

Plugging in the asymptotics from Assumption~\ref{assumption:asymptotic}, 
the above becomes 
\begin{equation}
    \label{eq:case1l} 
    \begin{split}
    &p_c [ 
   (1-q(\frac{p_1 \beta}{\beta p_1 + \alpha p_1 (v-1) + p_1 v (\tau-1) + p_2 v (T-\tau)}))^2 \\
   &\quad + q(\frac{p_1 \alpha}{\beta p_1 + \alpha p_1 (v-1) + p_1 v (\tau-1) + p_2 v (T-\tau)})^2 (v-1) \pm O(\frac{1}{\tau}) 
   ] 
    \end{split}
\end{equation}

\paragraph{Case 2: $\tilde{w}_j \ne w_j, \texttt{topic}(\tilde{w}_j) = \texttt{topic}(w_j)$}
When $\tilde{w}_j$ after masking is not the correct token but belongs to the same topic as the correct token $w_j$. 
(This happens with probability $\frac{p_r}{T} (1-\frac{1}{v})$)

$\left[ \tilde{\mX} A(\tilde{\mX}) \right]_{:j}$ is the same as \eqref{eq:pre_Wv_weighted_sum_case_majority}.

Hence the loss is
\begin{align*}
    &\quad \frac{p_r}{T}(1-\frac{1}{v}) [(1-q(\frac{p_1 \alpha}{z_1}))^2 + q(\frac{p_1 \beta}{z_1}))^2 + q(\frac{p_1 \alpha}{z_1})^2 \cdot (v-2) + q(\frac{p_1}{z_1})^2 \cdot v (\tau-1) + q(\frac{p_2}{z_1})^2 \cdot v (T-\tau)] \\
    &= \frac{p_r}{T}(1-\frac{1}{v}) [(1-q(\frac{p_1 \alpha}{\beta p_1 + \alpha p_1 (v-1) + p_1 v (\tau-1) + p_2 v (T-\tau)}))^2 \\
    &+ q(\frac{p_1 \beta}{\beta p_1 + \alpha p_1 (v-1) + p_1 v (\tau-1) + p_2 v (T-\tau)})^2 \\
    &+ q(\frac{p_1 \alpha}{\beta p_1 + \alpha p_1 (v-1) + p_1 v (\tau-1) + p_2 v (T-\tau)})^2 (v-2) \\
    &+ q(\frac{p_1}{\beta p_1 + \alpha p_1 (v-1) + p_1 v (\tau-1) + p_2 v (T-\tau)})^2 v (\tau - 1) \\
    &+ q(\frac{p_2}{\beta p_1 + \alpha p_1 (v-1) + p_1 v (\tau-1) + p_2 v (T-\tau)})^2 v (T-\tau)] 
\end{align*}

Plugging in the asymptotics from Assumption~\ref{assumption:asymptotic}, 
the above terms vanish.

\paragraph{Case 3: $\texttt{topic}(\tilde{w}_j) \in \{t_1, \cdots, t_\tau\} \backslash \{\texttt{topic}(w_j)\}$}
When $\tilde{w}_j$ after masking belongs to a different topic from that of the correct token $w_j$, but still a topic existing in $\vw$. 
(This happens with probability $\frac{p_r (\tau - 1)}{T}$)

$\left[ \tilde{\mX} A(\tilde{\mX}) \right]_{:j}$ is the same as \eqref{eq:pre_Wv_weighted_sum_case_majority}.

Hence the loss is
\begin{align*}
    &\quad p_r \frac{\tau-1}{T} [(1-q(\frac{p_1}{z_1}))^2 + q(\frac{p_1 \beta}{z_1}))^2 + q(\frac{p_1 \alpha}{z_1})^2 \cdot (v-1) + q(\frac{p_1}{z_1})^2 \cdot (v (\tau-1) - 1) + q(\frac{p_2}{z_1})^2 \cdot v (T-\tau)] \\
    &= p_r \frac{\tau-1}{T} [(1-q(\frac{p_1}{\beta p_1 + \alpha p_1 (v-1) + p_1 v (\tau-1) + p_2 v (T-\tau)})^2 \\
    &+ q(\frac{p_1 \beta}{\beta p_1 + \alpha p_1 (v-1) + p_1 v (\tau-1) + p_2 v (T-\tau)})^2 \\
    &+ q(\frac{p_1 \alpha}{\beta p_1 + \alpha p_1 (v-1) + p_1 v (\tau-1) + p_2 v (T-\tau)})^2 (v-1) \\
    &+ q(\frac{p_1}{\beta p_1 + \alpha p_1 (v-1) + p_1 v (\tau-1) + p_2 v (T-\tau)})^2 (v (\tau-1) - 1) \\
    &+ q(\frac{p_2}{\beta p_1 + \alpha p_1 (v-1) + p_1 v (\tau-1) + p_2 v (T-\tau)})^2 v (T-\tau) ]
\end{align*}

Plugging in the asymptotics from Assumption~\ref{assumption:asymptotic}, 
the above terms vanish.

\paragraph{Case 4: $\texttt{topic}(\tilde{w}_j) \in [T] \backslash \{t_1, \cdots, t_\tau\}$}
When $\tilde{w}_j$ after masking belongs to a topic that does not exist in $\vw$. 
(This happens with probability $p_r (1-\frac{\tau}{T})$)

By \eqref{eq:pre_Wv_weighted_sum},
\begin{equation}
    \label{eq:pre_Wv_weighted_sum_case_minority_diagonal}
    \left[ \tilde{\mX} A(\tilde{\mX}) \right]_{ij} = \begin{cases}
        \frac{p_2 \beta}{z_2}, \quad &\text{if } i = \tilde{w}_j  \\
        \frac{p_2 \alpha}{z_2}, \quad &\text{if } i \ne \tilde{w}_j, \texttt{topic}(i) = \texttt{topic}(\tilde{w}_j)  \\
        \frac{p_1}{z_2}, \quad &\text{if } \texttt{topic}(i) \ne \texttt{topic}(\tilde{w}_j), \texttt{topic}(i) \in \{t_1, \cdots, t_\tau\}  \\
        \frac{p_2}{z_2}, \quad &\text{if } \texttt{topic}(i) \ne \texttt{topic}(\tilde{w}_j), \texttt{topic}(i) \notin \{t_1, \cdots, t_\tau\}
    \end{cases}
\end{equation}

Hence the loss is
\begin{align*}
    &\quad p_r (1-\frac{\tau}{T}) [(1-q(\frac{p_1}{z_2}))^2 + q(\frac{p_1}{z_2})^2 \cdot (v \tau - 1) + q(\frac{p_2 \beta}{z_2})^2 + q(\frac{p_2 \alpha}{z_2})^2 \cdot (v-1) + q(\frac{p_2}{z_2})^2 \cdot v (T-\tau - 1)] \\
    &= p_r (1-\frac{\tau}{T}) [(1-q(\frac{p_1}{\beta p_2 + \alpha p_2 (v-1) + p_1 v \tau + p_2 v (T-\tau-1)}))^2 \\
    &+ q(\frac{p_1}{\beta p_2 + \alpha p_2 (v-1) + p_1 v \tau + p_2 v (T-\tau-1)})^2 (v \tau - 1) \\
    &+ q(\frac{p_2 \beta}{\beta p_2 + \alpha p_2 (v-1) + p_1 v \tau + p_2 v (T-\tau-1)})^2 + q(\frac{p_2 \alpha}{\beta p_2 + \alpha p_2 (v-1) + p_1 v \tau + p_2 v (T-\tau-1)})^2 (v-1) \\
    &+ q(\frac{p_2}{\beta p_2 + \alpha p_2 (v-1) + p_1 v \tau + p_2 v (T-\tau-1)})^2 v (T-\tau-1)]
\end{align*}

Plugging in the asymptotics from Assumption~\ref{assumption:asymptotic}, 
the above becomes 
\begin{equation}
    \label{eq:case4l}
    \begin{split}
    &p_r [ (1 + q(\frac{p_2 \beta}{\beta p_2 + \alpha p_2 (v-1) + p_1 v \tau + p_2 v (T-\tau-1)})^2 ) \\
    &\quad + q(\frac{p_2 \alpha}{\beta p_2 + \alpha p_2 (v-1) + p_1 v \tau + p_2 v (T-\tau-1)})^2 (v-1) ] \pm o(1)
    \end{split}
\end{equation}

\paragraph{Combining the above cases}
Adding \eqref{eq:case1l} and \eqref{eq:case4l}, we can see in the asymptotic regime of interest: 
\begin{align*}
    &L(\alpha, \beta) = p_c [ 
   (1-q(\frac{p_1 \beta}{\beta p_1 + \alpha p_1 (v-1) + p_1 v (\tau-1) + p_2 v (T-\tau)}))^2 \\
   &\quad + q(\frac{p_1 \alpha}{\beta p_1 + \alpha p_1 (v-1) + p_1 v (\tau-1) + p_2 v (T-\tau)})^2 (v-1)]  \\
    &\quad + p_r [ (1 + q(\frac{p_2 \beta}{\beta p_2 + \alpha p_2 (v-1) + p_1 v \tau + p_2 v (T-\tau-1)})^2 ) \\
    &\quad + q(\frac{p_2 \alpha}{\beta p_2 + \alpha p_2 (v-1) + p_1 v \tau + p_2 v (T-\tau-1)})^2 (v-1) ] \pm o(1) \\
    &= p_c [ 
   (1-\frac{\frac{1}{v \tau} \beta}{\beta \frac{1}{c_4 v \tau} + \alpha \frac{v-1}{c_4 v \tau} + \frac{1}{c_4} + p_m p_r})^2 
   + (\frac{\frac{1}{v \tau} \alpha}{\beta \frac{1}{c_4 v \tau} + \alpha \frac{v-1}{c_4 v \tau} + \frac{1}{c_4} + p_m p_r})^2 (v-1)]  \\
    &\quad + p_r [ (1 + (\frac{c_4 \frac{p_m p_r}{v T} \beta}{\beta \frac{p_m p_r}{v T} + \alpha \frac{p_m p_r (v-1)}{v T} + \frac{1}{c_4} + p_m p_r})^2 ) 
    + (\frac{c_4 \frac{p_m p_r}{v T} \alpha}{\beta \frac{p_m p_r}{v T} + \alpha \frac{p_m p_r (v-1)}{v T} + \frac{1}{c_4} + p_m p_r})^2 (v-1) ] \pm o(1)
\end{align*}

in which the constant $c_4$ is defined as
\begin{itemize}
    \item $c_4 \coloneqq \frac{1}{1-(1-p_c) p_m} \in (1, 2)$ by Assumption~\ref{assumption:correct_random_balance}.
\end{itemize}

\begin{equation} \label{eq:l_alpha_beta_diagonal_asymptotic}
\begin{split}
&L(\alpha, \beta) = p_c [ 
(1-\frac{c_4 \beta}{\beta + (v-1) \alpha + (1 + c_4 p_m p_r) v \tau})^2 
+ (\frac{c_4 \alpha}{\beta + (v-1) \alpha + (1 + c_4 p_m p_r) v \tau})^2 (v-1)]  \\
&\quad + p_r [ (1 + (\frac{c_4 \beta}{\beta + (v-1) \alpha + (\frac{1}{c_4 p_m p_r} + 1) v T})^2 ) 
+ (\frac{c_4 \alpha}{\beta + (v-1) \alpha + (\frac{1}{c_4 p_m p_r} + 1) v T})^2 (v-1) ] \pm o(1) \\
\end{split}
\end{equation}

We will again consider several possible cases for $\alpha, \beta$ in \eqref{eq:l_alpha_beta_diagonal_asymptotic}. 
\begin{itemize}
    \item \textbf{Case 1,  $\mathbf{\beta \le \frac{1 + c_4 p_m p_r}{100 c_4} v \tau}$}: 
    then $\frac{c_4 \beta}{\beta + (v-1) \alpha + (1 + c_4 p_m p_r) v \tau} \le \frac{c_4 \beta}{\beta + 100 c_4 \beta} < \frac{1}{100}$, 
    and hence 
    \[ L(\alpha, \beta) \ge p_c (1-\frac{1}{100})^2 + p_r \pm o(1) \]
    \item \textbf{Case 2,  $\mathbf{\beta > \frac{1 + c_4 p_m p_r}{100 c_4} v \tau}$}: 
    we have the following subcases:
    \begin{itemize}
        \item If $\alpha \ge \frac{c_4}{v-1} \beta$, 
        then $\frac{c_4 \beta}{\beta + (v-1) \alpha + (1 + c_4 p_m p_r) v \tau} < \frac{c_4 \beta}{\beta + (v-1) \alpha} \le \frac{c_4 \beta}{\beta + c_4 \beta} < \frac{c_4}{1 + c_4}$,
        and hence by \eqref{eq:l_alpha_beta_diagonal_asymptotic} $L(\alpha, \beta) \ge p_c (1-\frac{c_4}{1 + c_4})^2 + p_r \pm o(1)$.
        \item If $\alpha < \frac{c_4}{v-1} \beta$, 
        then $\frac{c_4 \beta}{\beta + (v-1) \alpha + (\frac{1}{c_4 p_m p_r} + 1) v T} > \frac{c_4 \beta}{\beta + c_4 \beta + (\frac{1}{c_4 p_m p_r} + 1) v T}$
        \begin{itemize}
            \item If $\beta \ge c_7 (\frac{1}{c_4 p_m p_r} + 1) v T$ (for some constant $c_7 \coloneqq \frac{1}{c_4 (\sqrt{v-1} - \frac{1}{c_4} - 1)}$),
            then $\frac{c_4 \beta}{\beta + (v-1) \alpha + (\frac{1}{c_4 p_m p_r} + 1) v T} > \frac{c_4 \beta}{\beta + c_4 \beta + (\frac{1}{c_4 p_m p_r} + 1) v T} \ge \frac{c_4 \beta}{\beta + c_4 \beta + \frac{1}{c_7} \beta} = \frac{c_4}{1 + c_4 + \frac{1}{c_7}}$, 
            and hence $L(\alpha, \beta) > p_r [1 + (\frac{c_4}{1 + c_4 + \frac{1}{c_7}})^2] \pm o(1)$
            \item If $\beta < c_7 (\frac{1}{c_4 p_m p_r} + 1) v T$:
            note that this case is the complement of all cases (and subcases) above,
            and so we have considered all possibilities.
            We will show that there exists $(\alpha, \beta)$ in this case such that $L(\alpha, \beta)$ is smaller (by an $\Omega(1)$ constant difference) than the lower bound of $L(\alpha, \beta)$ proven in all cases above,
            based on which we know $\argmin L(\alpha, \beta)$ cannot lie in any of the above cases,
            and thus conclude that $\argmin L(\alpha, \beta)$ is within this case.
            
            Specifically: let 
            $\alpha = \sqrt{\tau T}$ 
            and $\beta = \frac{v-1}{c_4 - 1} \alpha = \frac{v-1}{c_4 - 1} \sqrt{\tau T}$,
            then 
            \begin{align*}
                \frac{c_4 \beta}{\beta + (v-1) \alpha + (1 + c_4 p_m p_r) v \tau} &= \frac{c_4 \frac{v-1}{c_4 - 1}}{\frac{v-1}{c_4 - 1} + (v-1)} \pm o(1) = 1 \pm o(1) \\
                \frac{c_4 \alpha}{\beta + (v-1) \alpha + (1 + c_4 p_m p_r) v \tau} &= \frac{c_4}{\frac{v-1}{c_4 - 1} + (v-1)} \pm o(1) = \frac{c_4 - 1}{v-1} \pm o(1) \\
                \frac{c_4 \beta}{\beta + (v-1) \alpha + (\frac{1}{c_4 p_m p_r} + 1) v T} &= o(1) \\
                \frac{c_4 \alpha}{\beta + (v-1) \alpha + (\frac{1}{c_4 p_m p_r} + 1) v T} &= o(1) 
            \end{align*}
            Plugging into \eqref{eq:l_alpha_beta_diagonal_asymptotic}:
            \begin{align*}
                L(\alpha, \beta) &= p_c [ 
                (1-(1 \pm o(1)))^2 
                + (\frac{c_4 - 1}{v-1} \pm o(1))^2 (v-1)]  
                + p_r [ (1 + (o(1))^2 ) 
                + (o(1))^2 (v-1) ] \pm o(1) \\
                &= p_c [ \frac{(c_4 - 1)^2}{v-1}]  
                + p_r \pm o(1) \\
                &< p_c \frac{1}{v-1} + p_r \pm o(1)
            \end{align*}
        \end{itemize}
        Note that this is smaller than all previous cases, because
        \begin{itemize}
            \item $\frac{1}{v-1} < (1-\frac{1}{100})^2$ since $v$ is a large finite constant (see Assumption~\ref{assumption:asymptotic} and Assumption~\ref{assumption:correct_random_balance}).
            \item $\frac{1}{v-1} < (1-\frac{c_4}{1 + c_4})^2$ since $v$ is a large finite constant (see Assumption~\ref{assumption:asymptotic} and Assumption~\ref{assumption:correct_random_balance}).
            \item $\frac{1}{v-1} < (\frac{c_4}{1 + c_4 + \frac{1}{c_7}})^2$ by the definition of $c_7$ above.
        \end{itemize}
    \end{itemize}
\end{itemize}

Therefore, we conclude that all $\alpha, \beta > 0$ that minimize $L(\alpha, \beta)$  must satisfy 
\begin{align*}
    \frac{1 + c_4 p_m p_r}{100 c_4} v \tau < &\beta < c_7 (\frac{1}{c_4 p_m p_r} + 1) v T \\
    \alpha &< \frac{c_4}{v-1} \beta
\end{align*}
\end{proof}

\begin{remark}
    Remark~\ref{rem:challenge_of_precise_optima} applies to this proof too. 
\end{remark}

\clearpage
\subsection{Loss landscape with respect to attention weights in the non-asymptotic setting}
\label{sec:appendix:attention:non_asymptotic}

When $T, \tau$ are finite, 
the loss expression turns out to be too complicated to characterize in closed form
(because all the $o(1)$ terms need to be expanded).
So we instead numerically compute the loss landscape as a function of $\alpha$ and $\beta$.

We set $T = 100$ following our experimental setup on Wikipedia dataset (in Section~\ref{sec:experiments}),
and $v = 300$ (so total vocabulary size $T v = 30000$) following the pre-trained BERT tokenizer in Huggingface implementation \cite{wolf2020transformers}.
We will vary $\tau \in \{20, 40, 60, 80\}$. 

\paragraph{Diagonal $\mW^V$}
First, when $\mW^V$ is fixed to a diagonal structure (Definition~\ref{def:diagonal_Wv}),
Theorem~\ref{thm:optimal_attention_weights_updated:WvI} predicts that 
the loss is lowest when $\beta$ is within an interval (boundaries controlled by $\tau$ and $T$),
and $\alpha$ is less than a constant multiple of $\beta$. 
Both constraints are visible in the non-asymptotic setting, as we show in the following:

\begin{figure}[!h]
  \centering
  \begin{minipage}[b]{0.23\textwidth}
    \includegraphics[width=\textwidth]{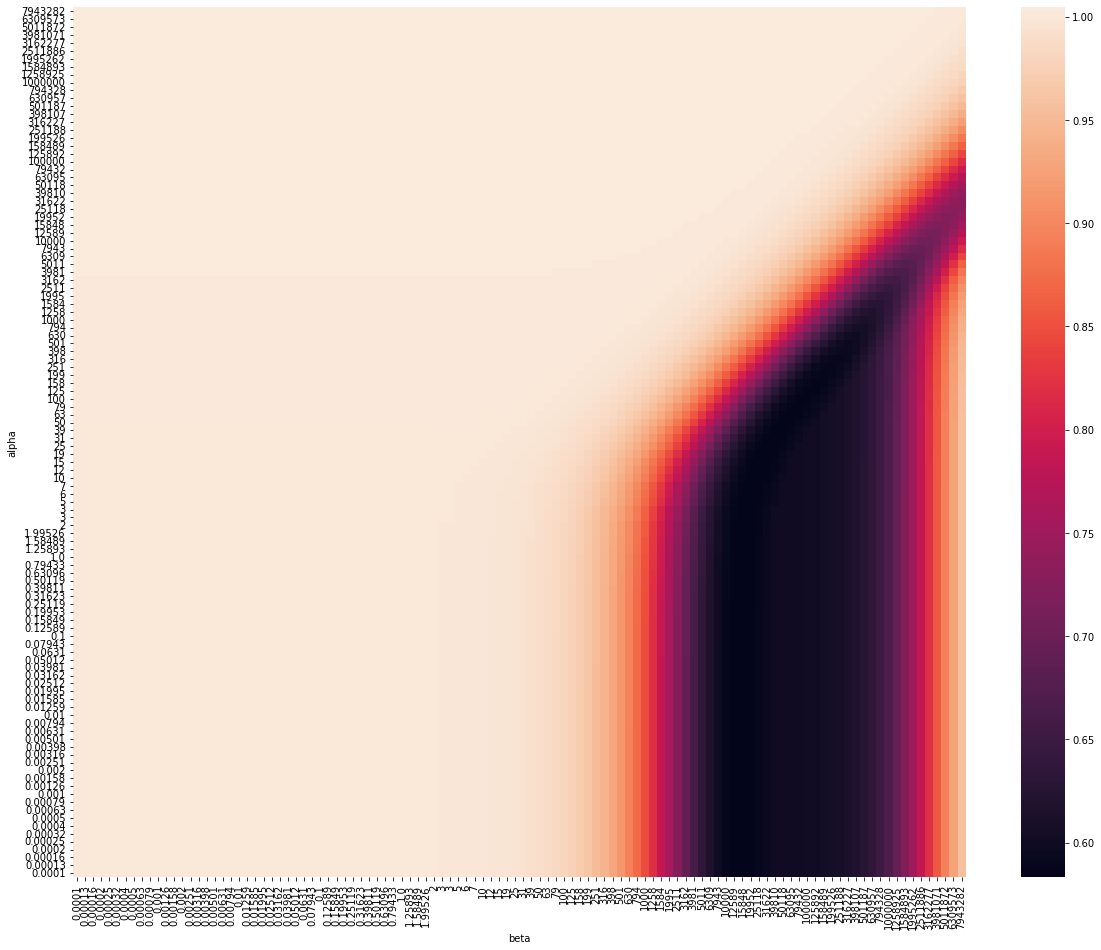}
  \end{minipage}
  \hfill
  \begin{minipage}[b]{0.23\textwidth}
    \includegraphics[width=\textwidth]{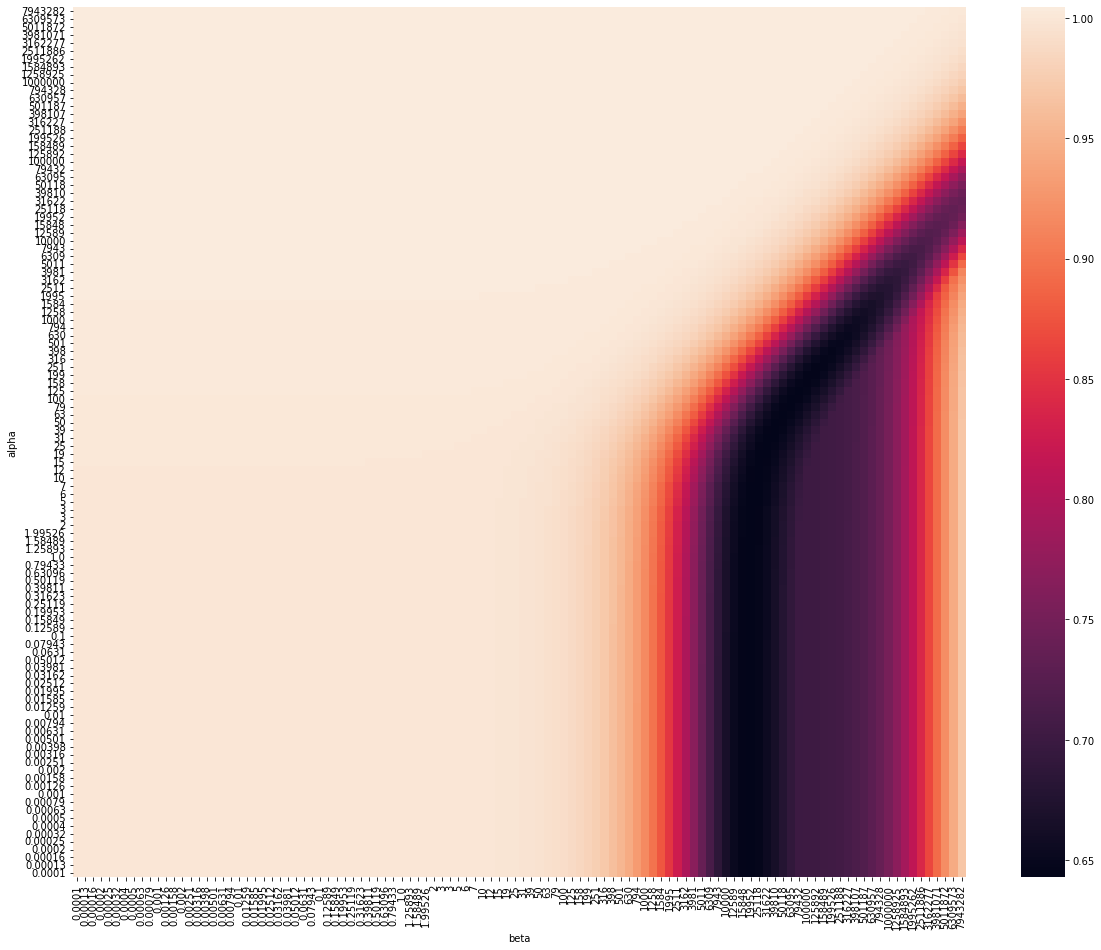}
  \end{minipage}
  \hfill
  \begin{minipage}[b]{0.23\textwidth}
    \includegraphics[width=\textwidth]{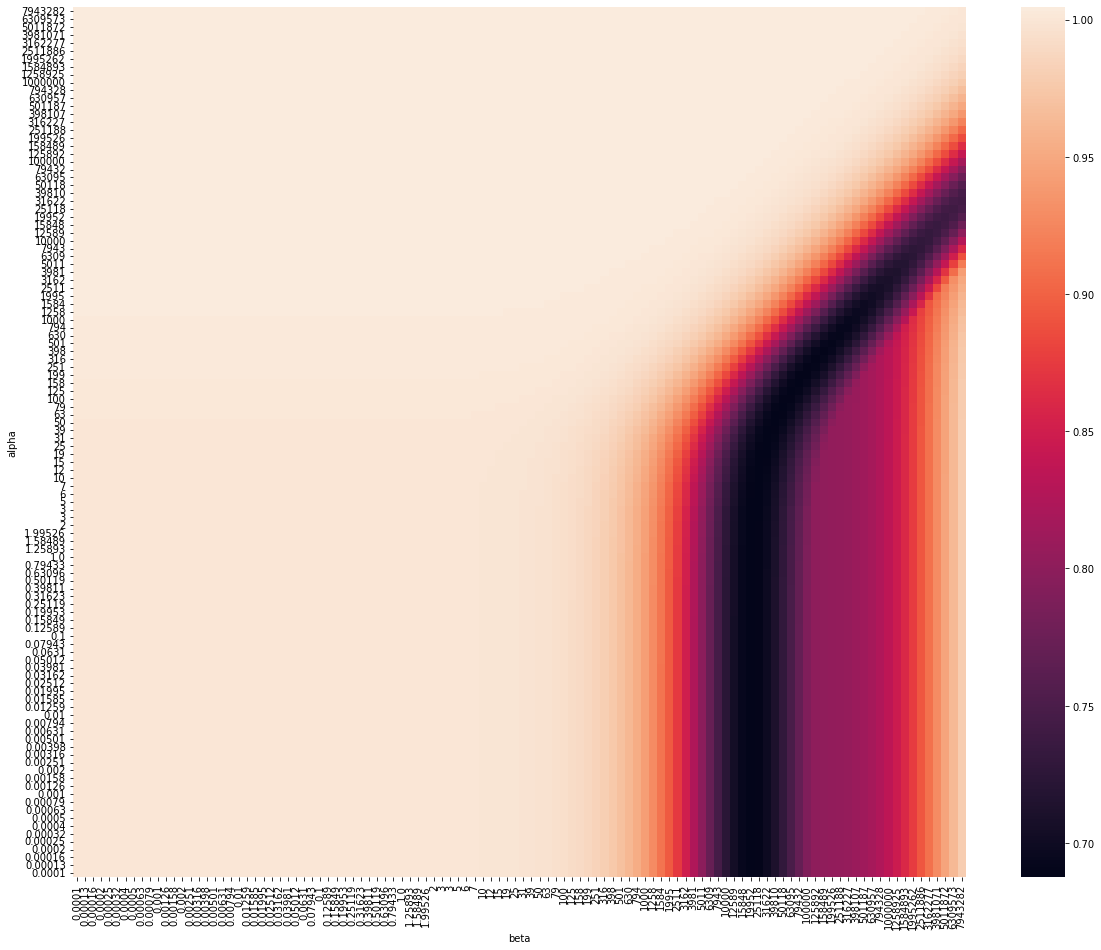}
  \end{minipage}
  \hfill
  \begin{minipage}[b]{0.23\textwidth}
    \includegraphics[width=\textwidth]{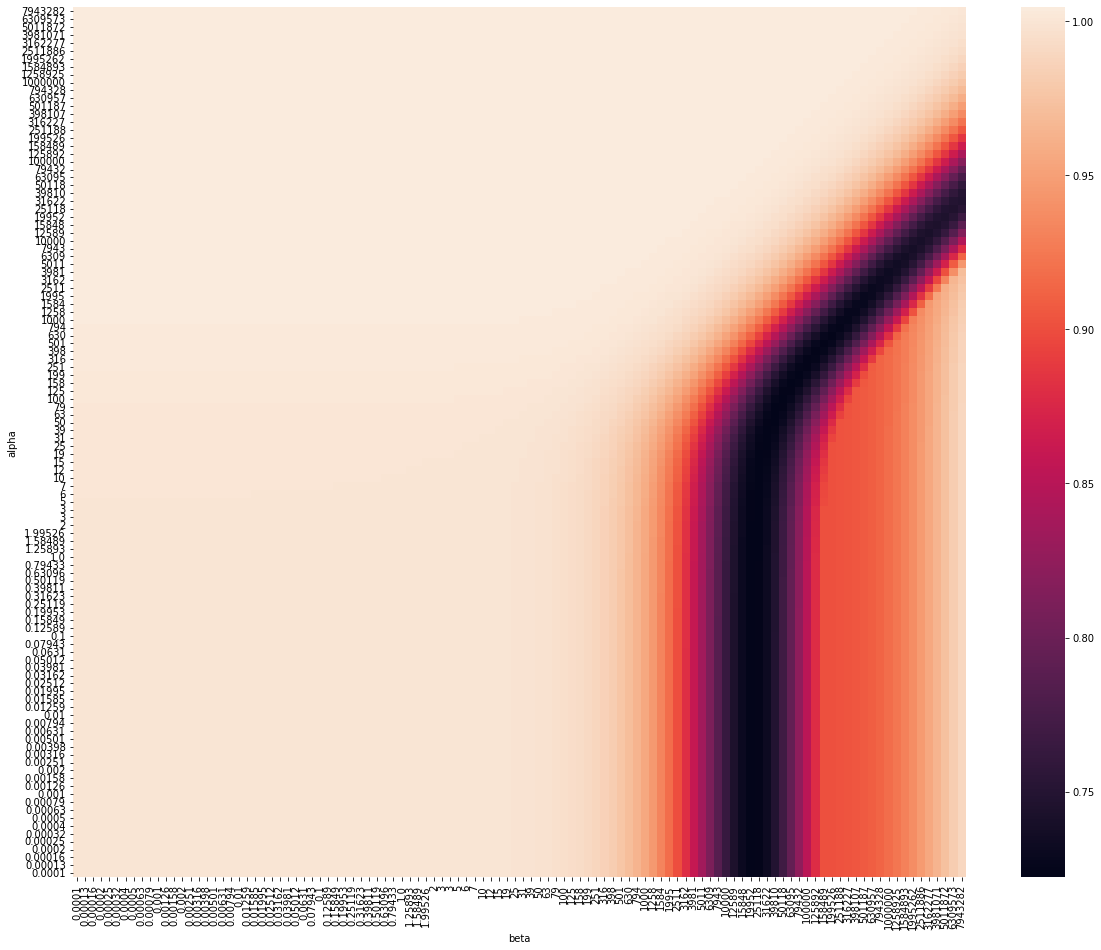}
  \end{minipage}
  \\~\\
  \caption{Landscape of squared loss under diagonal $\mW^V$ (Definition~\ref{def:diagonal_Wv}), $T=100, v=300$. 
  (left-to-right) $\tau = 20$, $\tau = 40$, $\tau = 60$, $\tau = 80$. 
  In each plot, we perform a grid search over $\alpha, \beta \in [10^{-4}, 10^7]$ (both axes use log-scale). 
  Darker color represents lower loss. 
  Across a wide range of $\tau$ (compared to $T$), the loss is lowest when $\beta$ is within an interval (lower bound growing with $\tau$),
  and the optimal $\alpha$ is less than a constant multiple of $\beta$. 
  }  
\end{figure}

\paragraph{$\mW^V$ with uniform blocks}
On the other hand, when $\mW^V$ is fixed to a block-wise structure with uniform blocks (i.e. optima in Theorem~\ref{thm:optimal_Wv_given_uniform_attention_l2reg}),
Theorem~\ref{thm:optimal_attention_weights_updated} predicts that
the loss is lowest when a convex combination of $\alpha$ and $\beta$ is within an interval (boundaries controlled by $\tau$ and $T$).
As we show in the following, a variant of this constraint visibly holds in the non-asymptotic setting.

\begin{figure}[!h]
  \centering
  \begin{minipage}[b]{0.23\textwidth}
    \includegraphics[width=\textwidth]{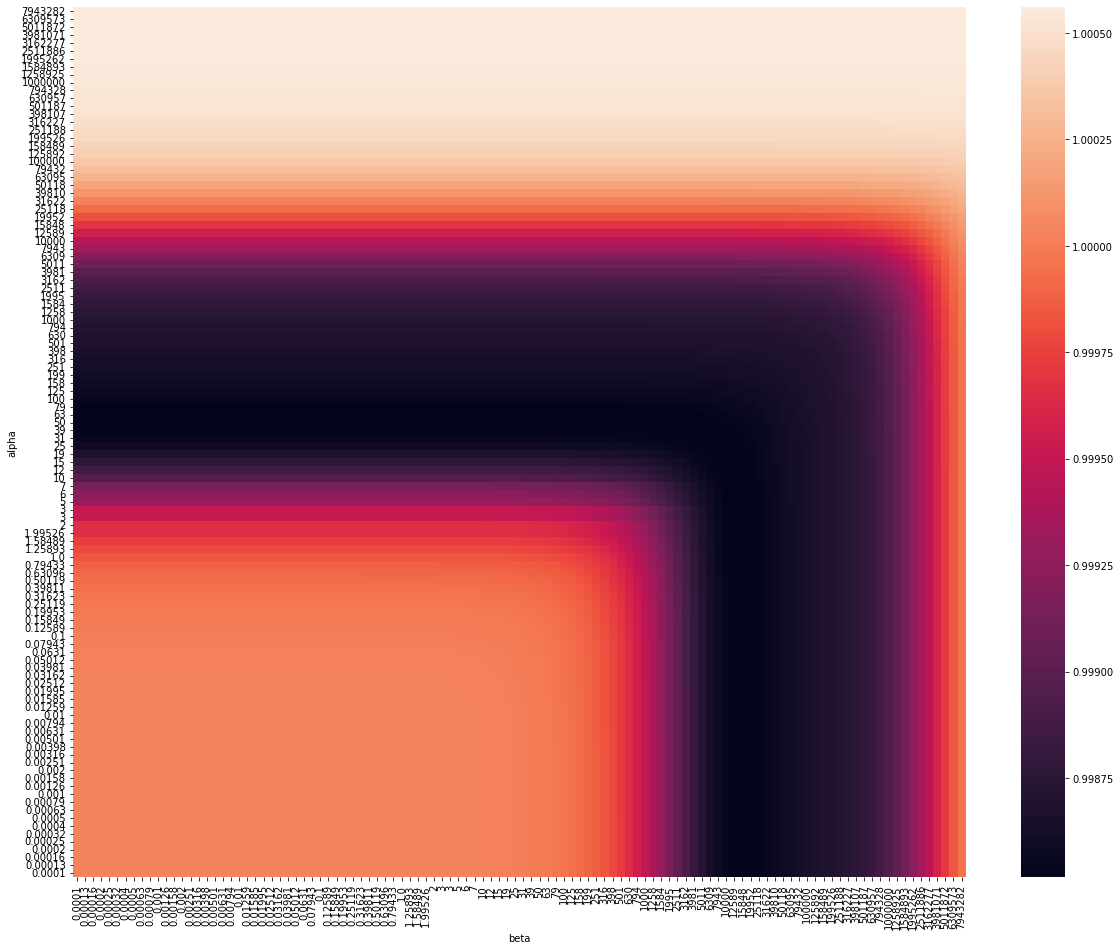}
  \end{minipage}
  \hfill
  \begin{minipage}[b]{0.23\textwidth}
    \includegraphics[width=\textwidth]{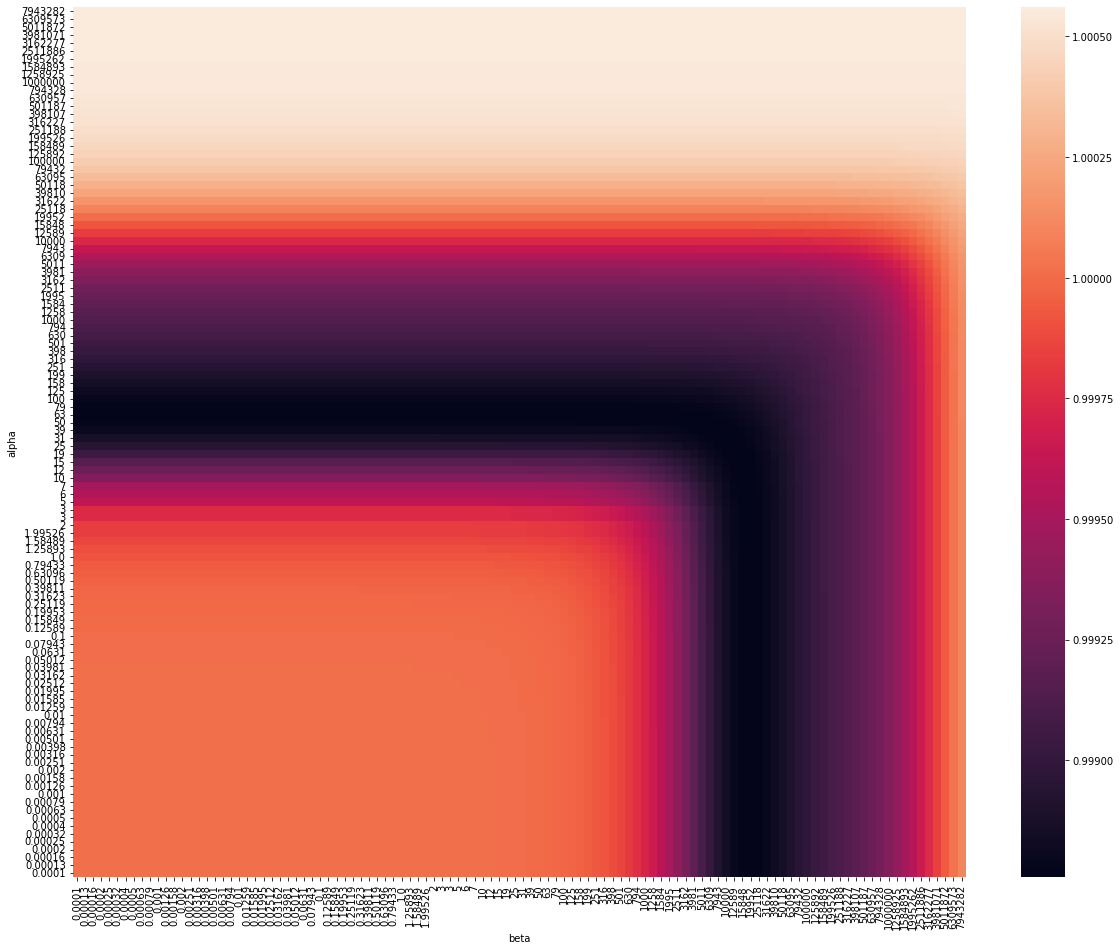}
  \end{minipage}
  \hfill
  \begin{minipage}[b]{0.23\textwidth}
    \includegraphics[width=\textwidth]{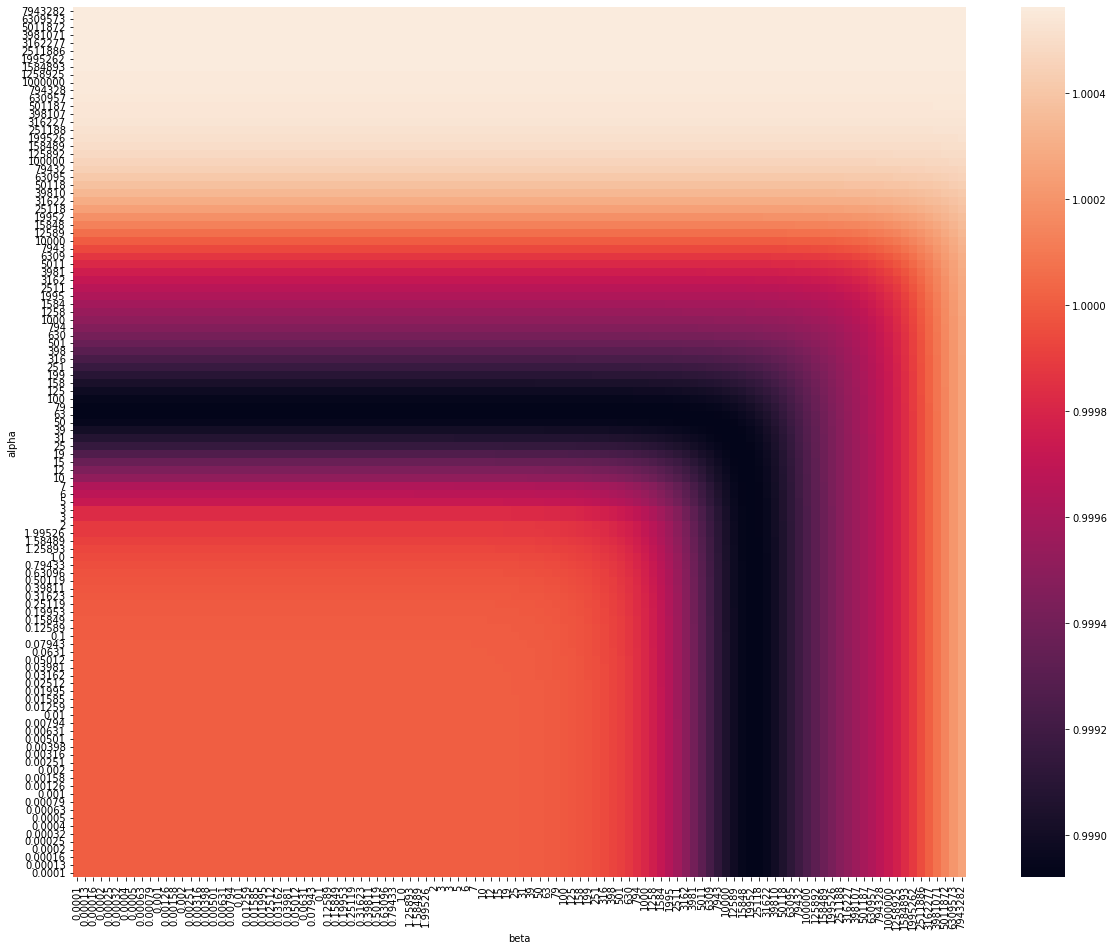}
  \end{minipage}
  \hfill
  \begin{minipage}[b]{0.23\textwidth}
    \includegraphics[width=\textwidth]{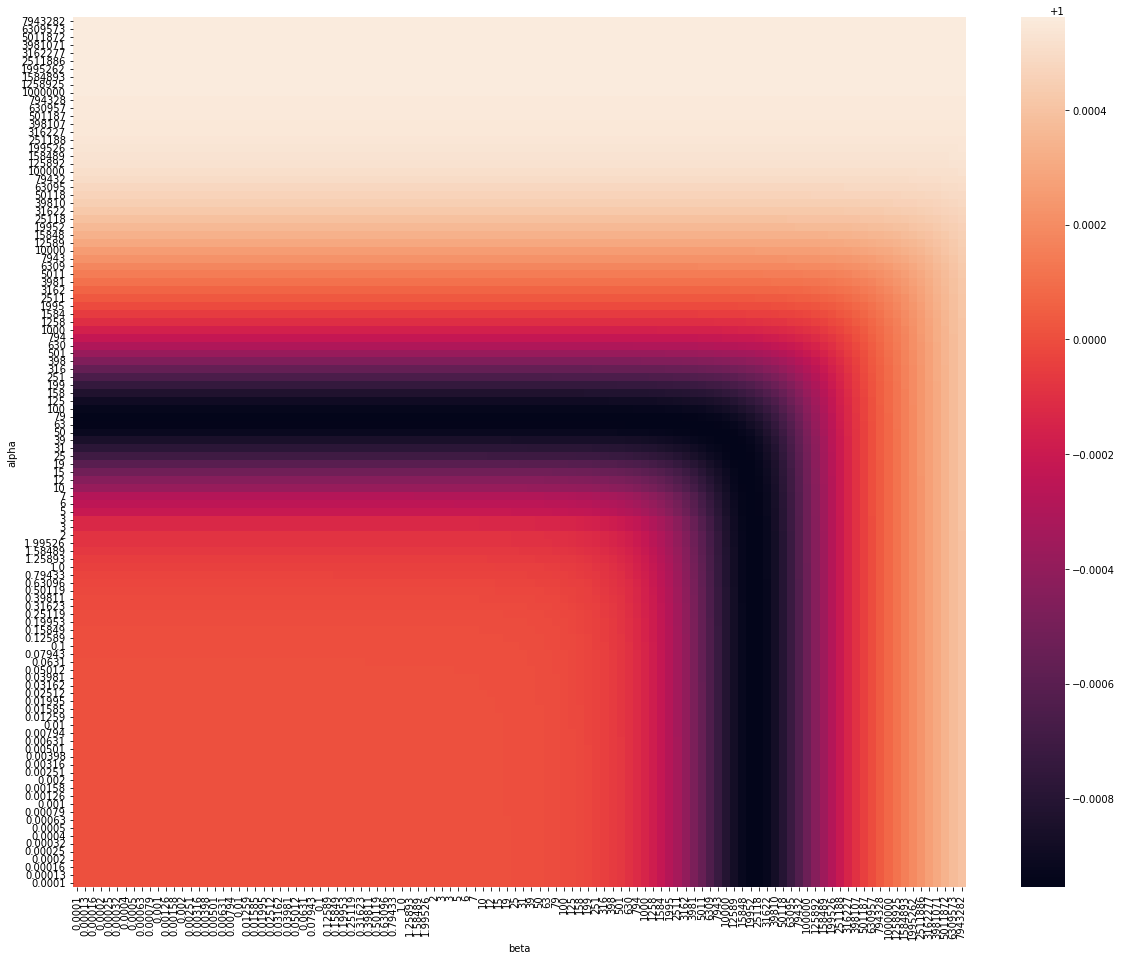}
  \end{minipage}
  \\~\\
  \caption{Landscape of squared loss for block-wise $\mW^V$ with uniform blocks (i.e. optima in Theorem~\ref{thm:optimal_Wv_given_uniform_attention_l2reg}), $T=100, v=300$. 
  (left-to-right) $\tau = 20$, $\tau = 40$, $\tau = 60$, $\tau = 80$. 
  In each plot, we perform a grid search over $\alpha, \beta \in [10^{-4}, 10^7]$ (both axes use log-scale). 
  Darker color represents lower loss. 
  Across a wide range of $\tau$ (compared to $T$), the loss is lowest when $(\alpha, \beta)$ is in some corner-shaped region (both $\alpha$ and $\beta$ are within some intervals whose lower bounds grow with $\tau$).
  }  
\end{figure}

\newpage
\clearpage
\section{ADDITIONAL EMPIRICAL RESULTS}
\label{sec:appendix:experiments}

\subsection{Additional results on learned value matrix \texorpdfstring{$\mW^V$}{Wv}}
\label{sec:appendix:experiments:Wv}

In Theorem~\ref{thm:optimal_Wv_given_uniform_attention_l2reg} and Figure~\ref{fig:Wv_one_hot_freeze_uniform_attention_l2reg} we have shown that 
when freezing uniform attention weights and one-hot word embedding,
under $L_2$-regularization,
training a single layer transformer on our synthetic topic modeling distribution (Section~\ref{sec:setup:topic_modeling}) would make its
$\mW^V$ converge to a block-wise pattern that encodes the topic structure.

In the following Figure~\ref{fig:Wv_one_hot_freeze_uniform_attention_no_l2reg}, we additionally show empirical results \emph{without} $L_2$-regularization,
matching our theory in Theorem~\ref{thm:optimal_Wv_given_uniform_attention}.

\begin{figure}[h]
  \centering
  \begin{minipage}[b]{0.23\textwidth}
    \includegraphics[width=\textwidth]{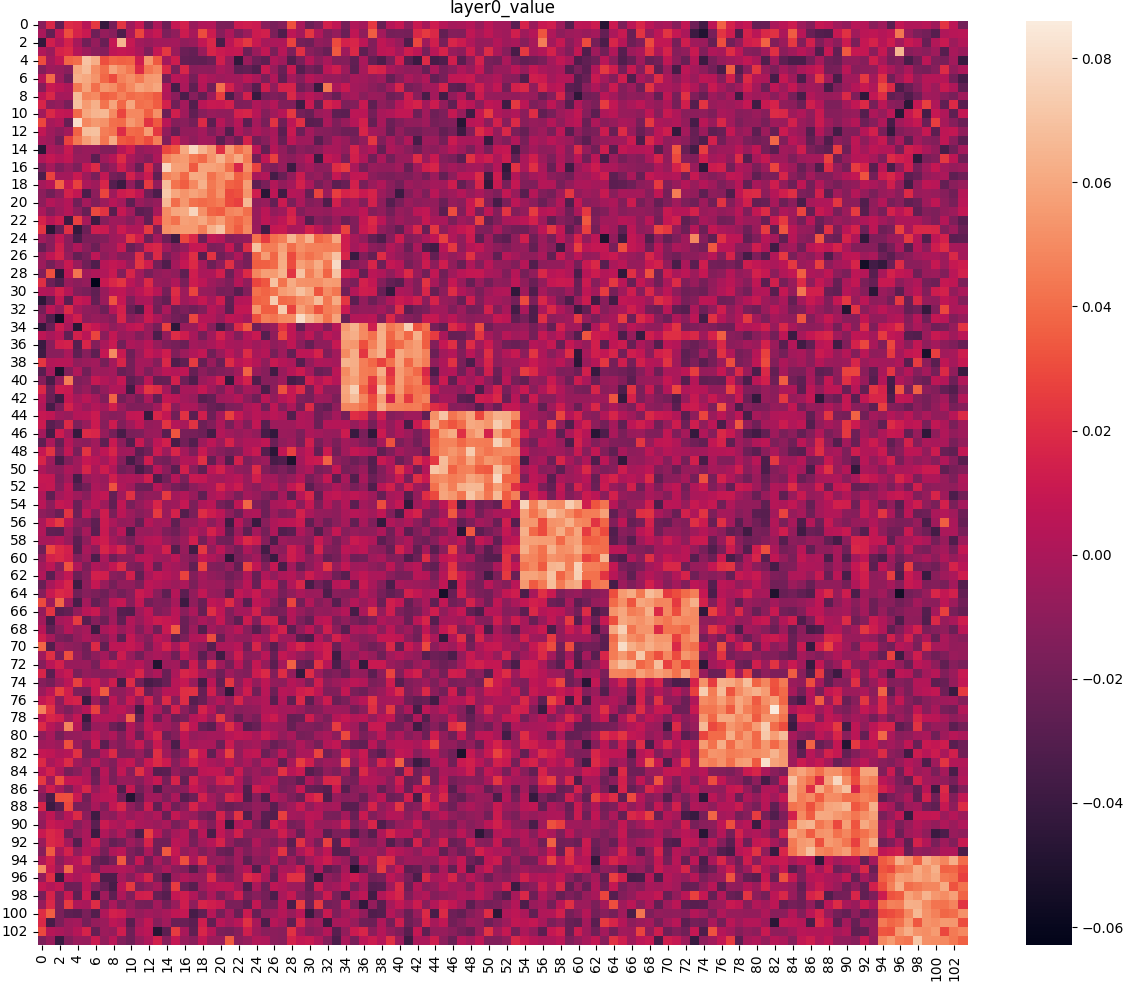}
  \end{minipage}
  \hfill
  \begin{minipage}[b]{0.23\textwidth}
    \includegraphics[width=\textwidth]{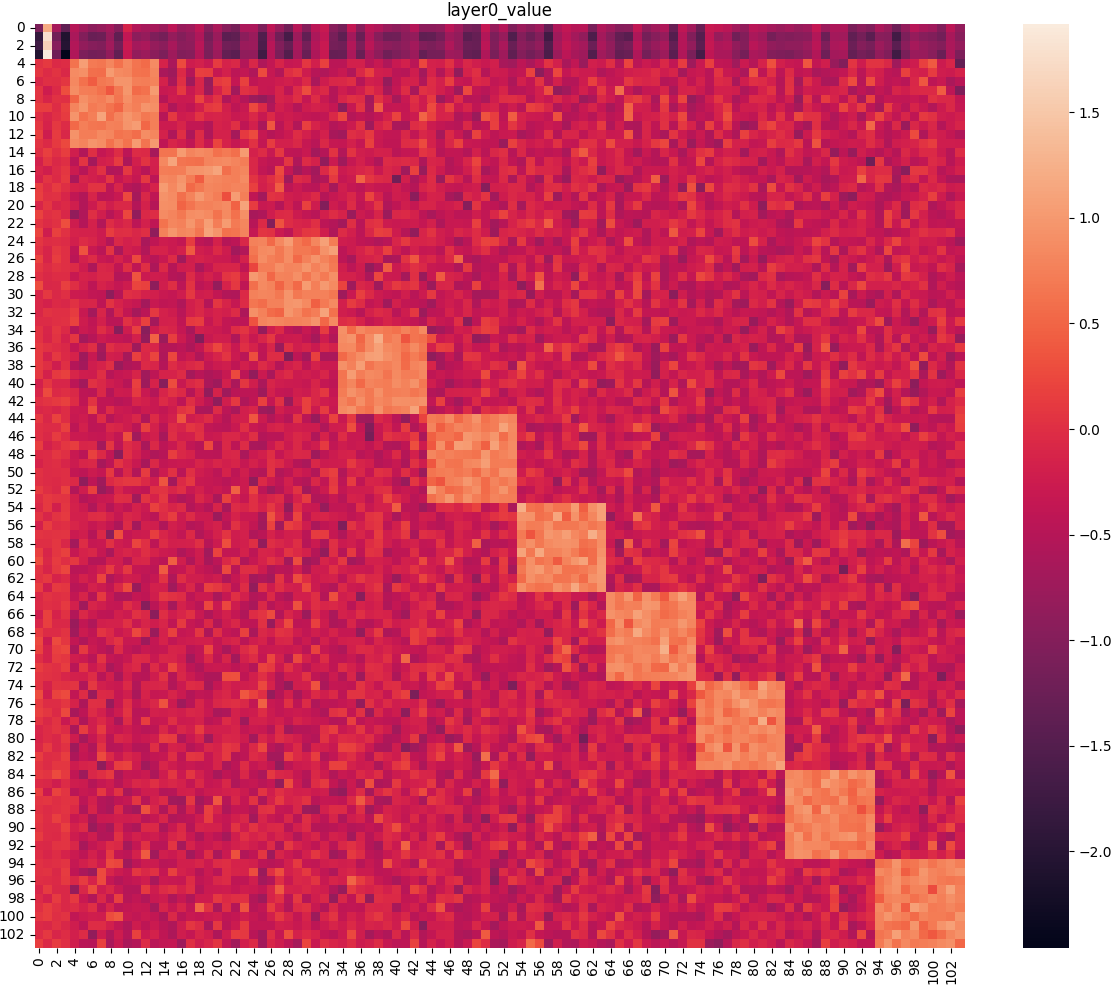}
  \end{minipage}
  \begin{minipage}[b]{0.23\textwidth}
    \includegraphics[width=\textwidth]{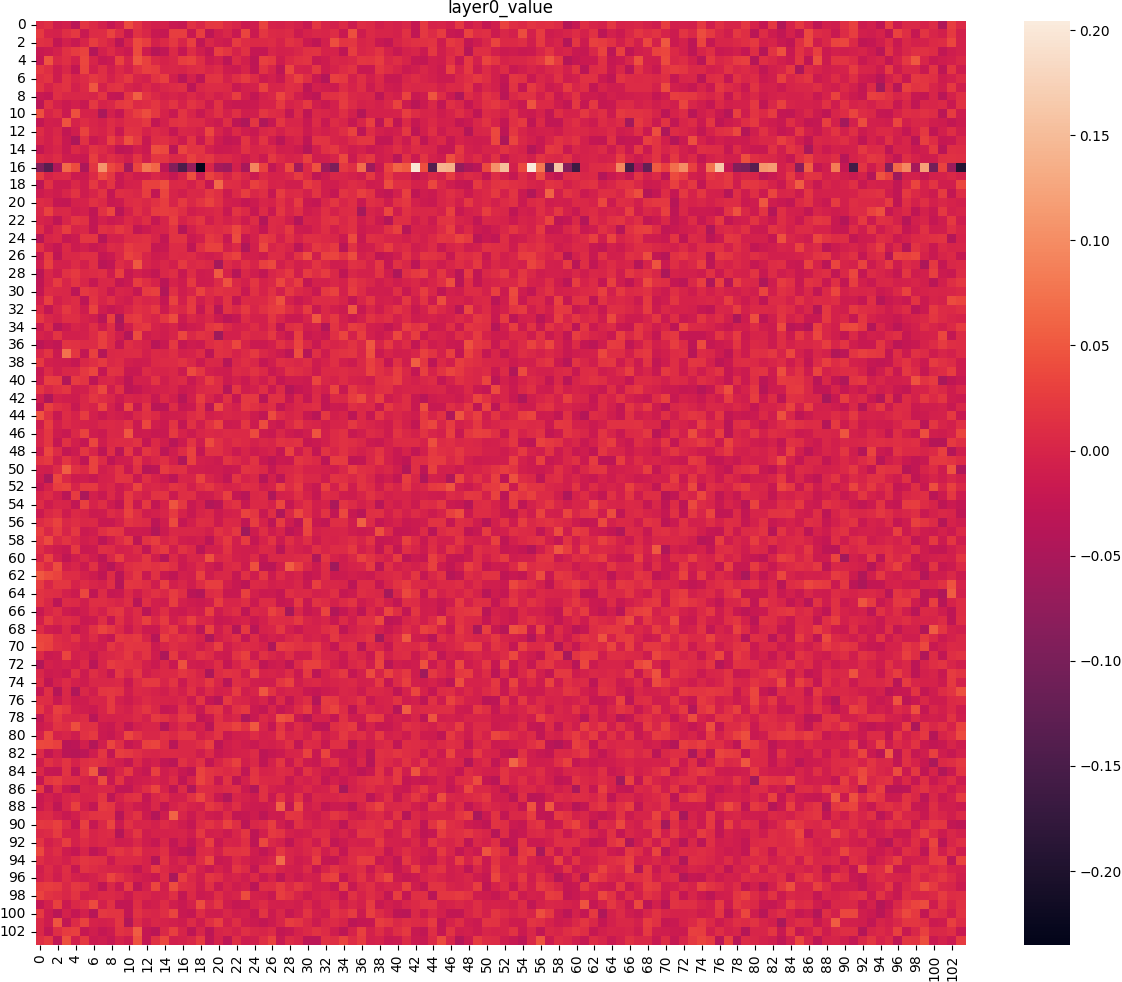}
  \end{minipage}
  \hfill
  \begin{minipage}[b]{0.23\textwidth}
    \includegraphics[width=\textwidth]{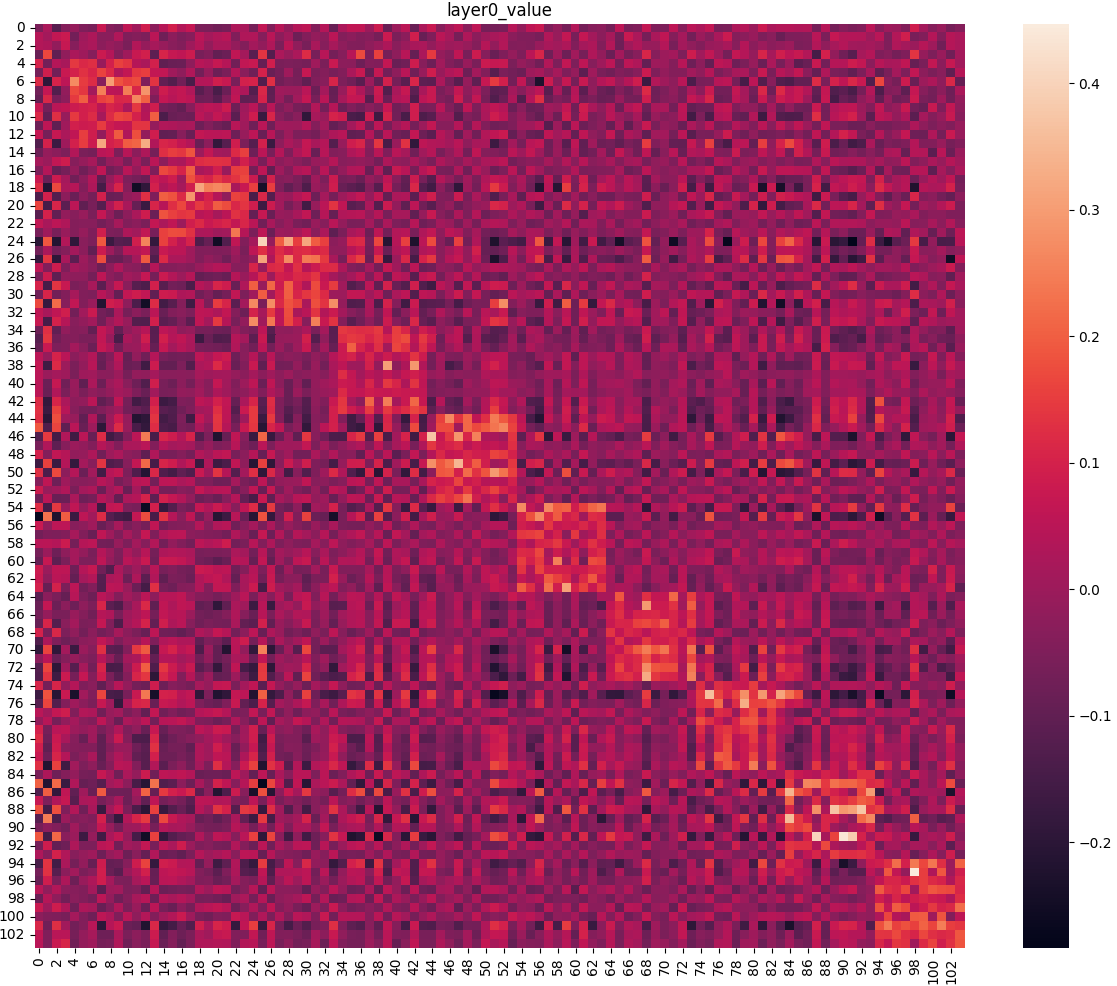}
  \end{minipage}
  \caption{Convergence point of trained $\mW^V$ (no $L_2$-regularization) when freezing uniform attention weights and one-hot word embedding. 
  The four plots correspond to different combinations of loss function and optimizer. 
  (Left to right) cross-entropy with SGD, cross-entropy with Adam, squared loss with SGD, squared loss with Adam,
  all using learning rate 0.01.
  The block-wise pattern verifies our theory in Section~\ref{sec:attention:value}.
  The 10 blocks correspond to the 10 topics in the data distribution.
  In particular, in the 
  third
  figure, the blocks are very weak and not easily visible, but we checked that the mean of the 1000 entries corresponding to the block positions is 0.00552563, which is over 10x the magnitude of the mean of a random subset of 1000 non-block entries (mean -0.00015675332, stdev 0.00060286524). 
  }
\label{fig:Wv_one_hot_freeze_uniform_attention_no_l2reg}
\end{figure}

Complementing our experimental results in Section~\ref{sec:experiments},
Figure~\ref{fig:Wv_one_hot_trained_attention} shows that even when the attention weights $\mW^K, \mW^Q$ are \emph{jointly trained} with $\mW^V$,
the model would still approximately converge to the type of block-wise $\mW^V$ described in our analyses in Section~\ref{sec:attention:value}. 

\begin{figure}[ht]
  \centering
  \begin{minipage}[b]{0.23\textwidth}
    \includegraphics[width=\textwidth]{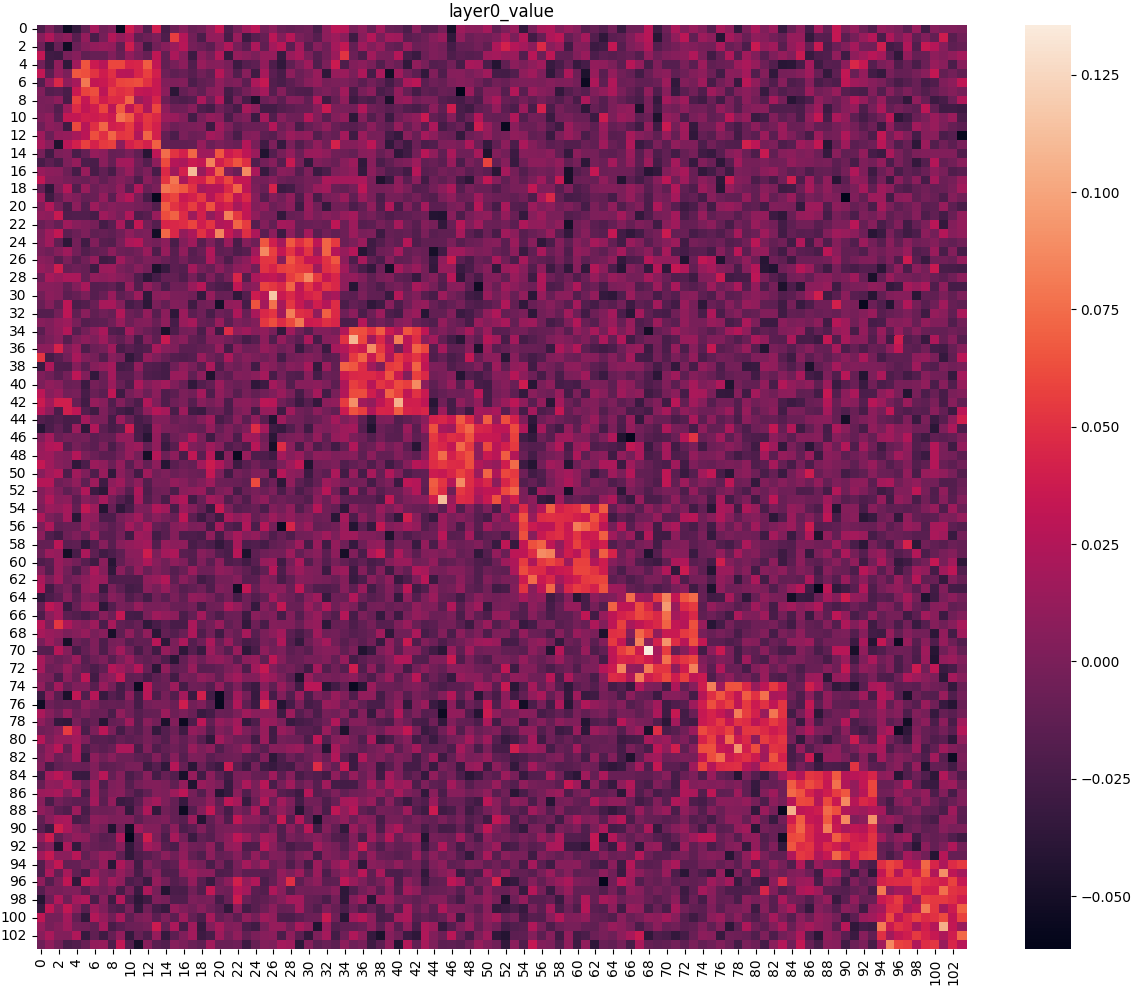}
  \end{minipage}
  \hfill
  \begin{minipage}[b]{0.23\textwidth}
    \includegraphics[width=\textwidth]{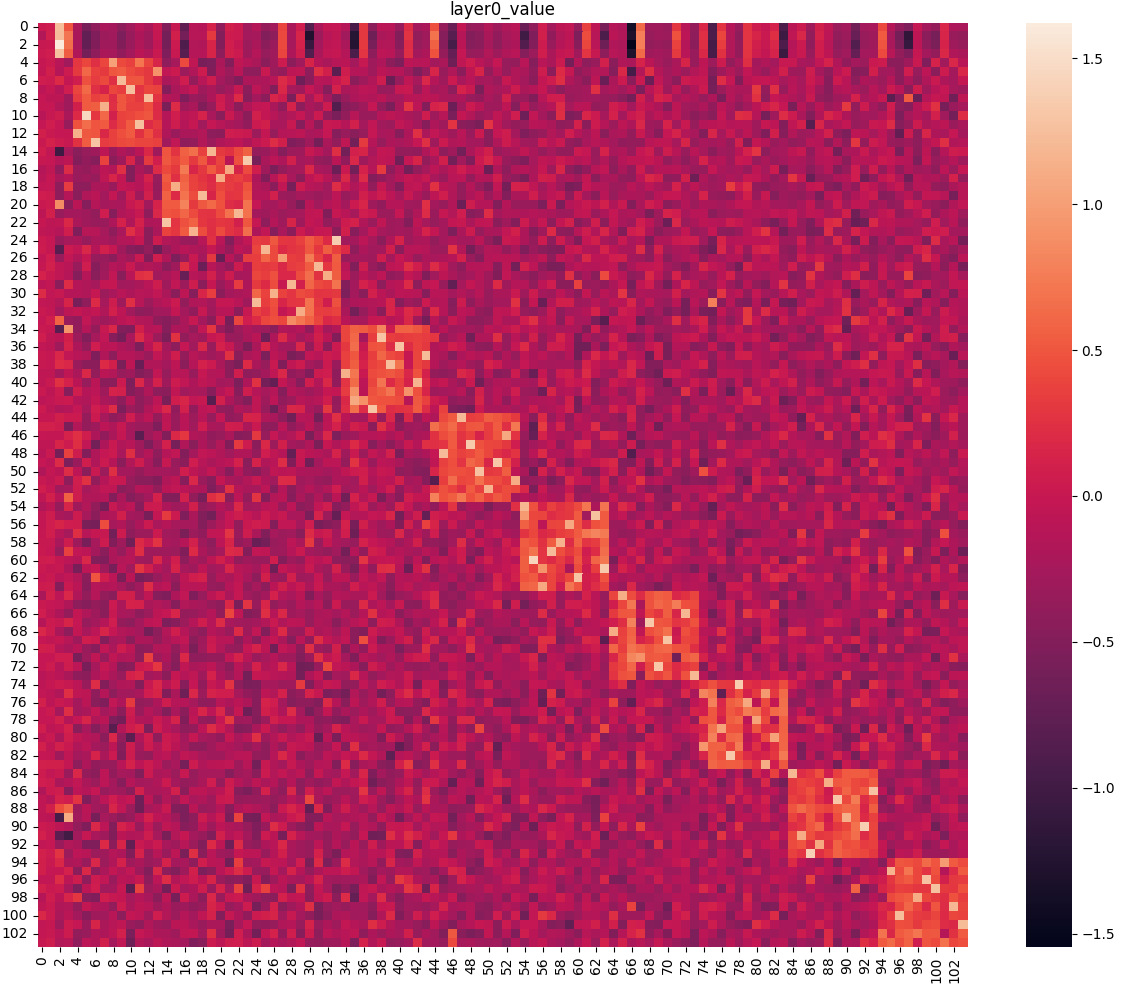}
  \end{minipage}
  \begin{minipage}[b]{0.23\textwidth}
    \includegraphics[width=\textwidth]{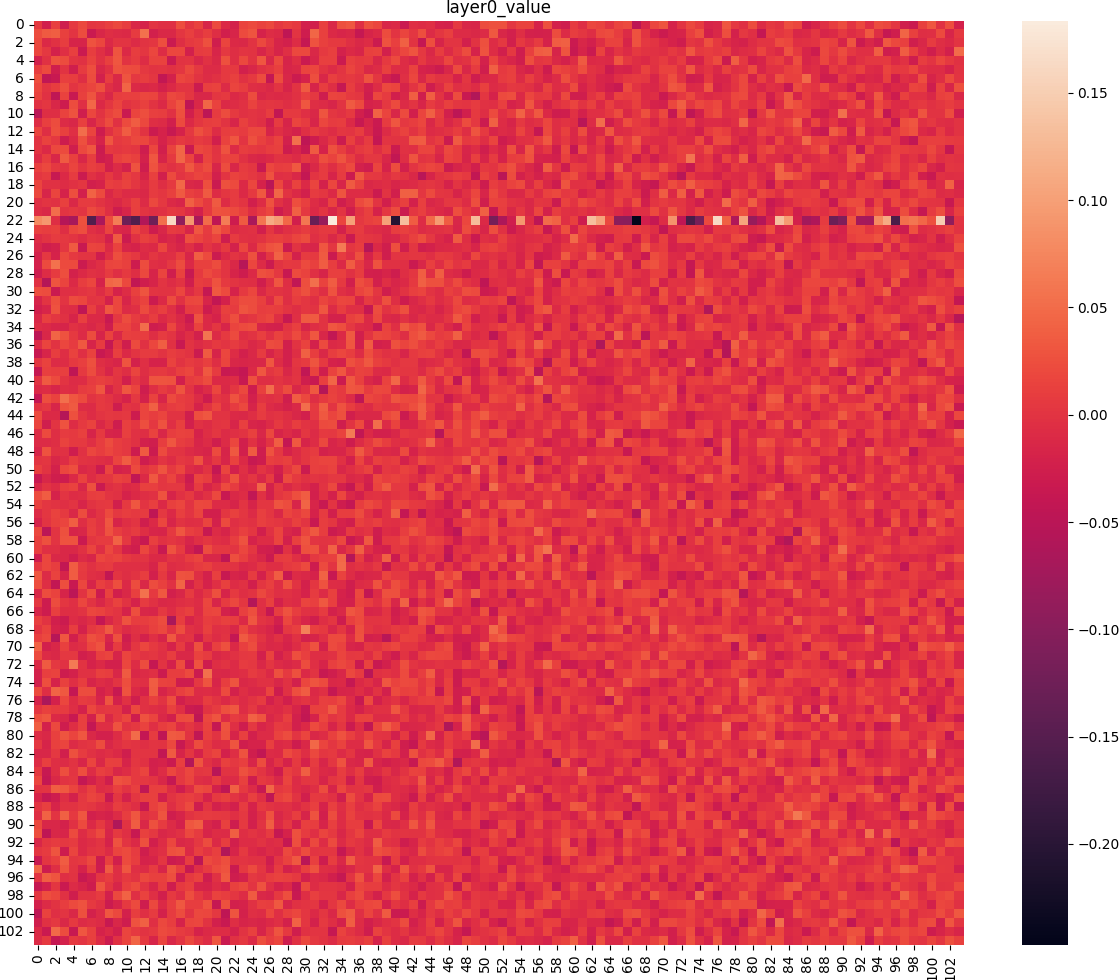}
  \end{minipage}
  \hfill
  \begin{minipage}[b]{0.23\textwidth}
    \includegraphics[width=\textwidth]{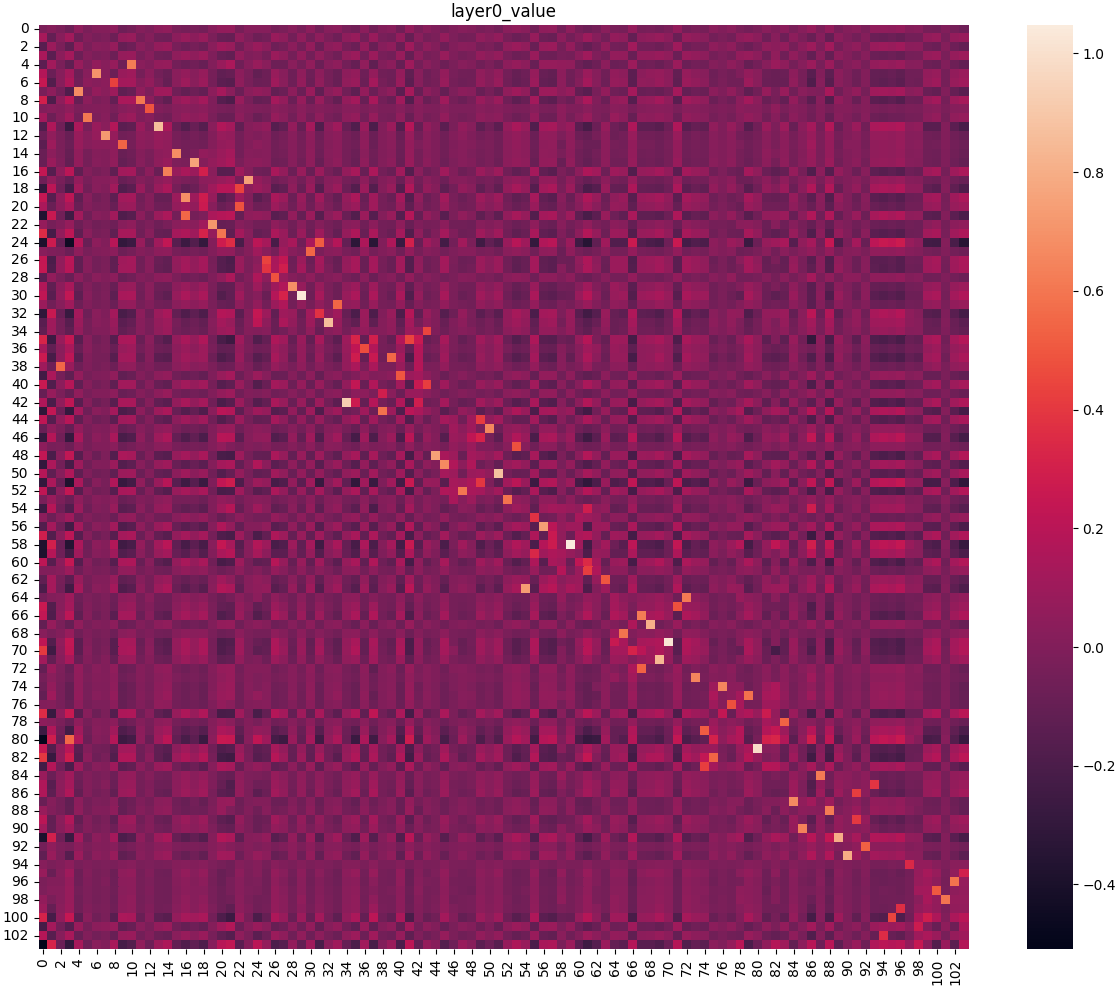}
  \end{minipage}
  \caption{Convergence point of trained $\mW^V$ when freezing one-hot word embedding but training attention weights. 
  (Left to right) cross-entropy with SGD, cross-entropy with Adam, squared loss with SGD, squared loss with Adam,
  all using learning rate 0.01.
  The block-wise pattern shows that our analysis in Section~\ref{sec:attention:value} closely approximates the empirical training dynamics when $\mW^K, \mW^Q, \mW^V$ are trained jointly. 
  The 10 blocks correspond to the 10 topics in the data distribution.
  In particular, in the 
  third
  figure, the blocks are very weak and not easily visible, but we checked that the mean of the 1000 entries corresponding to the block positions is 0.006545205, which is over 10x the magnitude of the mean of a random subset of 1000 non-block entries (mean -0.0006503917, stdev 0.0006370574).  
  }
\label{fig:Wv_one_hot_trained_attention}
\end{figure}

\clearpage
\subsection{Additional results on learned attention weights}
\label{sec:appendix:experiments:attn}

Complementing our experimental results in
Section~\ref{sec:experiments},
Table~\ref{tab:topic_attn_trained_Wv_block} shows that 
when the trained $\mW^V$ is closer to {\bf uniform within each block},
i.e. on average, each word pays more attention to \emph{different words of the same topic} than to words of \emph{different topics}.

\begin{table*}[ht]
\centering
\begin{tabular}{c|ccc}
\toprule
{\bf Optimizer and} & {\bf Avg Same-Word} & {\bf Avg Same-Topic-} & {\bf Avg Different-Topic}  \\
{\bf Learning Rate} & {\bf Attention} & {\bf -Different-Word Attention} & {\bf Attention} \\
\midrule
Adam 0.003 &  $0.00759 \pm 0.00171$ & $0.0108 \pm 0.000657$ & $0.00689 \pm 0.000160$ \\
Adam 0.01 & $0.00811 \pm 0.000705$ & $0.010 \pm 0.000392 $ & $0.00707 \pm 0.000178$ \\
Adam 0.03 & $0.00453 \pm 0.000346$ & $0.0116 \pm 0.000460$ & $0.00665 \pm 0.000200$ \\
SGD 0.01 & $0.0105$ & $0.0106$ & $0.00673$ \\
SGD 0.03 & $0.0140 \pm 0.00158$ & $0.0103 \pm 0.000357 $ & $0.00641 \pm 0.0000239$ \\
\bottomrule
\end{tabular}
\caption{\label{tab:topic_attn_trained_Wv_block} 
Average attention weights when the model (with one-hot word embeddings) is trained under the cross-entropy loss and $\mW^V$ converges to a block-wise pattern with closer to uniform blocks. 
We report mean $\pm$ std. deviation over 3 runs.
The row ``SGD 0.01" only contains 1 run, and the row ``SGD 0.003" is removed, because these models had much higher final train and dev losses than others.
For these failed runs, all three types of attention weights have similar averages, 
a sign that $\mW^K$ and $\mW^Q$ did not learn meaningful topical structures.
Note that under most settings, \emph{same-word} attention is larger than \emph{same-topic-different-word} attention, which is larger than \emph{different-topic} attention, 
verifying our conclusion in Theorem~\ref{thm:optimal_attention_weights_updated}.
The models trained using ``Adam 0.03" has larger \emph{same-topic-different-word} attention, which possibly made it unnecessary to rely on \emph{same-word} attention to achieve a low loss,
though our theory suggests that increasing \emph{same-word} attention could further reduce the loss.
}
\end{table*}

On the other hand,
when the trained $\mW^V$ is closer to {\bf a diagonal pattern},
the above ordering is partially reversed,
Table~\ref{tab:topic_attn_trained_Wv_diagonal} shows that 
on average, each word pays the most attention to the \emph{same word} in the document, followed by words of \emph{different topics}, and the least attention to \emph{different words of the same topic}.

\begin{table*}[ht]
\centering
\begin{tabular}{c|ccc}
\toprule
{\bf Learning Rate} & {\bf Avg Same-Word} & {\bf Avg Same-Topic-} & {\bf Avg Different-Topic}  \\
 & {\bf Attention} & {\bf -Different-Word Attention} &  {\bf Attention} \\
\midrule
0.003 &  $0.0916 \pm 0.000901$ & $0.00185 \pm 0.000170$ & $0.00256 \pm 0.0000332$ \\
0.01 & $0.0918 \pm 0.00244$ & $0.00182 \pm 0.000474 $ & $0.00256 \pm 0.000109$ \\
\bottomrule
\end{tabular}
\caption{\label{tab:topic_attn_trained_Wv_diagonal} Average attention weights when the model is trained under the cross-entropy loss with the Adam optimizer and $\mW^V$ converges to a diagonal pattern. 
We report mean $\pm$ std. deviation over 7 runs, 
selected out of 10, by removing the runs in which the diagonal pattern in $\mW^V$ is not visible or weak.
Note that on average, \emph{same-word} attention is larger than \emph{different-topic} attention, which is larger than \emph{same-topic-different-word} attention, 
verifying our conclusion in Theorem~\ref{thm:optimal_attention_weights_updated:WvI}.
}
\end{table*}

\clearpage
\subsection{Additional details and results on natural language data}
\label{sec:appendix:experiments:wiki}

In particular, for fair comparison, we should focus on the embedding similarity and attention weights between \emph{different words of the same topic} and \emph{different words of different topics}.
(This is because those metrics are less meaningful for a pair of two \emph{same words}, since their embeddings dot product is expected to be larger, which further biases the attention score comparisons. )

\paragraph{Ambiguity filter}
We also note that, for each word, an LDA model assigns some probability distribution of its topics.
To determine whether two words are of the same topic, it is more meaningful if they share a topic in which both words have high likelihood.
(By contrast, if two words each has some rarely-used topic that happens to overlap, we intuitively think of them as having different topics.)

To formalize such intuition, we filter out stop tokens, and other tokens that are not central to any topic (determined by the LDA).
That is, for each topic $t$,
LDA assigns to it a likelihood $p_i$ for each word $w_i$ in the vocabulary (of size $n$).
We sort these (word, likelihood) pairs by decreasing likelihood: 
\[ (w_1, p_1), \cdots, (w_n, p_n) \]
then for a pre-defined threshold parameter $\theta \in (0, 1)$ controlling the proportion of words to be assigned to each topic, 
we only consider the topic $t$ to contain the following words
\[ \{ w_i: i \le \theta n \} \]

\paragraph{Debiasing average attention weight}
Moreover, we note that sentence length may cause a bias in attention weights calculation:
intuitively, the average attention weight is the inverse of sentence length,
but longer sentences usually contain more topics (and hence a larger proportion of different-topic word pairs).
Thus, we expect that the average attention weight between different-topic word pairs are smaller than that between same-topic word pairs, 
\emph{even for a transformer with random parameters}.
(Empirically this bias indeed exists robustly, both on synthetic data and on Wikipedia data.)
Therefore, we debias the effect of sentence length on attention weights:
for each sentence, while computing the pairwise attention weights among its words,
we ``normalize the sentence length to 100",
that is, we multiply the raw attention weights by sentence length, and then divide the result by 100.
In this way, the average attention weight in each sentence is always $\frac{1}{100}$,
regardless of the proportion of same-topic and different-topic word pairs.
Indeed, as Table~\ref{tab:wiki_emb_attn_1topics_per_word} and Table~\ref{tab:wiki_emb_attn} show,
for a randomly initialized BERT model, after our debiasing, the average same-topic and different-topic attention weights are roughly equal.

\paragraph{Results}
For a set of pre-trained transformer-based models downloaded from Huggingface \citep{wolf2020transformers},
we compare the embedding similarity and attention weights between same-topic tokens and different-topic tokens.
The topics are determined by fitting an LDA model with 100 topics on a sample of tokenized Wikipedia corpus. 
We apply the above-mentioned ambiguity filter and debiasing.

\begin{itemize}
    \item When we further restrict to keeping only one topic for each word (to be consistent with the setting in our theoretical analysis): see Table~\ref{tab:wiki_emb_attn_1topics_per_word}.
    \item Without the last restriction above: see the following Table~\ref{tab:wiki_emb_attn}.
\end{itemize}

\begin{table*}[!ht]
\centering
\begin{tabular}{cc|ccc}
\toprule
{\bf Model} & {\bf Ambiguity} &  {\bf Avg embedding} &  {\bf Avg embedding} &  {\bf Avg attn weight}  \\ & {\bf Threshold}   & {\bf Cosine Similarity} & {\bf Dot Product} &  \bf{(Same-topic} \\ &  & {\bf (Same-topic/Diff-topic)} & {\bf (Same-topic/Diff-topic)} & \bf /Diff-topic)  \\
\midrule
Bert & 0.0005 & 1.14 & 1.04 &  1.23 \\
 & 0.001   & 0.97 & 1.05 & 1.17 \\
 & 0.002   & 0.99 & 0.93 & 1.13 \\
\hline
Albert & 0.0005  & 4.15 & 3.06 & 1.23 \\
 & 0.001  & 3.09 & 3.04 & 1.17 \\
 & 0.002  & 1.54 & 1.44 & 1.11 \\
\hline
Bart & 0.0005  & 2.51 & 1.76 & 1.27 \\
 & 0.001 & 1.63 & 1.12 & 1.20 \\
 & 0.002  & 1.06 & 0.85 & 1.11 \\
\hline
Electra & 0.0005   & 5.28 & 3.99 & 1.70 \\
 & 0.001   & 5.56 & 5.57 & 1.58 \\
 & 0.002   & 6.39 & 5.61 & 1.48 \\
\hline
Roberta & 0.0005  & 4.39 & 5.01 & 1.19 \\
 & 0.001  & 5.20 & 4.25 & 1.15 \\
 & 0.002  & 4.71 & 4.15 & 1.12 \\
\hline
Bert & 0.0005 & 0.99814 & 0.99957 & 1.00009 \\
(randomly & 0.001  & 0.99820 & 1.00167 & 1.00013  \\
initialized) & 0.002  & 0.99964 & 0.99928 & 0.99978  \\
\bottomrule
\end{tabular}
\caption{\label{tab:wiki_emb_attn}
For models pretained on Wikipedia dataset, their token embeddings and attention weights encode topic structure. 
The different columns are:
(1) The ``ambiguity threshold", i.e. the number of words per topic, divided by the vocabulary size; each word is only assigned one {\bf or more} topic(s)
(2) The average embedding cosine similarity between different words of the \emph{same topic}, divided by that between words of \emph{different topics}.
(3) The average embedding dot product between different words of the \emph{same topic}, divided by that between words of \emph{different topics}.
(4) The average attention weight between different words of the \emph{same topic}, divided by that between words of \emph{different topics}. (The attention weights are normalized for debiasing, see Appendix~\ref{sec:appendix:experiments:wiki}).
Different rows represent different evaluation settings, controlled by ``ambiguity threshold".
Note that the avg same-topic embedding similarity and attention weight are mostly greater than the avg diff-topic counterparts (with some exceptions).
Allowing multiple topics per word is different from our theoretical setup,
so our conclusions in Theorem~\ref{thm:optimal_embedding} and Theorem~\ref{thm:optimal_attention_weights_updated} do not cover this setting,
though we conjecture that some variants of these theoretical results can be proven using similar approaches to ours.
}
\end{table*}

\end{document}